\documentclass{article}

\usepackage[final]{neurips_2022}
\usepackage[toc,page,header]{appendix}
\usepackage{minitoc}

\usepackage[utf8]{inputenc} 
\usepackage[T1]{fontenc}    
\usepackage{hyperref}       
\usepackage{url}            
\usepackage{booktabs}       
\usepackage{amsfonts}       
\usepackage{nicefrac}       
\usepackage{microtype}      
\usepackage{xcolor}         
\usepackage{amsmath}
\usepackage{amssymb}
\usepackage{mathtools}
\usepackage{amsthm}
\usepackage{multirow}
\usepackage{tablefootnote}
\usepackage{wrapfig}
\usepackage{graphicx}
\newtheorem{theorem}{Theorem}[section]
\newtheorem{proposition}[theorem]{Proposition}
\newtheorem{lemma}[theorem]{Lemma}
\newtheorem{corollary}[theorem]{Corollary}
\theoremstyle{definition}
\newtheorem{definition}[theorem]{Definition}

\theoremstyle{remark}

\usepackage[linesnumbered,ruled]{algorithm2e}

\DeclareMathOperator*{\argmin}{arg\,min}

\title{Formulating Robustness Against Unforeseen Attacks}

\author{%
  Sihui Dai \\
  Princeton University\\
  \texttt{sihuid@princeton.edu} \\
  \And
  Saeed Mahloujifar \\
  Princeton University\\
  \texttt{sfar@princeton.edu} \\
  \And
  Prateek Mittal \\
  Princeton University\\
  \texttt{pmittal@princeton.edu} \\
}

\newcommand{\mysubsection}[1]{\subsection{#1}}

\begin{document}

\maketitle

\doparttoc 
\faketableofcontents 

\begin{abstract}
Existing defenses against adversarial examples such as adversarial training typically assume that the adversary will conform to a specific or known threat model, such as $\ell_p$ perturbations within a fixed budget. In this paper, we focus on the scenario where there is a mismatch in the threat model assumed by the defense during training, and the actual capabilities of the adversary at test time. We ask the question: if the learner trains against a specific ``source" threat model, when can we expect robustness to generalize to a stronger unknown ``target" threat model during test-time? Our key contribution is to formally define the problem of learning and generalization with an unforeseen adversary, which helps us reason about the increase in adversarial risk from the conventional perspective of a known adversary. Applying our framework, we derive a generalization bound which relates the generalization gap between source and target threat models to variation of the feature extractor, which measures the expected maximum difference between extracted features across a given threat model. Based on our generalization bound, we propose variation regularization (VR) which reduces variation of the feature extractor across the source threat model during training. We empirically demonstrate that using VR can lead to improved generalization to unforeseen attacks during test-time, and combining VR with perceptual adversarial training \citep{laidlaw2020perceptual} achieves state-of-the-art robustness on unforeseen attacks.  Our code is publicly available at  \url{https://github.com/inspire-group/variation-regularization}.
\end{abstract}


\section{Introduction}
\label{sec:intro}
Neural networks have impressive performance on a variety of datasets \citep{lecun1998gradient, he2015rectifiers, KrizhevskySH17, EveringhamGWWZ10} but can be fooled by imperceptible perturbations known as adversarial examples \citep{szegedy2013intriguing}.  The conventional paradigm to mitigate this threat often assumes that the adversary generates these examples using some known threat model, primarily $\ell_p$ balls of specific radius \citep{cohen2019certified, zhang2020towards, madry2017towards}, and evaluates the performance of defenses based on this assumption. This assumption, however, is unrealistic in practice. In general, the learner does not know exactly what perturbations the adversary will apply during test-time.

To bridge the gap between the setting of robustness studied in current adversarial ML research and robustness in practice, we study the problem of \textit{learning with an unforeseen adversary}.  In this problem, the learner has access to adversarial examples from a proxy ``source" threat model but wants to be robust against a more difficult ``target" threat model used by the adversary during test-time. We ask the following questions:
\begin{enumerate}
\item \textit{When can we expect robustness on the source threat model to generalize to the true unknown target threat model used by the adversary? }

\item \textit{How can we design a learning algorithm that reduces the drop in robustness from source threat model to target threat model?}
\end{enumerate} 

To address the first question, we introduce \emph{unforeseen adversarial generalizability} which provides a framework for reasoning about what types of learning algorithms produce models that generalize well to unforeseen attacks.  Based on this framework, we derive a generalization bound which relates the difference in adversarial risk across source and target threat models to a quantity we call variation: the expected maximum difference between extracted features across a given threat model.

Our bound addresses the second question; it suggests that learning algorithms that bias towards models with small variation across the source threat model exhibit smaller drop in robustness to particular unforeseen attacks. Thus, we propose variation regularization (VR) to improve robustness to unforeseen attacks.  We then empirically demonstrate that when combined with adversarial training, VR improves generalization to unforeseen attacks during test-time across multiple datasets and architectures.
Our contributions are as follows:

\textbf{We formally define the problem of learning with an unforeseen adversary with respect to adversarial risk.}  We make the case that one way of learning with an unforeseen adversary is to ensure that the gap between the empirical adversarial risk measured on the source adversary and the expected adversarial risk on the target adversary remains small.  To this end, we define \textit{unforeseen adversarial generalizability} which provides a framework for understanding under what conditions we would expect small generalization gap.

\textbf{Under our framework for generalizability, we derive a generalization bound for generalization across threat models.}  Our bound relates the generalization gap to a quantity we define as variation, the expected maximum difference between extracted features across a given threat model.  We demonstrate that under certain conditions, we can decrease this upper bound \textit{while only using information about the source threat model}.

\textbf{Using our bound, we propose a regularization term which we call variation regularization (VR).}  We incorporate this regularization term into adversarial training and perceptual adversarial training \citep{laidlaw2020perceptual}, leading to learning algorithm that we call AT-VR and PAT-VR respectively. We find that VR can lead to improved robustness on unforeseen attacks across datasets such as CIFAR-10, CIFAR-100, and ImageNette over adversarial training without VR.  Additionally, PAT-VR achieves state-of-the-art (SOTA) robust accuracy on LPIPS-based attacks, improving over PAT by 21\% and SOTA robust accuracy on a union of $\ell_{\infty}$, $\ell_2$, spatially transformed \citep{XiaoZ0HLS18}, and recolor attacks \citep{LaidlawF19}.
\section{Related Works}
\label{sec:rel_works}
\textbf{Adversarial examples and defenses} Previous studies have shown that neural networks can be fooled by perturbations known as adversarial examples, which are imperceptible to humans but cause NNs to predict incorrectly with high confidence \citep{szegedy2013intriguing}.  These adversarial examples can be generated by various threat models including $\ell_p$ perturbations, spatial transformations \citep{XiaoZ0HLS18}, recoloring \citep{LaidlawF19}, and broader threat models such as fog and snow distortions \citep{kang2019robustness}.  While many defenses have been proposed, most defenses provide guarantees for specific threat models \citep{cohen2019certified, zhang2020towards, Croce020, yang2020randomized, zhang2020black} or use knowledge of the threat model during training \citep{madry2017towards, zhang2019theoretically, wu2020adversarial, TB19, MainiWK20}.  Adversarial training is a popular defense framework in which a model is trained using adversarial examples generated by a particular threat model, such as $\ell_2$ or $\ell_{\infty}$ attacks \citep{madry2017towards, zhang2019theoretically, wu2020adversarial}.  Prior works have also extended adversarial training to defend against unions of attack types such as unions of $\ell_p$-balls \citep{MainiWK20, TB19} and against stronger adversaries more aligned to human perception \citep{laidlaw2020perceptual}.

\textbf{Bounds for Learning with Adversarial Examples} 
An interesting body of work studies generalization bounds for specific attacks \cite{cullina2018pac,attias2019improved,montasser2019vc,raghunathan2019adversarial,chen2020more,diakonikolas2019nearly,yu2021understanding,diochnos2019lower}. In particular, they study generalization in the setting where the learning algorithm minimizes the adversarial risk on the training set and hopes to generalize to same adversary during test-time.  \citet{montasser2021adversarially} provide bounds for the problem of generalizing to an unknown adversary with oracle access to that adversary during training.  Our work differs since we study generalization and provide bounds under the setting in the learner only has access to samples from a weaker adversary than present at test-time.

\textbf{``Unforeseen" attacks and defenses } While several prior works have studied ``unforeseen" attacks \citep{kang2019robustness, stutz2020confidence, laidlaw2020perceptual, jin2020manifold}, these works are empirical works, and the term ``unforeseen attack" has not been formally defined.  \citet{kang2019robustness} first used the term ``unforeseen attack" when proposing a set of adversarial threat models including Snow, Fog, Gabor, and JPEG to evaluate how well defenses can generalize from $\ell_{\infty}$ and $\ell_2$ to broader threat models.  \citet{stutz2020confidence} and \citet{chen2022revisiting} propose adversarial training based techniques with a mechanism for abstaining on certain inputs to improve generalization from training on $\ell_{\infty}$ to \textit{stronger} attacks including $\ell_p$ attacks of larger norm.  Other defenses against ``unforeseen attacks" consider them to be attacks that are not used during training, but not necessarily stronger than those used in training.  For instance, \citet{laidlaw2020perceptual} propose using LPIPS \citep{zhang2018lpips}, a perceptually aligned image distance metric, to generate adversarial examples during training.  They demonstrate that by training using adversarial examples using this distance metric, they can achieve robustness against a variety of adversarial threat models including $\ell_{\infty}$, $\ell_2$, recoloring, and spatial transformation.  However, the LPIPS attack is the strongest out of all threat models tested and contains a large portion of those threat models.  To resolve these differences in interpretation of ``unforeseen attack", we provide a formal definition of learning with an unforeseen adversary.

\textbf{Domain Generalization} A problem related to generalizing to unforeseen attacks is the problem of domain generalization under covariate shift.  In the domain generalization problem, the learner has access to multiple training distributions and has the goal of generalizing to an unknown test distribution.  \citep{albuquerque2019generalizing} demonstrate that when the test distribution lies within a convex hull of the training distributions, learning is feasible. \citep{ye2021towards} propose a theoretical framework for domain generalization in which they derive a generalization bound in terms of the variation of features across training distributions.  We focus on the problem of generalizing to unforeseen adversaries and demonstrate that a generalization bound in terms of variation of features across the training threat model exists.

\section{Adversarial Learning with an Unforeseen Adversary}
\textbf{Notations} We use $\mathcal{D} = (\mathcal{X}, \mathcal{Y})$ to denote the data distribution and $D_m$ to denote a dataset formed by $m$ iid samples from $\mathcal{D}$.  We use $X$ to denote the support of $\mathcal{X}$. To match learning theory literature, we will refer to the defense as a learning algorithm $\mathcal{A}$, which takes the adversarial threat model and training data as inputs and outputs the learned classifier ($\hat{f} = \mathcal{A}(S, D_m)$ where $S$ is the threat model). We use $\mathcal{F}$ to denote the function class that $\mathcal{A}$ is applied over and $\mathcal{A}(S, D_m) \in \mathcal{F}$. $\mathcal{F} = \mathcal{G} \circ \mathcal{H}$ denotes a function class where $\forall f \in \mathcal{F}$, $f = g \circ h$ where $g \in \mathcal{G}$, $h \in \mathcal{H}$.

\label{sec:formulation}
In this section, we will define what constitutes an unforeseen attack and the learner's goal in the presence of unforeseen attacks.  We then introduce unforeseen adversarial generalizability which provides a framework for reasoning about what types of learning algorithms give models that generalize well to unforeseen adversaries.

\mysubsection{Formulating Learning with an Unforeseen Adversary}
To formulate adversarial learning with an unforeseen adversary, we begin by defining threat model and adversarial risk.  We will then use these definitions to explain the goal of the learner in the presence of an unforeseen adversary.

\begin{definition}[Threat Model] The \textit{threat model} is defined by a neighborhood function $N(\cdot): X \rightarrow \{0,1\}^{X}$.  For any input $x \in X$, $N(x)$ contains $x$. 
\end{definition}
\begin{definition}[Expected and Empirical Adversarial Risk] We define \textit{expected adversarial risk} for a model $f$ with respect to a threat model $N$ as
 $L_N(f) = \mathbb{E}_{(x, y) \sim \mathcal{D}} \max_{x' \in N(x)} \ell(f(x'), y)$
where $\ell$ is a loss function.  In practice, we generally do not have access to the true data distribution $\mathcal{D}$, but have $m$ iid samples $\{(x_i, y_i)\}_{i=1}^m$.  We can approximate $L_N(f)$ with the \textit{empirical adversarial risk} defined as 
$\hat{L}_N(f) = \frac{1}{m}\sum_{i=1}^m \max_{x_i' \in N(x_i)} \ell(f(x_i'), y_i)$
\end{definition}

In adversarial learning, the learner's goal is to find a function $f \in \mathcal{F}$ that minimizes $L_T$ where $T$ threat model used by the adversary. We call $T$ the \textit{target threat model}.  We call the threat model that the learner has access to during training the \textit{source threat model}.  We divide the adversarial learning problem into 2 cases, learning with a foreseen adversary and learning with an unforeseen adversary.  To distinguish between these 2 cases, we first define the subset operation for threat models.

\begin{definition}[Threat Model Subset and Superset]  We call a threat model $S$ a subset of another threat model $T$ (and $T$ a superset of $S$) if $S(x) \subseteq T(x)$ almost everywhere in $\mathcal{X}$.  We denote this as $S \subseteq T$ (or $T \supseteq S$).  If $S(x) \subset T(x)$ almost everywhere in $\mathcal{X}$, then we call $S$ a strict subset of $T$ (and $T$ a strict superset of $S$) and denote this as $S \subset T$ (or $T \supset S)$. 
\end{definition}

\textbf{Learning with a Foreseen Adversary} In learning with a foreseen adversary, the target threat model $T$ is a subset of the source threat model $S$ ($T \subseteq S$).  The learner has access to $S$ and a dataset $D_m$ of $m$ iid samples from the data distribution $\mathcal{D}$.  The learner would like to use a learning algorithm $\mathcal{A}$ for which $f = \mathcal{A}(S, D_m)$ achieves $L_T(f) < \epsilon$ for some small $\epsilon > 0$. The learner cannot compute $L_T(f)$, but can compute $\hat{L}_S(f) \ge \hat{L}_T(f)$.  This setting of learning with a foreseen adversary represents when the adversary is weaker than assumed by the learner and since $L_S(f) \ge L_T(f)$, which means that as long as the learner can achieve $L_S(f) < \epsilon$, then they are guaranteed that $L_T(f) < \epsilon$.

\textbf{Learning with an Unforeseen Adversary} In learning with an unforeseen adversary, the target threat model $T$ is a strict superset of the source threat model $S$ ($T \supset S$).  In this setting, we call $T$ an unforeseen adversary.  The learner has access to $S$ and a dataset $D_m$ of $m$ iid samples from the data distribution $\mathcal{D}$.  The learner would like to use a learning algorithm $\mathcal{A}$ for which $f = \mathcal{A}(S, D_m)$ achieves $L_T(f) < \epsilon$ for some small $\epsilon > 0$.  This setting of learning with an unforeseen adversary represents when the adversary is strictly stronger than assumed by the learner.  Compared to learning with a foreseen adversary, this problem is more difficult since $L_S(f)$ may not be reflective of $L_T(f)$.  By construction $L_T(f) \ge L_S(f)$, but it is unclear how much larger $L_T(f)$ is. When can we guarantee that $L_T(f)$ is close to $L_S(f)$?  We will address this question in the Section \ref{sec:generalizability_def} when we define threat model generalizability and Section \ref{sec:gen_bound} when we provide a bound for $L_T(f) - L_S(f)$.

\mysubsection{Formulating Generalizability with an Unforeseen Adversary}
\label{sec:generalizability_def}
How should we define $\mathcal{A}$ that performs well against an unforeseen adversary?  One way is to have $f = \mathcal{A}(S, D_m)$ achieves small $\hat{L}_S(f)$ (which can be measured by $\mathcal{A}$) while ensuring that $\hat{L}_S(f)$ is close to $L_T(f)$.  This leads us to the following definition for generalization gap.

\begin{definition}[Generalization Gap] For threat models $S$ and $T$, the generalization gap is defined as $L_T(f) - \hat{L}_S(f)$.  We observe that
\begin{equation*}
  L_T(f) - \hat{L}_S(f) = \underbrace{L_T(f) - L_S(f)}_{\text{threat model generalization gap}} + \underbrace{L_S(f) - \hat{L}_S(f)}_{\text{sample generalization gap}}
\end{equation*}
We note that in the special case of learning with a foreseen adversary, $L_T(f) - L_S(f) \le 0$, so 
$L_T(f) - \hat{L}_S(f) \le L_S(f) - \hat{L}_S(f)$
and bounding the generalization gap be achieved by bounding the sample generalization gap, which has been studied by prior works \citep{attias2019improved,raghunathan2019adversarial,chen2020more,yu2021understanding}.
\end{definition}

We would like to ensure that the generalization gap is small with high probability.  We can achieve this by ensuring that both the sample generalization gap and threat model generalization gap are small.  This leads us to define \textit{robust sample generalizability} and \textit{threat model generalizability} which describe conditions necessary for us to expect the respective generalization gaps to be small.  We then combine these generalizability definitions and define \textit{unforeseen adversarial generalizability} which describes the conditions necessary for a learning algorithm to be able to generalize to unforeseen attacks.

\begin{definition}[Robust Sample Generalizability] \label{def:sample_generalizability} A learning algorithm $\mathcal{A}$ robustly $(\epsilon(\cdot), \delta)$-sample generalizes across function class $\mathcal{F}$ on threat model $S$ where $\epsilon: \mathbb{N} \to \mathbb{R}^+$, if for any distribution $\mathcal{D}$
when running $\mathcal{A}$ on 
$m$ iid samples $D_m$ from $\mathcal{D}$, we have
\begin{equation*}
    \mathbb{P}[L_S(\mathcal{A}(S, D_m)) \le \hat{L}_S(\mathcal{A}(S, D_m)) + \epsilon(m)] \ge 1 - \delta
\end{equation*}
\end{definition}

Definition \ref{def:sample_generalizability} implies that any learning algorithm that $(\epsilon(\cdot), \delta)$-robustly sample generalizes across our chosen hypothesis class $\mathcal{F}$ with $\epsilon(m) << 1,\delta << 1$, we can achieve small sample generalization gap with high probability.

We now define generalizability for the threat model generalization gap.  

\begin{definition}[Threat Model Generalizability] \label{def:tm_generalizability} Let $S$ be the source threat model used by the learner.  
A learning algorithm $\mathcal{A}$ $(\epsilon(\cdot, \cdot), 
\delta)$-robustly generalizes to target threat model $T$ where $\epsilon: T \times \mathbb{N} \to \mathbb{R}^+ \cup \{\infty\}$ and $\delta \in [0,1]$ if for any data distribution $\mathcal{D}$ and any training dataset $D_m$ with $m$ iid samples from $\mathcal{D}$, we have: 
$$\mathbb{P}[L_T(\mathcal{A}(S, D_m)) \le L_S(\mathcal{A}(S, D_m)) + \epsilon(T,m)] \ge 1 - \delta$$
\end{definition}

We note that the Definition \ref{def:tm_generalizability} considers generalization to a given $T$, which does not fully account for the unknown nature of $T$, since from the learner's perspective, the learner does not know which threat model it wants $L_T$ to be small for.  We address this in the following definition where we combine Definitions \ref{def:sample_generalizability} and \ref{def:tm_generalizability} and define generalizability to unforeseen adversarial attacks.

\begin{definition}[Unforeseen Adversarial Generalizability]\label{def:uagen} A learning algorithm $\mathcal{A}$ on function class $\mathcal{F}$ with source adversary $S$, $(\epsilon(\cdot, \cdot),\delta)$-robustly generalizes to unforeseen threat models where $\epsilon: N \times \mathbb{N} \to \mathbb{R}^+ \cup \{\infty\}$ if there exists $\epsilon_1, \epsilon_2$ with $\epsilon_1(m) + \epsilon_2(T, m) \le \epsilon(T, m)$ such that $\mathcal{A}$ robustly $(\epsilon_1, \delta)$-sample generalizes and $(\epsilon_2, \delta)$-robustly generalizes to \textit{any} threat model $T$. 
\end{definition}

We remark that in Definition \ref{def:uagen}, $\epsilon$ is a function of $T$, which accounts for differences in difficulty of possible target threat models.  Ideally, we would like $\epsilon(T, m)$ at sufficiently large $m$ to be small across a set of reasonable threat models $T$ (ie. imperceptible perturbations) and expect it to be large (and possibly infinite) for difficult or unreasonable threat models (ie. unbounded perturbations).
\section{A Generalization Bound for Unforeseen Attacks}
\label{sec:gen_bound}
While prior works have proposed bounds on sample generalization gap  \citep{attias2019improved,raghunathan2019adversarial,chen2020more,yu2021understanding}, to the best of our knowledge, prior works have not provided bounds on threat model generalization gap.  In this section, we demonstrate that we can bound the threat model generalization gap in terms of a quantity we define as variation, the expected maximum difference across features learned by the model across the target threat model.  We then show that with the existence of an expansion function, which relates source variation to target variation, any learning algorithm which with high probability outputs a model with small source variation can achieve small threat model generalization gap.

\mysubsection{Relating generalization gap to variation}
We now consider function classes  of the form $\mathcal{F} = \mathcal{G} \circ \mathcal{H}$ where $\forall g \in \mathcal{G}, g: \mathbb{R}^d \to \mathbb{R}^K$ is a top level classifier into $K$ classes and $\forall h \in \mathcal{H}, h: \mathcal{X} \rightarrow \mathbb{R}^d$ is a $d$-dimensional feature extractor.  Since the top classifier $g$ is fixed for a function $f$, if $h(\hat{x})_i, i \in [1...d]$ fluctuates a lot across the threat model $\hat{x} \in T(x)$, then the adversary can manipulate this feature to cause misclassification.  The relation between features and robustness has been analyzed by prior works such as \citep{ilyas2019adversarial, tsipras2019robustness, TB19}.  We now demonstrate that we can bound the threat model generalization gap in terms of a measure of the fluctuation of $h$ across $T$, which we call \textit{variation}.

\begin{definition}[Variation]
The variation of a feature vector $h(\cdot):\mathcal{X} \rightarrow \mathbb{R}^d$ across a threat model $N$ is given by 
$$\mathcal{V}(h, N) =\mathbb{E}_{(x,y) \sim \mathcal{D}} \max_{x_1,x_2 \in N(x)} ||h(x_1) - h(x_2)||_2$$
\end{definition}


\begin{theorem}[Variation-Based Threat Model Generalization Bound] \label{generalization_bound_part1}
Let $S$ denote the source threat model and $\mathcal{D}$ denote the data distribution.  Let $\mathcal{F} = \mathcal{G} \circ \mathcal{H}$ where $\mathcal{G}$ is a class of Lipschitz classifiers with Lipschitz constant upper bounded by $\sigma_{\mathcal{G}}$. Let the loss function be $\rho$-Lipschitz. 
Consider a learning algorithm $\mathcal{A}$ over $\mathcal{F}$ and denote $f = \mathcal{A}(S, D_m) = g \circ h$.  If with probability $1-\delta$ over the randomness of $D_m$, $\mathcal{V}(h, T) \le \epsilon(T, m)$ where $\epsilon: T \times \mathbb{N} \to \mathbb{R}^+ \cup \{\infty\}$, then $\mathcal{A}$ ~($\rho \sigma_{\mathcal{G}}\epsilon(T, m), \delta)$-robustly generalizes from $S$ to $T$.
\end{theorem}

Theorem \ref{generalization_bound_part1} shows we can bound the threat model generalization gap between any source $S$ and unforeseen adversary $T$ in terms of variation across $T$.  With regards to Definition \ref{def:tm_generalizability}, Theorem \ref{generalization_bound_part1} suggests that any learning algorithm over $\mathcal{F}$ that with high probability outputs models with low variation on the target threat model can generalize well to that target.

\mysubsection{Relating source and target variation}
Since the learning algorithm $\mathcal{A}$ cannot use information from $T$, it is unclear how to define such $\mathcal{A}$ that achieves small $\mathcal{V}(h, T)$.  We address this problem by introducing the notion of an expansion function, which relates the source variation (which can be computed by the learner) to target variation.

\begin{definition}[Expansion Function for Variation \citep{ye2021towards}]
\label{def:exp} A function $s: \mathbb{R}^+ \cup \{0\} \rightarrow \mathbb{R}^+ \cup \{0, +\infty\}$ is an expansion function relating variation across source threat model $S$ to target threat model $T$ if the following properties hold:
\vspace{-10pt}
\begin{enumerate}
\itemsep0em 
    \item $s(\cdot)$ is monotonically increasing and $s(x) \ge x, \forall x \ge 0$
    \item $\lim_{x\rightarrow 0^+}s(x) = s(0) = 0$
    \item For all $h$ that can be modeled by function class $\mathcal{F}$, $s(\mathcal{V}(h, S)) \ge \mathcal{V}(h, T)$
\end{enumerate}
\end{definition}
When an expansion function for variation from $S$ to $T$ exists, then we can bound the threat model generalization gap in terms of variation on $S$.  This follows from Theorem \ref{generalization_bound_part1} and Definition \ref{def:exp}.

\begin{corollary}[Source Variation-Based Threat Model Generalization Bound] \label{generalization_bound}
Let $S$ denote the source threat model and $\mathcal{D}$ denote the data distribution.  Let $\mathcal{F} = \mathcal{G} \circ \mathcal{H}$ where $\mathcal{G}$ is a class of Lipschitz classifiers with Lipschitz constant upper bounded by $\sigma_{\mathcal{G}}$. Let the loss function be $\rho$-Lipschitz. Let $T$ be any unforeseen threat model for which an expansion function $s$ from $S$ to $T$ exists. 
Consider a learning algorithm $\mathcal{A}$ over $\mathcal{F}$ and denote $f = \mathcal{A}(S, D_m) = g \circ h$.  If with probability $1-\delta$ over the randomness of $D_m$, $s(\mathcal{V}(h, S)) \le \epsilon(T, m)$ where $\epsilon: T \times \mathbb{N} \to \mathbb{R}^+ \cup \{\infty\}$, then $\mathcal{A}$ ~($\rho \sigma_{\mathcal{G}}\epsilon(T, m), \delta)$-robustly generalizes from $S$ to $T$.
\end{corollary}

Corollary \ref{generalization_bound} allows us to relate generalization across threat models of a model $f=g\circ h$ to $s(\mathcal{V}(h, S))$ instead of $\mathcal{V}(h, T)$. While this expression is still dependent on the target threat model $T$ (since $s$ is dependent on $T$), we can reduce $s(\mathcal{V}(h, S))$ \textit{without knowledge of} $T$ due to the monotonicity of the expansion function.  Thus, provided that an expansion function exists, we can use techniques such as regularization in order to ensure that our learning algorithm actively chooses models with low source variation.  This result leads to the question: when does the expansion function exist?

\mysubsection{When does the expansion function exist?}\label{sec:exp_fun}
We now demonstrate a few cases in which the expansion function exists or does not exist. We begin by providing basic examples of source threat models $S$ and target threat models $T$ 
without constraints on function class.
\begin{proposition} When $S=T$, an expansion function $s$ exists and is given by $s(x) = x$.
\end{proposition}

\begin{proposition} \label{rem:clean} Let $S = \{x\}$, and $T$ be a threat model such that $S \subset T$.  Then, for all feature extractors $h$, we have that $V(h, S) = 0$ while $V(h, T)$ can be greater than 0.  In this case, no expansion function exists such that $s(V(h, S)) \ge V(h, T)$.
\end{proposition}

While we did not consider a constrained function class in the previous two settings, the choice of function class can also impact the existence of an expansion function.  For instance, in the setting of Proposition \ref{rem:clean}, if we constrain $\mathcal{F}$ to only use feature extractors with a constant output, then the expansion function $s(x) = x$ is valid.  We now consider the case where our function class $\mathcal{F}$ uses linear feature extractors and derive expansion functions for $\ell_p$ adversaries.

\begin{theorem}[Linear feature extractors with $\ell_p$ threat model ($p \in \mathbb{N} \cup +\infty$)]\label{thm:linear} Let inputs $x \in \mathbb{R}^n$ and corresponding label $y \in [1...K]$. Consider $S(x) = \{\hat{x} | ~ ||\hat{x} - x||_p \le \epsilon_{1} \}$ and $U(x) = \{\hat{x} | ~ ||\hat{x} - x||_q \le \epsilon_{2} \}$ with $p,q \in \mathbb{N}^+, p,q > 0$. Define target threat model $T(x) = S(x) \cup U(x)$. Consider a linear feature extractor with bounded condition number: $h \in \{Wx + b| W \in \mathbb{R}^{d\times n}, b \in \mathbb{R}^d, \frac{\sigma_{\max}(W)}{\sigma_{\min}(W)} \le B < \infty \}$.  Then, an expansion function exists and is linear.
\end{theorem}
Theorem \ref{thm:linear} demonstrates that in the case of a linear feature extractor a linear expansion function exists for any data distribution from a source $\ell_p$ adversary to a union of $\ell_p$ adversaries.  This result suggests that with a function class using linear feature extractors, we can improve generalization to $\ell_p$ balls with larger radii by using a learning algorithm that biases towards models with small $\mathcal{V}(h, S)$.  We demonstrate this in Appendix \ref{app:gauss} where we experiment with linear models on Gaussian data.  We also provide visualizations of expansion function for a nonlinear model (ResNet-18) on CIFAR-10 in Section \ref{sec:cif_exp_fun}.
\section{Adversarial Training with Variation Regularization}
\label{sec:exp}
Our generalization bound from Corollary \ref{generalization_bound} suggests that learning algorithms that bias towards small source variation can improve generalization to other threat models when an expansion function exists.  In this section, we propose adversarial training with variation regularization (AT-VR) to improve generalization to unforeseen adversaries and evaluate the performance of AT-VR on multiple datasets and model architectures.

\mysubsection{Adversarial training with variation regularization}
To integrate variation into AT, we consider the following training objective:
\begin{equation*}
\label{AT-VR}
    \min_{f \in \mathcal{F}, f = g \circ h} \frac{1}{n} \sum_{i=1}^n [ \underbrace{\max_{x' \in S(x_i)} \ell(f(x'), y_i)}_{\text{empirical adversarial risk}} + \lambda \underbrace{\max_{x', x'' \in S(x_i)} ||h(x') - h(x'')||_2}_{\text{empirical variation}} ]
\end{equation*}
where $\lambda \ge 0$ is the regularization strength.  For the majority of our experiments in the main text, we use the objective of PGD-AT \citep{madry2017towards} as the approximate empirical adversarial risk.  We note that this can be replaced with other forms of AT such as TRADES \citep{zhang2019theoretically}. We can approximate empirical variation by using gradient-based methods.  For example, when $N(x)$ is a $\ell_p$ ball around $x$, we compute the variation term by using PGD to simultaneously optimize over $x_1$ and $x_2$.  We discuss methods for computing variation for other source threat models in Appendix \ref{app:additional_sources}.

\mysubsection{Experimental Setup}
\label{setup}
We investigate the performance of training neural networks with AT-VR on image data for a variety of datasets, architectures, source threat models, and target threat models.  We also combine VR with perceptual adversarial training (PAT) \citep{laidlaw2020perceptual}, the current state-of-the-art for unforeseen robustness, which uses a source threat model based on LPIPS \citep{zhang2018lpips} metric.

\textbf{Datasets} We train models on CIFAR-10, CIFAR-100, \citep{krizhevsky2009learning} and ImageNette \citep{imagenette}.  ImageNette is a 10-class subset of ImageNet \citep{deng2009imagenet}.

\textbf{Model architecture} On CIFAR-10, we train ResNet-18 \citep{he2016deep}, WideResNet(WRN)-28-10 \citep{zagoruyko2016wide}, and VGG-16 \citep{simonyan2015vgg} architectures.  On ImageNette, we train ResNet-18 \citep{he2016deep}.  For PAT-VR, we use ResNet-50.  For all architectures, we consider the feature extractor $h$ to consist of all layers of the NN and the top level classifier $g$ to be the identity function.  We include experiments for when we consider $h$ to be composed of all layers before the fully connected layers in Appendix \ref{app:add_im_features}.

\textbf{Source threat models } Across experiments with AT-VR, we consider 2 different source threat models: $\ell_{\infty}$ perturbations with radius $\frac{8}{255}$ and $\ell_2$ perturbations with radius 0.5.  For PAT-VR, we use LPIPS computed from an AlexNet model \citep{KrizhevskySH17} trained on CIFAR-10. We provide additional details about training procedure in Appendix \ref{app:exp_setup}.  We also provide results for additional source threat models such as StAdv and Recolor in Appendix \ref{app:additional_sources}.

\textbf{Target threat models } We evaluate AT-VR on a variety of target threat models including, $\ell_p$ adversaries ($\ell_{\infty}$, $\ell_2$, and $\ell_1$ adversaries), spatially transformed adversary (StAdv) \citep{XiaoZ0HLS18}, and Recolor adversary \citep{LaidlawF19}.  For StAdv and Recolor threat models, we use the original bounds from \citep{XiaoZ0HLS18} and \citep{LaidlawF19} respectively.  For all other threat models, we specify the bound ($\epsilon$) within the figures in this section.  We also provide evaluations on additional adversaries including Wasserstein, JPEG, elastic, and LPIPS-based attacks in Appendix \ref{app:additional_adv} for CIFAR-10 ResNet-18 models.

\textbf{Baselines }We remark that we are studying the setting where the \textit{learner has already chosen a source threat model} and during testing the model is evaluated on a \textit{strictly larger unknown} target.  Because of this, for AT-VR experiments, we use standard PGD-AT \citep{madry2017towards} (AT-VR with $\lambda=0$) as a baseline.  For PAT-VR experiments, we use PAT (PAT-VR with $\lambda=0$) as a baseline.  We note that VR can be combined with other training techniques such as TRADES \citep{zhang2019theoretically} and provide results in Appendix \ref{app:trades-VR}.

\begin{table}[]
    \centering
\fontsize{8}{9}\selectfont
    \begin{tabular}{@{}cccc|cc|cccc|c@{}}
    \hline
         & & & & & & \multicolumn{4}{c|}{Union with Source} &\\
         \cline{7-10}
         Dataset & Architecture & Source & $\lambda$ & Clean & Source & $\ell_{\infty}$ & $\ell_2$ & StAdv & Re- & Union \\
         & & & & acc & acc &$\epsilon=\frac{12}{255}$& $\epsilon=1$ & &color & all \\
         \hline
         
         CIFAR-10 & ResNet-18 & $\ell_{2}$ & 0 & \textbf{88.49}& 66.65 & 6.44 & 34.72 & 0.76 & 66.52 & 0.33\\
         CIFAR-10 & ResNet-18 & $\ell_{2}$ & 1& 85.21 & \textbf{67.38} & \textbf{13.43} & \textbf{40.74} & \textbf{34.40} & \textbf{67.30} & \textbf{11.77} \\
        
         \hline
         CIFAR-10 & ResNet-18 & $\ell_{\infty}$ & 0 & \textbf{82.83} & 47.47 & 28.09 & \textbf{24.94} & 4.38 &  47.47 &  2.48 \\
         CIFAR-10 & ResNet-18 & $\ell_{\infty}$ & 0.5& 72.91 & \textbf{48.84} & \textbf{33.69} & 24.38 &\textbf{18.62} &\textbf{48.84 }& \textbf{12.59}\\
         
         \hline
         CIFAR-10 & WRN-28-10 & $\ell_{\infty}$ & 0 &  \textbf{85.93}& 49.86 & 28.73 & 20.89 & 2.28 & 49.86 & 1.10 \\
         CIFAR-10 & WRN-28-10 & $\ell_{\infty}$ & 0.7 & 72.73 & \textbf{49.94} & \textbf{35.11 }& \textbf{22.30} & \textbf{25.33 }& \textbf{49.94} & \textbf{14.72 }\\
         \hline
         CIFAR-10 & VGG-16 & $\ell_{\infty}$ & 0 & \textbf{79.67} & 44.36 & 26.14 & 30.82 & 7.31 & 44.36 & 4.35\\
         CIFAR-10 & VGG-16 & $\ell_{\infty}$ & 0.1 & 77.80 & \textbf{45.42} & \textbf{28.41} & \textbf{32.08} & \textbf{10.57} & \textbf{45.42} & \textbf{6.83} \\
                 \hline
         ImageNette & ResNet-18 & $\ell_2$ &  0 &\textbf{88.94} & \textbf{84.99} & 0.00 & 79.08 & 1.27 & 72.15  & 0.00 \\
         ImageNette & ResNet-18 & $\ell_2$ & 1 & 85.22 & 83.08 & \textbf{9.53} & \textbf{80.43} & \textbf{18.04} & \textbf{75.26} & \textbf{6.80}\\

         \hline
         
         ImageNette & ResNet-18 & $\ell_{\infty}$ & 0 & \textbf{80.56} & 49.63 & 32.38 & 49.63 & 34.27 & 49.63 & 25.68\\
         ImageNette & ResNet-18 & $\ell_{\infty}$ & 0.1 & 78.01 & \textbf{50.80} & \textbf{35.57} & \textbf{50.80} & \textbf{42.37} & \textbf{50.80} &\textbf{31.82}  \\
         
         \hline
         CIFAR-100 & ResNet-18 & $\ell_{2}$ & 0 & \textbf{60.92} & 36.01 & 3.98 & 16.90 & 1.80 & 34.87 & 0.40  \\
         CIFAR-100 & ResNet-18 & $\ell_{2}$  & 0.75 & 51.53 & \textbf{38.26} & \textbf{11.47} & \textbf{25.65} & \textbf{5.12} & \textbf{36.96} & \textbf{3.11}\\

         \hline
         CIFAR-100 & ResNet-18 & $\ell_{\infty}$ & 0 & \textbf{54.94} & 22.74 & 12.61 & 14.40 & 3.99 & 22.71 &  2.42 \\
         CIFAR-100 & ResNet-18 & $\ell_{\infty}$  & 0.2 & 48.97 & \textbf{25.04} &\textbf{16.48}  & \textbf{15.82}  & \textbf{4.96} & \textbf{24.95} & \textbf{3.48} \\
         \hline
    \end{tabular}
    \vspace{3pt}
    \caption{Robust accuracy of various models trained at different strengths of VR applied on logits on various threat models.  $\lambda=0$ represents the baseline (standard AT).  The ``source acc" column reports the accuracy on the source attack ($\ell_{\infty}, \epsilon=\frac{8}{255}$ or $\ell_2, \epsilon=0.5$).  For each individual threat model, we evaluate accuracy on a union with the source threat model.  The union all column reports the accuracy on the union across all listed threat models.}
    \vspace{-20pt}
    \label{tab:logits_acc}
\end{table}

\mysubsection{Performance of AT-VR across different imperceptible target threat models} \label{sec:impactofvar}
We first investigate the impact of AT-VR on robust accuracy across different target threat models that are strictly larger than the source threat model used for training.  To enforce this, we evaluate robust accuracy on a target threat model that is the union of the source with a different threat model. For $\ell_{\infty}$ and $\ell_2$ attacks, we measure accuracy using AutoAttack \citep{Croce020}, which reports the lowest robust accuracy out of 4 different attacks: APGD-CE, APGD-T, FAB-T, and Square.  For $\ell_{\infty}$ and $\ell_2$ threat models, we use radius $\epsilon = \frac{12}{255}$ and $\epsilon=1$ for evaluating unforeseen robustness.  We report clean accuracy, source accuracy (robust accuracy on the source threat model), and robust accuracy across various targets in Table \ref{tab:logits_acc}.
 We present results with additional strengths of VR in Appendix \ref{app:additional_lam}.

\textbf{AT-VR improves robust accuracy on broader target threat models. }  We find that overall across datasets, architecture, and source threat model, using AT-VR improves robust accuracy on unforeseen targets but trades off clean accuracy. For instance, we find that on CIFAR-10, our ResNet-18 model using VR improves robustness on the union of all attacks from 2.48\% to 12.59\% for $\ell_\infty$ source and from 0.33\% to 11.77\% for $\ell_2$.  The largest improvement we observe is a 33.64\% increase in robust accuracy for the ResNet-18 CIFAR-10 with $\ell_2$ source model on the StAdv target.

\textbf{AT-VR maintains accuracy on the source compared to standard AT, but trades off clean accuracy. }We find that AT-VR is able to maintain similar source accuracy in comparson to standard PGD AT but consistently trades off clean accuracy.  For example, for WRN-28-10 on CIFAR-10, we find that source accuracy increases slightly with VR (from 49.86\% to 49.94\%), but clean accuracy drops from 85.93\% to 72.73\%.  In Appendix \ref{app:additional_lam}, where we provide results on additional values of regularization strength ($\lambda$), \textit{we find that increasing $\lambda$ generally trades off clean accuracy but improves union accuracy}.  We hypothesize that this tradeoff occurs because VR enforces the decision boundary to be smooth, which may prevent the model from fitting certain inputs well.

\mysubsection{State-of-the-art performance with PAT-VR}
We now combine variation regularization with PAT.  We present results in Table \ref{tab:pat_vr}.

\begin{table}[ht]
    \centering
    \fontsize{8}{9}\selectfont
    \begin{tabular}{@{}cc|c|cccc|c|cc@{}}
    \hline
    Source &$\lambda$ & Clean & $\ell_{\infty}$ & $\ell_2$ & StAdv & Re- & Union & PPGD & LPA \\
          $\epsilon$ & & acc  &$\epsilon=\frac{8}{255}$& $\epsilon=1$ & &color & &  & \\
         
         \hline
         
         \hline
         
         0.5  & 0 & 86.6 & \textbf{38.8} & \textbf{44.3} & 5.8 & 60.8 & 2.1 & 16.2 & 2.2\\
         0.5 & 0.05 & \textbf{86.9} & 34.9 & 40.6 & 9.4 & 64.6 & 3.7 & 21.9 & 2.2\\
         0.5 & 0.1  & 85.1 & 31.4 & 37.1 & \textbf{44.9} & \textbf{80.5} & \textbf{24.9} & \textbf{48.7} & \textbf{29.7} \\
         \hline
         
         1 \tablefootnote{\label{note1}Values taken from \citet{laidlaw2020perceptual}} & 0 & 71.6 & 28.7 & 33.3 & \textbf{64.5} & 67.5 & 27.8 & 26.6 & 9.8\\
         1 & 0.05 & 72.1 & \textbf{29.5}& 34.8 & 59.6 & 69.7 & 28.2 & 56.7 & 18.5\\
         1 & 0.1  & \textbf{72.5} & 29.4 & \textbf{35.1} & 61.8 & \textbf{70.7} & \textbf{28.8} & \textbf{56.9} & \textbf{30.8} \\
         \hline
    \end{tabular}
        \vspace{3pt}

    \caption{Robust accuracy of ResNet-50 models trained using AlexNet-based PAT-VR with $\epsilon=0.5$ and $\epsilon=1$.  $\lambda=0$ corresponds to standard PAT.  The union column reports the accuracy obtained on the union of $\ell_{\infty}$, $\ell_2$, StAdv, and Recolor adversaries.  The PPGD and LPA columns report robust accuracy under AlexNet-based PPGD and LPA attacks with $\epsilon=0.5$.}
    \label{tab:pat_vr}
    \vspace{-20pt}
\end{table}

\textbf{PAT-VR achieves state-of-the-art robust accuracy on AlexNet-based LPIPS attacks (PPGD and LPA).} \citet{laidlaw2020perceptual} observed that LPA attacks are the strongest perceptual attacks, and that standard AlexNet-based PAT with source $\epsilon=1$ can only achieve 9.8\% robust accuracy on LPA attacks with $\epsilon=0.5$.  In comparison, we find that applying variation regularization can significantly improve over performance on LPIPS attacks.  In fact, using variation regularization strength $\lambda=0.1$ while training with $\epsilon=0.5$ can achieve 29.7\% robust accuracy on LPA, while training with $\lambda=0.1$ and $\epsilon=1$ improves LPA accuracy to 30.8\%.  

\textbf{PAT-VR achieves state-of-the-art union accuracy across $\ell_{\infty}$, $\ell_2$, StAdv, and Recolor attacks.} We observe that as regularization strength $\lambda$ increases, union accuracy also increases.  For source $\epsilon=0.5$, we find that union accuracy increases from 2.1\% without variation regularization to 24.9\% with variation regularization at $\lambda=0.1$.  For source $\epsilon=1$, we observe a 1\% increase in union accuracy from $\lambda=0$ to $\lambda=0.1$.  However, this comes at a trade-off with accuracy on specific threat models.  For example, when training with $\epsilon=0.5$, we find that variation regularization at $\lambda=0.1$ trades off accuracy on $\ell_\infty$ and $\ell_2$ sources (7.4\% and 7.2\% drop in robust accuracy respectively), but improves robust accuracy on StAdv attacks from 5.8\% to 44.9\%.  Meanwhile, for $\epsilon=1$, we find that at $\lambda=0.1$, variation regularization trades off accuracy on StAdv to improve accuracy across $\ell_{\infty}$, $\ell_2$, and Recolor threat models.

\textbf{Unlike AT-VR, PAT-VR maintains clean accuracy in comparison to PAT.} We find that PAT-VR generally does not trade off additional clean accuracy in comparison to PAT. In some cases (at source $\epsilon=1$), increasing variation regularization strength can even improve clean accuracy.
\mysubsection{Influence of AT-VR on threat model generalization gap across perturbation size}

\begin{figure}[h]
    \centering
    \includegraphics[width=0.9\textwidth]{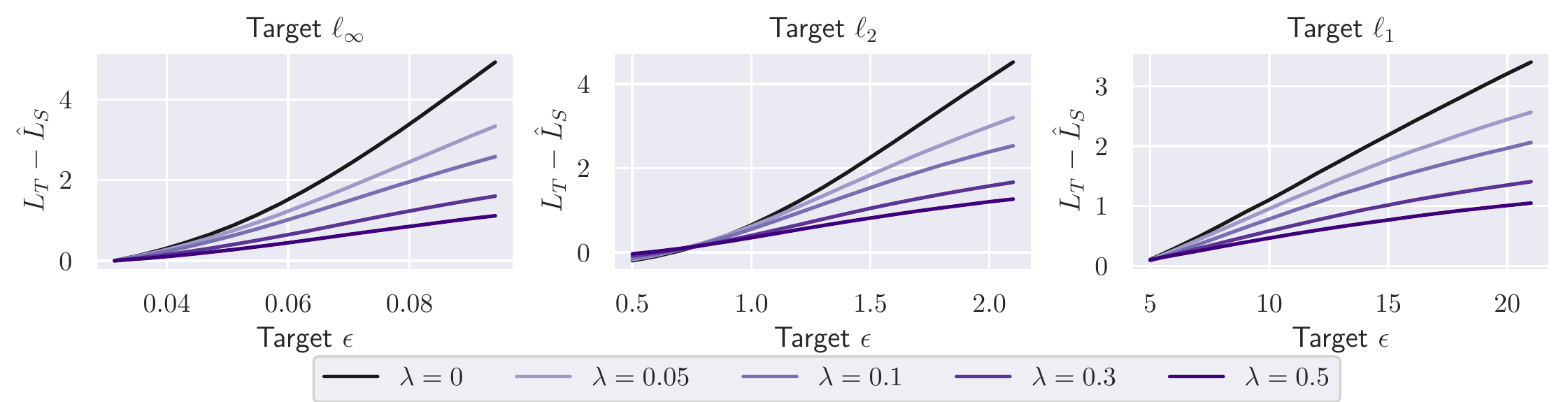}
    \caption{
    Threat model generalization gap of ResNet-18 models on CIFAR-10 trained using AT-VR at regularization strength $\lambda$ measured on target $\ell_p, p=\{\infty, 2, 1\}$ adversarial examples with radius $\epsilon$.  The generalization gap is measured with respect to cross entropy loss. All models are trained with source $\ell_\infty$ perturbations of radius $\frac{8}{255}$. We find that increasing VR strength decreases the generalization gap across $\epsilon$.}
    \label{fig:cif_gen}
    \vspace{-5pt}
\end{figure}

In Section \ref{sec:impactofvar}, we observed that AT-VR improves robust accuracy on a variety of unforeseen target threat models at the cost of clean accuracy.  This suggests that AT-VR makes the change in adversarial loss on more difficult threat models increase more gradually.  In this section, we experimentally verify this by plotting the gap between source and target losses (measured via cross entropy) across different perturbation strengths $\epsilon$ for $\ell_p$ threat models ($p \in \{\infty, 2, 1\}$) for ResNet-18 models trained on CIFAR-10.  We present results for models using AT-VR with $\ell_{\infty}$ source attacks in Figure \ref{fig:cif_gen}.  For these experiments, we generate adversarial examples using APGD from AutoAttack \citep{croce2020reliable}.  We also provide corresponding plots for $\ell_2$ source attacks in Appendix \ref{app:add_lp}.

We find that AT-VR consistently reduces the gap between source and target losses on $\ell_p$ attacks across different target perturbation strengths $\epsilon$. \textit{We observe that this gap decreases as regularization increases across target threat models}.  This suggests that VR can reduce the generalization gap across threat models, making the loss measured on the source threat model better reflect the loss measured on the target threat model, which matches our results from Corollary \ref{generalization_bound}.

\mysubsection{Visualizing the expansion function}\label{sec:cif_exp_fun}
The effectiveness of AT-VR suggests that an expansion function exists between across the different imperceptible threat models tested.  In this section, we visualize the expansion function between $\ell_{\infty}$ and $\ell_2$ source and target pairs for ResNet-18 models on CIFAR-10.  We train a total of 15 ResNet-18 models using PGD-AT with and without VR on $\ell_2$ and $\ell_{\infty}$ source threat models.  We evaluate variation on models saved every 10 epochs during training along with the model saved at epoch with best performance, leading to variation computation on a total of 315 models for each source threat model.  

 \begin{wrapfigure}{r}{0.5\textwidth}
     \centering
     \includegraphics[width=0.45\textwidth]{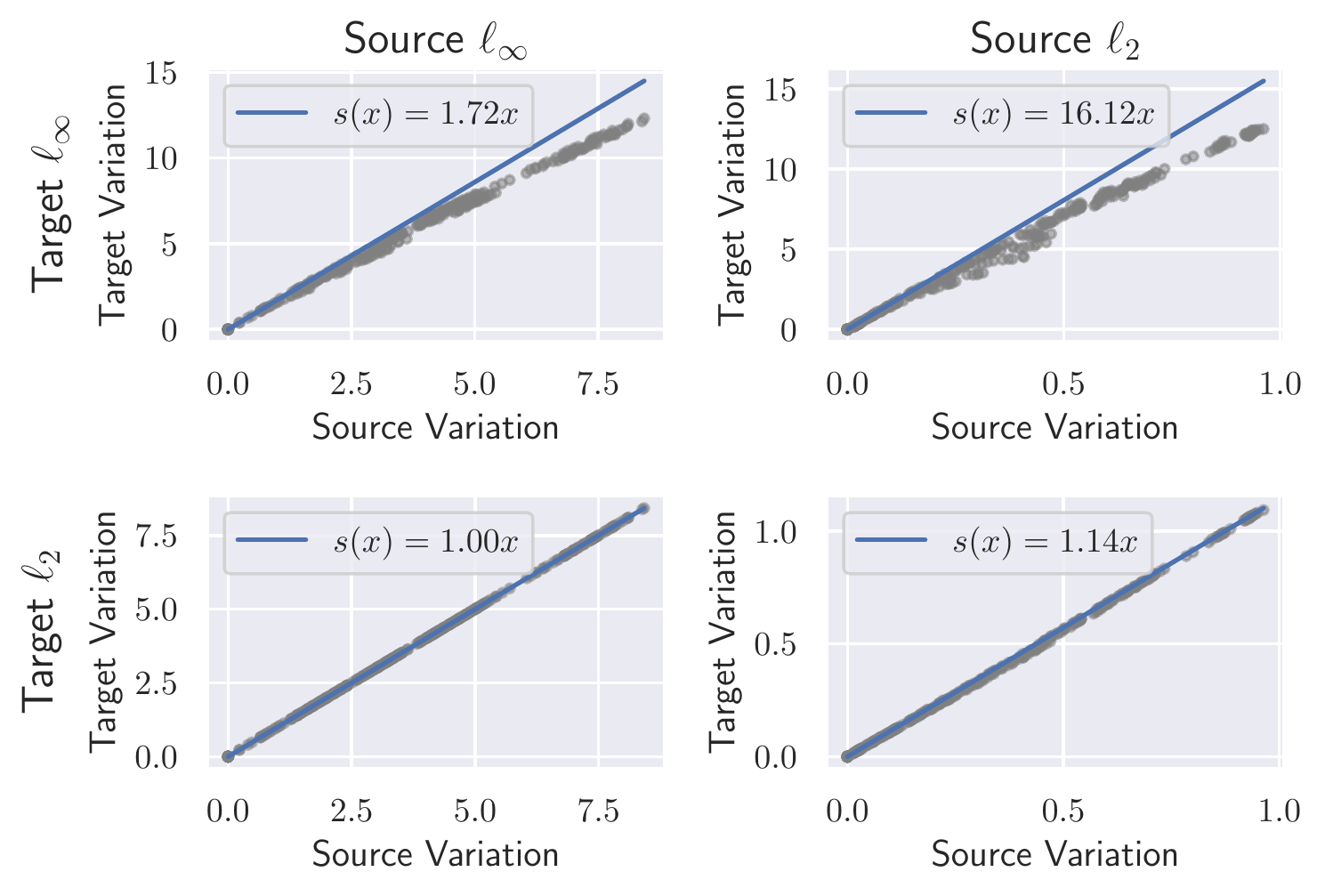}
     \caption{Plots of minimum linear expansion function $s$ shown in blue computed on 315 adversarially trained ResNet-18 models.  Each grey point represents variation measured on the source and target pair.  Variation is computed on the logits.  The two columns represent the source adversary ($\ell_{\infty}$ and $\ell_2$ respectively).  The two rows represent the target adversary ($\ell_{\infty}$ and $\ell_2$ respectively).}
     \label{fig:cifar_exp}
     \vspace{-15pt}
 \end{wrapfigure}
 
We consider 4 cases: (1) $\ell_\infty$ source with $\epsilon = \frac{8}{255}$ to $\ell_{\infty}$ target with $\epsilon = \frac{16}{255}$, (2) $\ell_\infty$ source with $\epsilon = \frac{8}{255}$ to a target consisting of the union of the source with an $\ell_2$ threat model with radius $0.5$, (3) $\ell_2$ source with $\epsilon = 0.5$ to a target consisting of the union of the source with an $\ell_\infty$ threat model with $\epsilon = \frac{8}{255}$, and (4) $\ell_2$ source with $\epsilon = 0.5$ to $\ell_2$ target with $\epsilon=1$. 
In cases (2) and (3), since the target is the union of $\ell_p$ balls, we approximate the variation of the union by taking the maximum variation across both $\ell_p$ balls.  We plot the measured source vs target variation along with the minimum linear expansion function $s$ in Figure \ref{fig:cifar_exp}.

We find that in all cases the distribution of source vs target variation is sublinear, and we can upper bound this distribution with a linear expansion function with relatively small slope.  Recall our finding in Theorem \ref{thm:linear} that for linear models there exists a linear expansion function across $\ell_p$ norms.  We hypothesize that this property also appears for ResNet-18 models because neural networks are piecewise linear.

\vspace{-5pt}
\section{Discussion, Limitations, and Conclusion}
\label{sec:conclusion}
We highlight a limitation in adversarial ML research: the lack of understanding of how robustness degrades when a mismatch in source and target threat models occurs. Our work takes steps toward addressing this problem by formulating the problem of learning with an unforeseen adversary and providing a framework for reasoning about generalization under this setting.  With this framework, we derive a bound for threat model generalization gap in terms of variation and use this bound to design an algorithm, adversarial training with variation regularization (AT-VR).  We highlight several limitations of our theoretical results: (1) the bounds provided can be quite loose and may not be good predictors of unforeseen loss, (2) while we show that an expansion function between $\ell_p$ balls exists for linear models, it is unclear if that is the case for neural networks.  Additionally, we highlight several limitations of AT-VR: (1) its success depends on the existence of an expansion function, (2) VR trades off additional clean accuracy and increases computational complexity of training. Further research on improving source threat models and the accuracy and efficiency of adversarial training algorithms can improve the performance of AT-VR.  Finally, we note that in some applications, such as defending against website fingerprinting \citep{rahman2020mockingbird} and bypassing facial recognition based surveillance \citep{shan2020fawkes},  adversarial examples are used for good, so improving robustness against adversarial examples may consequently hurt these applications.

\begin{ack}
We would like to thank Tianle Cai, Peter Ramadge, and Vincent Poor for their feedback on this work. This work was supported in part by the National Science Foundation under grants CNS-1553437 and CNS-1704105, the ARL’s Army Artificial Intelligence Innovation Institute (A2I2), the Office of Naval Research Young Investigator Award, the Army Research Office Young Investigator Prize, Schmidt DataX award, and Princeton E-ffiliates Award.  This material is also based upon work supported by the National Science Foundation Graduate Research Fellowship under Grant No. DGE-2039656.  Any opinions, findings, and conclusions or recommendations expressed in this material are those of the author(s) and do not necessarily reflect the views of the National Science Foundation.
\end{ack}

\bibliography{refs}
\bibliographystyle{icml2022}

\section*{Checklist}


\begin{enumerate}

\item For all authors...
\begin{enumerate}
  \item Do the main claims made in the abstract and introduction accurately reflect the paper's contributions and scope?
    \answerYes{}
  \item Did you describe the limitations of your work?
    \answerYes{See Section \ref{sec:conclusion} for a discussion of limitations.}
  \item Did you discuss any potential negative societal impacts of your work?
    \answerYes{See Section \ref{sec:conclusion}.  In some applications, such as evading website fingerprinting, adversarial examples are helpful so improving defenses against them reduces their benefit in these applications.}
  \item Have you read the ethics review guidelines and ensured that your paper conforms to them?
    \answerYes{}
\end{enumerate}

\item If you are including theoretical results...
\begin{enumerate}
  \item Did you state the full set of assumptions of all theoretical results?
    \answerYes{All assumptions are specified in the theorem statements (See Theorem \ref{generalization_bound_part1} and Theorem \ref{thm:linear})}
        \item Did you include complete proofs of all theoretical results?
    \answerYes{See Appendix \ref{app:proof_gen}, \ref{app:proof_gen2}, \ref{app:proof_lin}}
\end{enumerate}

\item If you ran experiments...
\begin{enumerate}
  \item Did you include the code, data, and instructions needed to reproduce the main experimental results (either in the supplemental material or as a URL)?
    \answerYes{We provide our code in the supplemental material}
  \item Did you specify all the training details (e.g., data splits, hyperparameters, how they were chosen)?
    \answerYes{See Appendix \ref{app:exp_setup} for training details}
        \item Did you report error bars (e.g., with respect to the random seed after running experiments multiple times)?
    \answerYes{We provide error bars for ResNet-18 models on CIFAR-10 with $\ell_{\infty}$ source in Appendix \ref{app:seeds}, but not on other experiments because of the high computational cost of adversarial training.}
        \item Did you include the total amount of compute and the type of resources used (e.g., type of GPUs, internal cluster, or cloud provider)?
    \answerYes{See Appendix \ref{app:exp_setup}}
\end{enumerate}

\item If you are using existing assets (e.g., code, data, models) or curating/releasing new assets...
\begin{enumerate}
  \item If your work uses existing assets, did you cite the creators?
    \answerYes{We use publicly available datasets such as CIFAR-10, CIFAR-100, and ImageNette as well as existing network architectures.  These are cited in Section \ref{setup}}
  \item Did you mention the license of the assets?
    \answerYes{See Section \ref{setup}}
  \item Did you include any new assets either in the supplemental material or as a URL?
    \answerNA{}
  \item Did you discuss whether and how consent was obtained from people whose data you're using/curating?
    \answerNA{}
  \item Did you discuss whether the data you are using/curating contains personally identifiable information or offensive content?
    \answerNA{}
\end{enumerate}

\item If you used crowdsourcing or conducted research with human subjects...
\begin{enumerate}
  \item Did you include the full text of instructions given to participants and screenshots, if applicable?
    \answerNA{}
  \item Did you describe any potential participant risks, with links to Institutional Review Board (IRB) approvals, if applicable?
    \answerNA{}
  \item Did you include the estimated hourly wage paid to participants and the total amount spent on participant compensation?
    \answerNA{}
\end{enumerate}

\end{enumerate}


\newpage
\appendix
\addcontentsline{toc}{section}{Appendix} 
\part{Appendix} 
\parttoc 


\section{Discussion of Related Works} \label{app:comparison}
\paragraph{Comparison to \citet{ye2021towards}}  \citet{ye2021towards} derive a generalization bound for the domain generalization problem in terms of variation across features.  They define variation as $\mathcal{V}_\rho(\phi, \mathcal{D}) = \max_{y\in \mathcal{Y}} \sup_{e,e' \in \mathcal{D}} \rho(\mathbb{P}(\phi(X_e) | y), \mathbb{P}(\phi(X_{e'})|y))$ where $\mathcal{D}$ is a set of training distributions, $\phi: \mathcal{X} \rightarrow \mathbb{R}$ is a function that maps the input to a 1-D feature, $\rho$ is a symmetric distance metric for distributions, and $X_e$ denotes inputs from domain $e$.  In comparison, we address the problem of learning with an unforeseen adversary and define unforeseen adversarial generalizability.  Our generalization bound using variation is an instantiation of our generalizability framework.   Additionally, we define a different measure of variation for threat models $\mathcal{V}(h, N) = \mathbb{E}_{(x, y) \sim \mathcal{D}} \max_{x_1, x_2 \in N(x)} ||h(x_1), -h(x_2)||_2$, which allows us to use it as regularization during training.

\textbf{Comparison to \citet{laidlaw2020perceptual}}  \citet{laidlaw2020perceptual} proposes training with perturbations of bounded LPIPS \citep{zhang2018lpips} since LPIPS metric is a better approximation of human perceptual distance than $\ell_p$ metrics.  In their proposed algorithm, perceptual adversarial training (PAT), they combine standard adversarial training with adversarial examples generated via their LPIPS bounded attack method.  In terms of terminology introduced in our paper, the \citet{laidlaw2020perceptual} improve the choice of source threat model while using an existing learning algorithm (adversarial training).  Meanwhile, our work takes the perspective of having a fixed source threat model and improving the learning algorithm.  This allows us to combine our approach with various source threat models including the attacks used by \citet{laidlaw2020perceptual} in PAT (see Appendix \ref{app:additional_sources}).

\textbf{Comparison to \citet{stutz2020confidence} and \citet{chen2022revisiting}} \citet{stutz2020confidence} and \citet{chen2022revisiting} address the problem of unforeseen attacks by adding a reject option in order to reject adversarial examples generated with a larger perturbation than used during training.  These techniques introduce a modified adversarial training objective that maximizes accuracy on perturbations within the source threat model and maximize rejection rate of large perturbations.  In comparison, we look at the problem of improving robustness on larger threat models instead of rejecting adversarial examples from larger threat models. Our algorithm AT-VR actively tries to find a robust model that minimizes our generalization bound without abstaining on any inputs.

\textbf{Comparison to \citet{Croce020}} \citet{Croce020} prove that certified robustness against $\ell_1$ and $\ell_{\infty}$ bounded perturbations implies certified robustness against attacks generated with any $\ell_p$ ball.  The size of this $\ell_p$ certified radius is the radius of the largest $\ell_p$ ball that can be contained in the convex hull of the $\ell_1$ and $\ell_{\infty}$ balls for which the model is certifiably robust.  In our work, we are primarily interested in empirical robustness on target threat models that are supersets of the source threat model used.  We demonstrate that by using variation regularization, we can improve robust performance on unforeseen threat models (including larger $\ell_p$ perturbations, StAdv, and Recolor) even when our learning algorithm optimizes for robust models on a single $\ell_p$ ball.

\textbf{Comparison to other forms of regularization for adversarial robustness } Prior works in adversarial training propose regularization techniques enforcing feature consistency to improve the performance of adversarial training.  For example, TRADES adversarial training\citep{zhang2019theoretically} uses a regularization term in the objective to reduce trade-off between clean accuracy and robust accuracy compared with PGD adversarial training.  This regularization term takes the form: $\lambda \max_{\hat{x} \in S(x)} \ell(f(x)f(\hat{x}))$.  Our variation regularization differs from TRADES since we regularize $\ell_2$ distance between extracted features.

Another regularization is adversarial logit pairing (ALP) \citep{kannan2018adversarial}.  This regularization is also $\ell_2$ based; in ALP, an adversarial example is first generated via $x' = \max_{\hat{x} \in S(x)} \ell(f(\hat{x}), y)$ and the $\ell_2$ distance between the logits of this adversarial example and the original image $\lambda||f(x') - f(x)||_2$ is added to the training objective.  ALP can be thought of as a technique to make the logits of adversarial examples close to the logits of clean images.  Variation regularization ($\lambda \max_{x', x''\in S(x)} ||f(x') - f(x'')||_2$) differs from ALP since it encourages the logits of any pair of images (not only adversarial examples) that lie within the source threat model to have similar features and does not use information about the label of the image.

\citet{jin2020manifold} propose a regularization technique motivated by concepts from manifold regularization.  Their regularization is computed with randomly sampled maximal perturbations $p \in \{\pm \epsilon\}^d$ and applied to standard training.  Their regularization includes 2 terms, one which regularizes the hamming distance between the ReLU masks of the network for $x - p$ and $x + p$ across inputs $x$, and the second which regularizes the squared $\ell_2$ distance between the network output on $x - p$ and $x + p$ ($||f(x + p) - f(x - p)||_2^2$).  They demonstrate that using both regularization terms with $\epsilon \in [2, 8]$ leads to robustness on $\ell_{\infty}$, $\ell_2$, and Wasserstein attacks.  In comparison, our variation regularization is motivated from the perspective of generalization across threat models.  We use use smaller values of $\epsilon$ in conjunction with adversarial training and regularize worst case $\ell_2$ distance between logits of any pair of examples within the source threat model.

         
         

%
\section{Proofs of Theorems and Analysis of Bounds}
\subsection{Proof of Theorem \ref{generalization_bound_part1}} \label{app:proof_gen}
\begin{proof}
By definition of expected adversarial risk, we have that for any $f \in \mathcal{F}$
$$L_T(f) - L_S(f) = \mathbb{E}_{(x, y) \sim \mathcal{D}} (\max_{x_1 \in T(x)} \ell(f(x_1), y) \\ - \max_{x_2 \in S(x)} \ell(f(x_2), y))$$
$$\le \mathbb{E}_{(x, y) \sim \mathcal{D}} \max_{x_1 \in T(x), x_2 \in S(x)} \left(\ell(f(x_1), y) - \ell(f(x_2), y)\right) $$
By $\rho$-Lipschitzness of $\ell$
$$\le \mathbb{E}_{(x, y) \sim \mathcal{D}}\max_{x_1 \in T(x), x_2 \in S(x)} \rho ||f(x_1) - f(x_2) ||_2$$
$$= \rho\mathbb{E}_{(x, y) \sim \mathcal{D}} \max_{x_1 \in T(x), x_2 \in S(x)} ||g(h(x_1)) - g(h(x_2)) ||_2$$
By $\sigma_{\mathcal{G}}$ Lipschitzness:
$$\le \rho \sigma_{\mathcal{G}} \mathbb{E}_{(x, y) \sim \mathcal{D}} \max_{x_1 \in T(x), x_2 \in S(x)}  ||h(x_1) - h(x_2) ||_2$$
Since $S \subseteq T$:
$$\le \rho\sigma_{\mathcal{G}}\mathbb{E}_{(x, y) \sim \mathcal{D}} \max_{x_1, x_2 \in T(x)}  ||h(x_1) - h(x_2) ||_2$$
$$ = \rho\sigma_{\mathcal{G}}\mathcal{V}(h, T)$$
Note that our learning algorithm $\mathcal{A}$ over $\mathcal{F}$ outputs a classifier with $\mathcal{V}(h, T) \le \epsilon(T, m)$ with probability $1-\delta$.  Combining this with the above bound gives:
$$\mathbb{P}[L_T(f) - L_S(f) \le \rho\sigma_{\mathcal{G}}\epsilon(T, m)] \ge 1 - \delta$$
\end{proof}

\subsection{Impact of choice of source and target}\label{app:proof_gen2}
In Theorem \ref{generalization_bound_part1}, we demonstrated a bound on threat model generalization gap that does not depend on source threat model, so this bound does not allow us to understand what types of targets are easier to generalize given a source threat model.  In this section, we will introduce a tighter bound in terms of Hausdorff distance, which takes both source and target threat model into account.
\begin{definition}[Directed Hausdorff Distance] Let $A, B \subset X$ and let $d: X \times X \rightarrow \mathbb{R}$ be a distance metric.  The Hausdorff distance from $A$ to $B$ based on $d$ is given by 
$$H_{d}(A, B) = \max_{x_1 \in A} \min_{x_2 \in B} d(x_1, x_2)$$
\end{definition}
Intuitively what this measures is, if we were to take every point in $A$ and project it to the nearest point in set $B$, what is the maximum distance projection?  We can then derive a generalization bound in terms of Hausdorff distance based on feature space distance.
\begin{theorem}[Threat Model Generalization Bound with Hausdorff Distance] \label{bound_hausdorff} Let $S$ denote the source threat model and $\mathcal{D}$ denote the data distribution.  Let $\mathcal{F} = \mathcal{G} \circ \mathcal{H}$.  Let $\mathcal{G}$ be a class of Lipschitz classifiers, where the Lipschitz constant is upper bounded by $\sigma_{\mathcal{G}}$.  Let $\ell$ be a $\rho$-Lipschitz loss function with respect to the 2-norm.  Then, for any target threat model $T$ with $S \subseteq T$.  Let $d_h :=  ||h(x_1) - h(x_2)||_2$ be the distance between feature extracted by the model.  Then,
$$L_T(f) - L_S(f) \le \rho \sigma_{\mathcal{G}} \mathbb{E}_{(x, y) \sim \mathcal{D}} [H_{d_h}(T(x), S(x))]$$
\end{theorem}
\begin{proof}
By definition of expected adversarial risk,
$$L_T(f) - L_S(f) = \mathbb{E}_{(x, y) \sim \mathcal{D}} (\max_{x_1 \in T(x)} \ell(f(x_1), y) \\ - \max_{x_2 \in S(x)} \ell(f(x_2), y))$$
Note that since we are subtracting the max across $S(x)$, this expression is upper bounded by any choice of $\hat{x} \in S(x)$.  Thus, we can choose $\hat{x}$ so that $\ell(f(\hat{x}), y)$ is close to $\max_{x_1 \in T(x)} \ell(f(x_1), y$.  This gives us:
$$\le \mathbb{E}_{(x, y) \sim \mathcal{D}} (\max_{x_1 \in T(x)} \min_{\hat{x} \in S(x)} \ell(f(x_1), y)  - \ell(f(\hat{x}), y))$$

By $\rho$-Lipschitzness of $\ell$
$$\le \mathbb{E}_{(x, y) \sim \mathcal{D}}\max_{x_1 \in T(x)} \min_{x_2 \in S(x)} \rho ||f(x_1) - f(x_2) ||_2$$
$$= \rho\mathbb{E}_{(x, y) \sim \mathcal{D}} \max_{x_1 \in T(x)} \min_{x_2 \in S(x)} ||g(h(x_1)) - g(h(x_2)) ||_2$$
By $\sigma_{\mathcal{G}}$ Lipschitzness:
$$\le \rho \sigma_{\mathcal{G}} \mathbb{E}_{(x, y) \sim \mathcal{D}} \max_{x_1 \in T(x)} \min_{x_2 \in S(x)}  ||h(x_1) - h(x_2) ||_2$$
Since $S \subseteq T$:
$$\le \rho\sigma_{\mathcal{G}}\mathbb{E}_{(x, y) \sim \mathcal{D}} \max_{x_1 \in T(x)} \min_{x_2 \in S(x)}  ||h(x_1) - h(x_2) ||_2$$
$$ = \rho\sigma_{\mathcal{G}}\mathbb{E}_{(x, y) \sim \mathcal{D}}[H_{d_h}(T(x), S(x))]$$
\end{proof}
We note that this bound is tighter than the variation based bound and goes to 0 when $S = T$.  Since this bound also depends on both $S$ and $T$, we can also see that the ``difficulty" of a target $T$ with respect to a chosen source threat model $S$ can be measured through the directed Hausdorff distance from $T(x)$ to $S(x)$.

\subsection{Proof of Theorem \ref{thm:linear}} \label{app:proof_lin}
\begin{lemma}[Variation Upper Bound for $\ell_p$ threat model, $p \in \mathbb{N} \cup +\infty$]
\label{linear_upper}
Let inputs $x \in \mathbb{R}^n$ and corresponding label $y \in [1...K]$.  Let the adversarial constraint be given by $T(x) = \{\hat{x} | ~ ||\hat{x} - x||_p \le \epsilon_{\max} \}$
Let $h$ be a linear feature extractor: $h \in \{Wx + b| W \in \mathbb{R}^{d\times n}, b \in \mathbb{R}^d\}$.  Then, variation is upper bounded by 
\[
\mathcal{V}_p(h, T) \le \begin{cases}  2\epsilon_{\max} n^{\frac{1}{2} - \frac{1}{p}}\sigma_{\max}(W) & p \ge 2 \\
2\epsilon_{\max} \sigma_{\max}(W) & p = 1, p = 2 
\end{cases}
\]
\end{lemma}
\begin{proof}
$$\mathcal{V}(h, T) = \mathbb{E}_{(x, y) \sim D} \max_{x_1, x_2 \in T(x)} ||h(x_1) - h(x_2)||_2$$
$$= \mathbb{E}_{(x, y) \sim D} \max_{x_1, x_2 \in T(x)} ||W(x_1 - x_2)||_2$$
\begin{equation}
\label{lem1_step}
    \le \mathbb{E}_{(x, y) \sim D} \max_{x_1, x_2 \in T(x)} \sigma_{\max}(W)||x_1 - x_2||_2
\end{equation}
Consider the case where $p > 2$.  Then, by Hölder's inequality:
$$\le \mathbb{E}_{(x, y) \sim D} \max_{x_1, x_2 \in T(x)} n^{\frac{1}{2} - \frac{1}{p}}\sigma_{\max}(W)||x_1 - x_2||_p$$
$$\le 2\epsilon_{\max} n^{\frac{1}{2} - \frac{1}{p}}\sigma_{\max}(W)$$
When $p = 1$ or $p=2$, then from \ref{lem1_step}, we have:
$$\le \mathbb{E}_{(x, y) \sim D} \max_{x_1, x_2 \in T(x)}\sigma_{\max}(W)||x_1 - x_2||_p$$
$$\le 2 \epsilon_{\max}\sigma_{\max}(W)$$
\end{proof}

\begin{lemma}[Variation lower bound for $\ell_p$ threat model, $p \in \mathbb{N} \cup +\infty$]
\label{linear_lower}
Let inputs $x \in \mathbb{R}^n$ and corresponding label $y \in [1...K]$. Let the adversarial constraint be given by $T(x) = \{\hat{x} | ~ ||\hat{x} - x||_p \le \epsilon_{\max} \}$. 
Let $h$ be a linear feature extractor: $h \in \{Wx + b| W \in \mathbb{R}^{d\times n}, b \in \mathbb{R}^d\}$.  Then, variation is lower bounded by
\[\mathcal{V}_p(h, T) \ge \begin{cases} 2\epsilon_{\max} \sigma_{\min}(W) & p \ge 2 \\
\frac{2\epsilon_{\max}}{\sqrt{n}} \sigma_{\min}(W) & p = 1 \end{cases}
\]
where $\sigma_{\min}(W)$ denotes the smallest singular value of $W$.
\end{lemma}
\begin{proof}
$$\mathcal{V}(h, T) = \mathbb{E}_{(x, y) \sim D} \max_{x_1, x_2 \in T(x)} ||W(x_1 - x_2)||_2$$
\begin{equation}
\label{lem2_step}
    \ge \sigma_{\min}(W) \mathbb{E}_{(x, y) \sim D} \max_{x_1, x_2 \in T(x)} ||x_1 - x_2||_2
\end{equation}
Then, for $p \ge 2$:
$$\ge \sigma_{\min}(W)\mathbb{E}_{(x, y) \sim D} \max_{x_1, x_2 \in T(x)} ||x_1 - x_2||_p$$
$$=  2\epsilon_{\max} \sigma_{\min}(W)$$
For $p=1$ from \ref{lem2_step}, we have:
$$\ge \frac{1}{\sqrt{n}}\sigma_{\min}(W)\mathbb{E}_{(x, y) \sim D} \max_{x_1, x_2 \in T(x)} ||x_1 - x_2||_1$$
$$= \frac{2\epsilon_{\max}}{\sqrt{n}} \sigma_{\min}(W)$$
\end{proof}            

\begin{lemma}[$\ell_p$ threat models with larger radius, $p \in \mathbb{N} \cup +\infty$]\label{cor:larger_rad} Let inputs $x \in \mathbb{R}^n$ and corresponding label $y \in [1...K]$. Let $S(x) =  \{\hat{x} | ~ ||\hat{x} - x||_p \le \epsilon_{1} \}$ and $T(x) = \{\hat{x} | ~ ||\hat{x} - x||_p \le \epsilon_2 \}$ where $\epsilon_2 \ge \epsilon_1$.
Consider a linear feature extractor with bounded condition number: $h \in \{Wx + b| W \in \mathbb{R}^{d\times n}, b \in \mathbb{R}^d, \frac{\sigma_{\max}(W)}{\sigma_{\min}(W)} \le B < \infty \}$.  Then a valid expansion function is given by:

\[s_p(z) = \begin{cases} \sqrt{n}B\frac{\epsilon_2}{\epsilon_1}z & p = 1 \\
B\frac{\epsilon_2}{\epsilon_1}z & p = 2 \\
n^{\frac{1}{2} - \frac{1}{p}}B\frac{\epsilon_2}{\epsilon_1} z & p > 2

\end{cases}\]
\end{lemma}
\begin{proof}
By Lemma \ref{linear_upper} and Lemma \ref{linear_lower}, $\mathcal{V}(h, T) \le s(\mathcal{V}(h, S))$.  Additionally, it is clear that $s_p$ satisfies properties of expansion function.
\end{proof}

\begin{lemma}[Variation upper bound for the union of $\ell_p$ and $\ell_q$ threat models ($p, q \in \mathbb{N} \cup +\infty$)] \label{lemma_p2q}Let inputs $x \in \mathbb{R}^n$ and corresponding label $y \in [1...K]$. Consider $T_1(x) = \{\hat{x} | ~ ||\hat{x} - x||_p \le \epsilon_{1} \}$ and $T_2(x) = \{\hat{x} | ~ ||\hat{x} - x||_q \le \epsilon_{2} \}$. Define adversarial constraint $T = T_1 \cup T_2$. 
Let $h$ be a linear feature extractor: $h \in \{Wx + b| W \in \mathbb{R}^{d\times n}, b \in \mathbb{R}^d\}$.  Let $v(p, h, T)$ be defined as
\[
v(p, \epsilon, W) = \begin{cases}  2\epsilon n^{\frac{1}{2} - \frac{1}{p}}\sigma_{\max}(W) & p \ge 2 \\
2\epsilon \sigma_{\max}(W) & p = 1, p = 2 
\end{cases}
\]
where $\sigma_{\max}(W)$ denotes the largest singular value of $W$.

Then variation is upper bounded by:
\[
\mathcal{V}_{(p, \epsilon_1), (q, \epsilon_2)}(h, T) \le \max (v(p, \epsilon_1, W), v(q, \epsilon_2, W))
\]
\end{lemma}
\begin{proof}
$$\mathcal{V}(h, T) = \mathbb{E}_{(x, y) \sim D} \max_{x_1, x_2 \in T(x)} ||W(x_1 - x_2)||_2$$
$$\le \mathbb{E}_{(x, y) \sim D} \max_{x_1, x_2 \in T(x)} \sigma_{\max}(W)||x_1 - x_2||_2$$
Since $T = T_1 \cup T_2$, $\max_{x_1, x_2 \in T(x)} ||x_1 - x_2||_2$ can be upper bounded by the diameter of the hypersphere that contains both $T_1$ and $T_2$.  We can compute this diameter by taking the max out of the diameter of the hypersphere containing $T_1$ and the diameter of the hypersphere containing $T_2$.  This was computed in proof of Lemma \ref{linear_upper} to bound the case of a single $\ell_p$ norm.  Thus,
\[
\mathcal{V}_{(p, \epsilon_1), (q, \epsilon_2)}(h, T) \le \max (v(p, \epsilon_1, W), v(q, \epsilon_2, W))
\]
where the expression for $v$ follows from the result of Lemma \ref{linear_upper}.
\end{proof}

\begin{lemma}[$\ell_p$ to union of $\ell_p$ and $\ell_q$ threat model ($p, q \in \mathbb{N} \cup +\infty$) ] \label{cor:lp2lq}Let inputs $x \in \mathbb{R}^n$ and corresponding label $y \in [1...K]$. Consider $S(x) = \{\hat{x} | ~ ||\hat{x} - x||_p \le \epsilon_{1} \}$ and $U(x) = \{\hat{x} | ~ ||\hat{x} - x||_q \le \epsilon_{2} \}$. Define target threat model $T = S \cup U$. Consider a linear feature extractor with bounded condition number: $h \in \{Wx + b| W \in \mathbb{R}^{d\times n}, b \in \mathbb{R}^d, \frac{\sigma_{\max}(W)}{\sigma_{\min}(W)} \le B < \infty \}$.  Then a valid expansion function is given by:

\[s_{p,q}(z) = \begin{cases}  \sqrt{n}B\frac{\max(\epsilon_{2}, \epsilon_1)}{\epsilon_1}& p = 1, q = 2 \\
 \sqrt{n}B\frac{\max(n^{\frac{1}{2} - \frac{1}{q}}\epsilon_{2}, \epsilon_1)}{\epsilon_1} & p = 1, q > 2 \\
 B\frac{\max(\epsilon_{2}, \epsilon_1)}{\epsilon_1} & p = 2, q = 1 \\
 B\frac{\max(n^{\frac{1}{2} - \frac{1}{q}} \epsilon_{2}, \epsilon_1)}{\epsilon_1} & p = 2, q > 2 \\
 B\frac{\max(\epsilon_{2}, n^{\frac{1}{2} - \frac{1}{p}}\epsilon_1)}{\epsilon_1} & p > 2, q \le 2 \\
 B\frac{\max(n^{\frac{1}{2} - \frac{1}{q}}\epsilon_{2}, n^{\frac{1}{2} - \frac{1}{p}}\epsilon_1)}{\epsilon_1} & p > 2, q > 2 

\end{cases}\]
\label{cor:larger_radius}
\end{lemma}
\begin{proof}
By Lemma \ref{lemma_p2q} and Lemma \ref{linear_lower}, $\mathcal{V}(h, T) \le s(\mathcal{V}(h, S))$.  Additionally, it is clear that $s_p$ satisfies properties of expansion function.
\end{proof}

\begin{proof}[Proof of Theorem \ref{thm:linear}] Directly follows from Lemma \ref{cor:larger_rad} and Lemma \ref{cor:lp2lq}.
\end{proof}

\subsection{How well can empirical expansion function predict loss on the target threat model for neural networks?}
Using the empirical expansion function slopes from Figures \ref{fig:cifar_exp} and \ref{fig:lp_to_stadv}, we compute the expected cross entropy loss (with softmax) on the target threat model via Corollary \ref{generalization_bound}.  We provide the predicted and true target losses in Table \ref{tab:predvstrueloss_linf} for $\ell_{\infty}, \epsilon=\frac{8}{255}$ source threat model and \ref{tab:predvstrueloss_l2} for $\ell_2, \epsilon=0.5$.

\begin{table}[h]
\centering
\begin{tabular}{@{}c|rr|rr|r@{}}
\hline
Target      & Source Variation & Source Loss & Predicted Target loss & True Target Loss & Gap   \\
\hline
$\ell_{\infty}, \epsilon=\frac{16}{255}$ & 4.90                   & 0.93              & 12.85                                 & 2.44             & 10.41 \\
$\ell_{2}, \epsilon=0.5$      & 4.90                   & 0.93              & 7.86                                  & 0.93             & 6.93  \\
StAdv $\epsilon=0.05$    & 4.90                   & 0.93              & 9.87                                  & 5.13             & 4.74  \\
\hline
$\ell_{\infty}, \epsilon=\frac{16}{255}$ & 0.98                   & 1.26              & 3.64                                  & 1.76             & 1.88  \\
$\ell_{2}, \epsilon=0.5$      & 0.98                   & 1.26              & 2.64                                  & 1.27             & 1.37  \\
StAdv $\epsilon=0.05$      & 0.98                   & 1.26              & 3.05                                  & 2.11             & 0.94 \\
\hline
\end{tabular}
\caption{Predicted and measured losses on multiple target threat models for ResNet-18 model trained on CIFAR-10 with $\ell_{\infty}, \epsilon=8/255$ source threat model.}
\label{tab:predvstrueloss_linf}
\end{table}

\begin{table}[h]
\centering
\begin{tabular}{@{}c|rr|rr|r@{}}
\hline
Target      & Source Variation & Source Loss & Predicted Target loss & True Target Loss & Gap   \\
\hline
$\ell_{\infty}, \epsilon=\frac{8}{255}$                 & 0.78                                       & 0.64                                  & 18.52                                                     & 2.65                                 & 15.87                   \\
$\ell_{2}, \epsilon=1$                       & 0.78                                       & 0.64                                  & 1.91                                                      & 1.93                                 & 0.02                    \\
StAdv  $\epsilon=0.05$               & 0.78                                       & 0.64                                  & 16.51                                                     & 12.26                                & 4.25                    \\
\hline
$\ell_{\infty}, \epsilon=\frac{8}{255}$                 & 0.20                                       & 0.85                                  & 5.38                                                      & 1.77                                 & 3.61                    \\
$\ell_{2}, \epsilon=1$                       & 0.20                                       & 0.85                                  & 1.16                                                      & 1.51                                 & 0.35                    \\
StAdv $\epsilon=0.05$                & 0.20                                       & 0.85                                  & 4.88                                                      & 2.10                                 & 2.78                    \\
\hline
\end{tabular}
\caption{Predicted and measured losses on multiple target threat models for ResNet-18 model trained on CIFAR-10 with $\ell_2, \epsilon=0.5$ source threat model.}
\label{tab:predvstrueloss_l2}
\end{table}

In general, we find that for models with smaller variation (those trained with variation regularization), the loss estimate using the slope from the expansion function generally improves.  In the case where the target threat model is Linf, we believe that the large gap between predicted and true loss for the unregularized model stems from the fact that we model the expansion function with a linear model.  From Figure \ref{fig:cifar_exp}, we can see that at larger values of source variation the linear model for expansion function becomes an increasingly loose upper-bound.  Improving the model for expansion function (ie. using a log function) may reduce this gap.

\section{Additional Experimental Setup Details} \label{app:exp_setup}
\paragraph{Variation Computation Algorithm for AT-VR}
We provide the algorithm we use to compute variation regularization for $\ell_p$ source adversaries in Algorithm \ref{alg:var_comp}.
\begin{algorithm}
\caption{Variation Regularization Computation for $\ell_p$ ball}\label{alg:var_comp}
\SetKwInOut{Input}{Input}
\SetKwInOut{Notations}{Notations}
\SetKwInOut{Output}{Output}

\Input{$x \in \mathcal{X}$, $\ell_p$ radius $\epsilon$, $n$ number of steps, $\alpha$ step size, feature extractor $h$}
\Notations{$\mathcal{U}$ denotes the uniform distribution of dimension of the input, $\prod_{\ell_p, \epsilon}$ denotes projection onto $\ell_p$ ball of radius $\epsilon$}
\Output{Variation $v \in \mathbb{R}$}

$x_1 \gets \prod_{\ell_p, \epsilon} (x + \mathcal{U}(-\epsilon, \epsilon))$ \tcp*{randomly initialize $x_1$}
$x_2 \gets \prod_{\ell_p, \epsilon} (x + \mathcal{U}(-\epsilon, \epsilon))$ \tcp*{randomly initialize $x_1$}
\For{$i = 1...n$}{
    $v \gets ||h(x_1) - h(x_2)||_2$ \tcp*{compute objective}
    $x_1 \gets \prod_{\ell_p, \epsilon} (x_1 + \alpha \nabla_{x_1}v)$ \tcp*{single step of PGD to optimize $x_1$}
    $x_2 \gets \prod_{\ell_p, \epsilon} (x_2 + \alpha \nabla_{x_2} v)$ \tcp*{single step of PGD to optimize $x_2$}}
$v \gets ||h(x_1) - h(x_2)||_2$ \tcp*{compute final variation}
\Return{$v$}
\end{algorithm}

\paragraph{Variation Computation for PAT-VR} \citet{laidlaw2020perceptual} propose an algorithm called Fast Lagrangian Perceptual Attack (Fast-LPA) to generate adversarial examples for training using PAT.  We make a slight modification of the Fast-LPA algorithm, replacing the loss in the original optimization objective with the variation objective, to obtain Fast Lagrange Perceptual Variation (Fast-LPV).  The explicit algorithm for Fast-LPV is provided in Algorithm \ref{alg:fastlpv}.  We use the variation obtained from Fast-LPV for training with PAT-VR.

\begin{algorithm}
\caption{Fast Lagrangian Perceptual Variation (Fast-LPV)} \label{alg:fastlpv}
\SetKwInOut{Input}{Input}
\SetKwInOut{Notations}{Notations}
\SetKwInOut{Output}{Output}

\Input{feature extractor $h(\cdot)$, LPIPS distance $d(\cdot, \cdot)$, input $x$, label $y$, bound $\epsilon$, number of steps $n$}
\Output{variation $v$}
$x_1 \gets x + 0.01 * \mathcal{N}(0.1)$ \tcp*{randomly initialize $x_1$}
$x_2 \gets x + 0.01 * \mathcal{N}(0.1)$ \tcp*{randomly initialize $x_2$}

\For{$i= 1...n$}{
    $\tau \gets 10^{i/n}$ \tcp*{$\tau$ increases exponentially}
    $obj \gets ||h(x_1) - h(x_2)||_2 - \tau (\max(0, d(x_1,x) - \epsilon) + \max(0, d(x_2,x) - \epsilon))$ \tcp*{obj represents optimization objective}
    $\Delta_1 \gets \frac{\nabla_{x_1}[obj]}{||\nabla_{x_1}[obj]||_2} $ \tcp*{compute normalized gradient wrt $x_1$}
    $\Delta_2 \gets \frac{\nabla_{x_2}[obj]}{ ||\nabla_{x_2}[obj]||_2}$\tcp*{compute normalized gradient wrt $x_2$}
    $\eta \gets \epsilon * (0.1)^{i/n}$ \tcp*{step size $\eta$ decays exponentially}
    $m_1 \gets d(x_1, x_1 + 0.1 \Delta_1)/ 0.1$ \tcp*{derivative of $d$ in direction of $\Delta_1$}
    $m_2 \gets d(x_2, x_2 + 0.1 \Delta_2)/ 0.1$ \tcp*{derivative of $d$ in direction of $\Delta_2$}
    $x_1 \gets x_1 + (\eta/ m)\Delta_1$ \tcp*{take a step of size $\eta$ in LPIPS distance}
    $x_2 \gets x_2 + (\eta/ m)\Delta_2$ \tcp*{take a step of size $\eta$ in LPIPS distance}
    }
$v \gets ||h(x_1) - h(x_2)||_2$\;
\Return{$v$}
\end{algorithm}

\paragraph{Additional Experimental Setup Details for AT-VR} We run all training on a NVIDIA A100 GPU.  For all datasets and architectures, we perform PGD adversarial training \citep{madry2017towards} and add variation regularization to the objective. For all datasets, train with normalized inputs.  On ImageNette, we normalize using ImageNet statistics and resize all images to $224 \times 224$. We train models on seed 0 for 200 epochs with SGD with initial learning rate of 0.1 and decrease the learning rate by a factor of 10 at the 100th and 150th epoch.  We use 10-step PGD and train with $\ell_{\infty}$ perturbations of radius $\frac{8}{255}$ and $\ell_2$ perturbations with radius 0.5.  For $\ell_{\infty}$ perturbations we use step size of $\alpha=\frac{2}{255}$ while for $\ell_2$ perturbations we use step size of 0.075.  We use the same settings for computing the variation regularization term.  For all models, we evaluate performance at the epoch which achieves the highest robust accuracy on the test set.

\paragraph{Additional Experimental Setup Details for PAT-VR} We build off the official code repository for PAT and train ResNet-50 models on CIFAR-10 with AlexNet-based LPIPS threat model with bound $\epsilon = 0.5$ and $\epsilon=1$.  For computing AlexNet-based LPIPS we use the CIFAR-10 AlexNet pretrained model provided in the PAT official code repository.  We train all models for 100 epochs and evaluate the model saved at the final epoch of training.  To match evaluation technique as \citet{laidlaw2020perceptual}, we evaluate with $\ell_{\infty}$ attacks, $\ell_2$ attacks, StAdv, and Recolor with perturbation bounds $\frac{8}{255}$, 1, 0.05, 0.06 respectively.  Additionally, we present accuracy measured using AlexNet-based PPGD and LPA attacks \citep{laidlaw2020perceptual} with $\epsilon=0.5$.

\section{Experiments for linear models on Gaussian data} \label{app:gauss}
In Section \ref{sec:exp_fun}, we demonstrated that for a linear feature extractor, the expansion function exists, so decreasing variation across an $\ell_p$ source adversarial constraint should improve generalization to other $\ell_p$ constraints.  We now verify this experimentally by training a linear model for binary classification on isotropic Gaussian data. 
\subsection{Experimental Setup}
\paragraph{Data generation }We sample data from 2 isotropic Gaussians with covariance $\sigma^2I_n$ where $I_n$ denotes the $n \times n$ identity matrix.  For class 0, we sample from a Gaussian with mean $\theta_0 = (0.25, 0, 0, ..., 0)^T$, and for class 1, we sample from a Gaussian with mean $\theta_1 = (0.75, 0, 0, ..., 0)^T$ and clip all samples to range $[0, 1]$ to simulate image data.  We sample 1000 points per class. We vary $\sigma \in \{0.125, 0.25\}$ and $n \in \{25, 100\}$.  Since our generalization bound considers only threat model generalization gap and not sample generalization gap, we evaluate the models using the same data samples as used during training for the bulk of our experiments, but we provide results on a separate test set for one setting in Appendix \ref{app:gauss_sep_test}.  
 
\paragraph{Model architecture }We train a model consisting of a linear feature extractor and linear top level classifier: $f = g \circ h$ where $h(x) = Wx + b_1$ where $W \in \mathbb{R}^{n \times d}$, $b_1 \in \mathbb{R}^d$ and $g(x) = Ax + b_2$ where $A \in \mathbb{R}^{d \times 2}$, $b_2 \in \mathbb{R}^2$.  For our experiments, we use $d \in \{5, 25\}$.
 
\paragraph{Source Threat Models}  Across experiments we use 2 different source threat models: $\ell_{\infty}$ perturbations with radius $0.01$ and $\ell_2$ perturbations with radius 0.01.
 
\paragraph{Training Details }We perform AT-VR with adversarial examples generated using APGD \citep{croce2020reliable} for 200 epochs.   We train models using SGD with learning rate of 0.1 and momentum of 0.9.  We use cross entropy loss during both training and evaluation.  For variation regularization, we use 10 steps for optimization and use step size $\epsilon / 9$ where $\epsilon$ is the radius of the $\ell_p$ ball used during training/evaluation.

\subsection{Visualizing the expansion function for Gaussian data} 
 In Section \ref{sec:exp_fun}, we demonstrated that with a linear feature extractor, for any dataset there exists a linear expansion function.  We now visualize the true expansion function for Gaussian data with a linear function class across 4 different combinations of input dimension $n$, Gaussian standard deviation $\sigma$, and feature dimension $d$.  We set our source threat model to be $\ell_p, p\in \{\infty, 2\}$ with perturbation size of $0.01$.  We set our target threat model to be $\ell_p, p\in \{\infty, 2\}$ perturbation size of $0.05$ and compute source variation and target variation of 100 randomly sampled $h$.  We sample parameters of $h$ from a standard normal distribution.  We plot the linear expansion function with minimum slope in Figure \ref{fig:gauss_exp}.  We find that across all settings we can find a linear expansion function that is a good fit.  Additionally, we find that the slope of this linear expansion function stays consistent across changes in $\sigma$ and $d$.  We find that input dimension $n$ can influence the expansion function which matches; for example, the slope of the expansion function for source $\ell_2$ to target $\ell_{\infty}$ increases from $\sim 21$ to $39.09$.  This matches our results in Lemma \ref{cor:larger_rad} and Lemma \ref{cor:lp2lq} where our computed expansion function depends on $n$.
 
 \begin{figure}[h]
     \centering
     \includegraphics[width=0.45\textwidth]{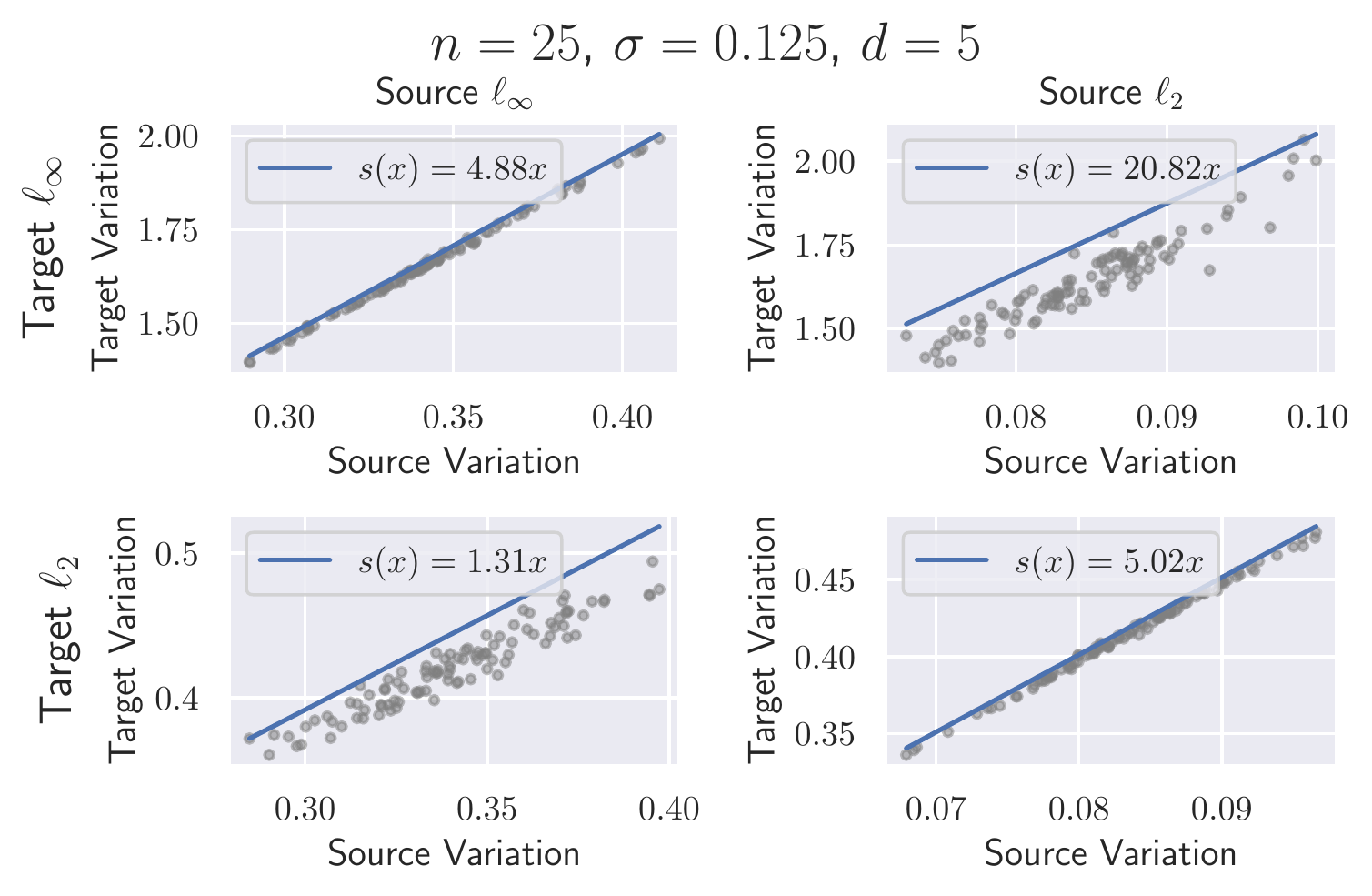}
     \includegraphics[width=0.45\textwidth]{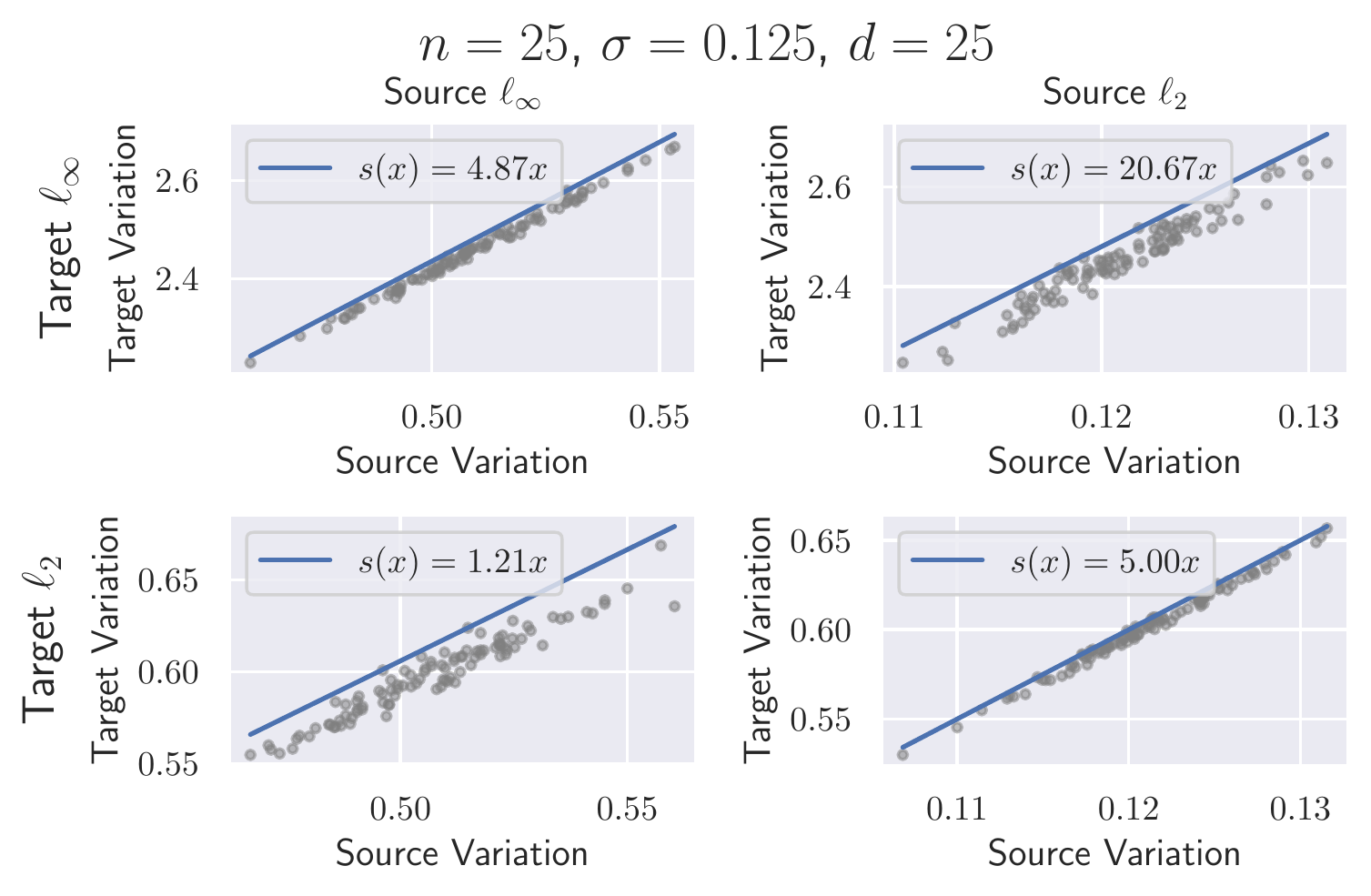}
     
     \includegraphics[width=0.45\textwidth]{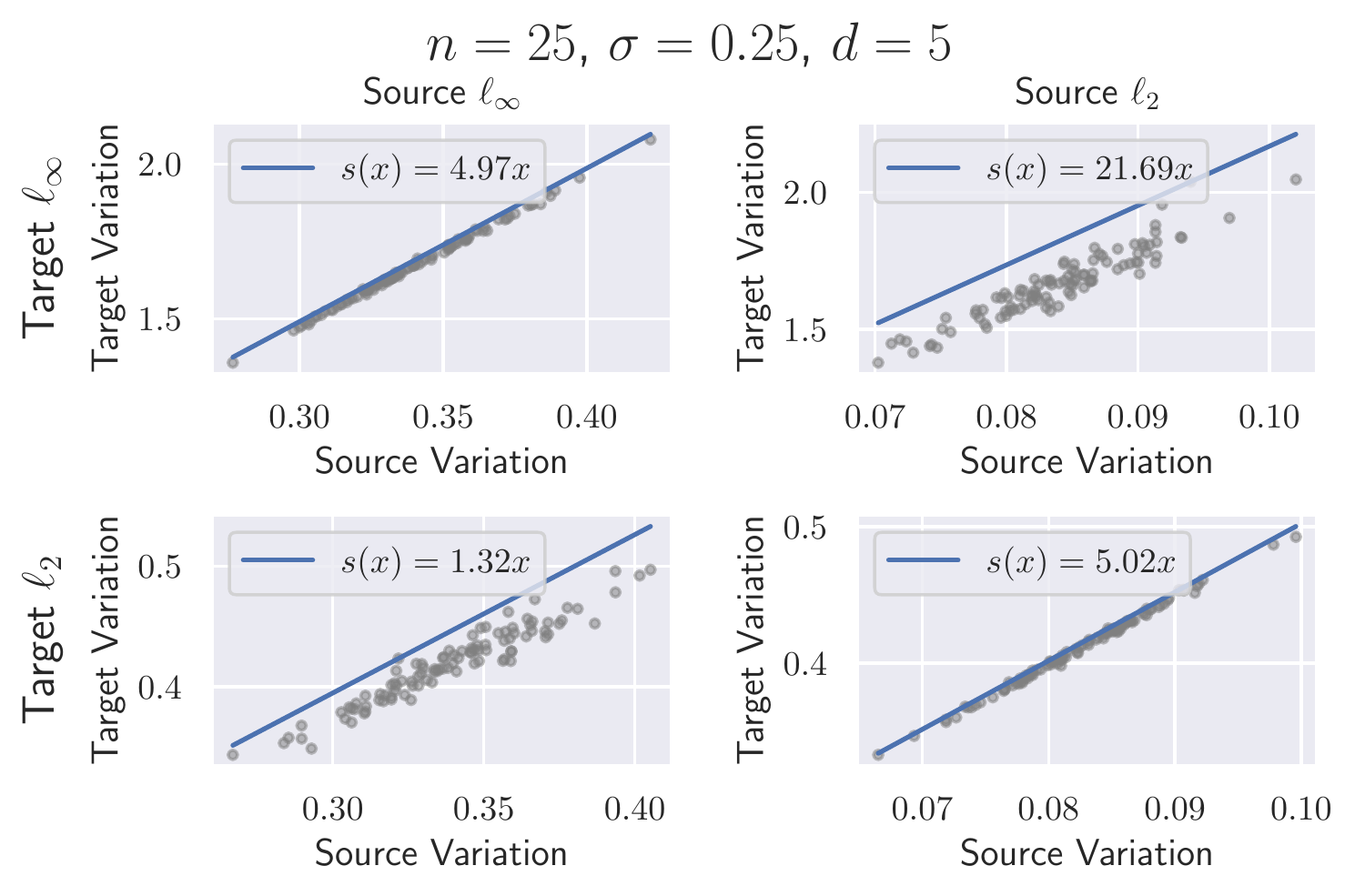}
     \includegraphics[width=0.45\textwidth]{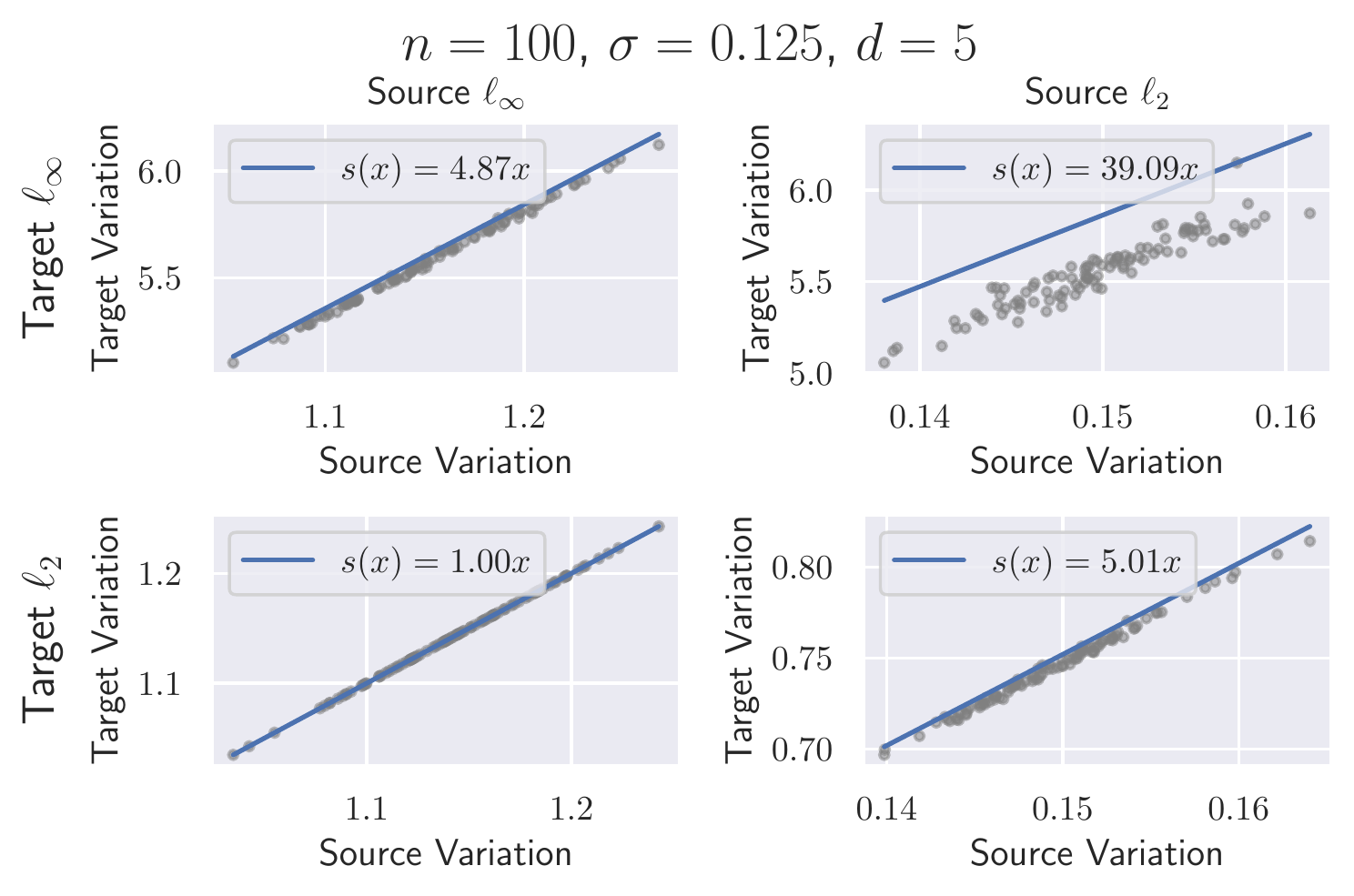}
     \caption{Plots of minimum linear expansion function $s$ shown in blue computed on 100 randomly initialized linear feature extractors across 4 different combinations of input dimension $n$, Gaussian standard deviation $\sigma$, and feature dimension $d$.  Each grey point represents the variation of a single model measured across the source and target.  Columns represent the source threat model ($\ell_{\infty}$ and $\ell_2$ with $\epsilon=0.01$) while rows represent the target threat model ($\ell_{\infty}$ and $\ell_2$ with $\epsilon=0.05$).}
     \label{fig:gauss_exp}
 \end{figure}

\subsection{Generalization curves}
We plot the threat model generalization curves for varied settings of input dimension $n$, feature extractor output dimension $d$, and Gaussian standard deviation $\sigma$ in Figure \ref{fig:gauss_gen_abl}.  We find that across all settings, applying variation regularization leads to smaller generalization gap across values of $\ell_p$ radius $\epsilon$.

\begin{figure}[ht]
    \centering
    \includegraphics[width=0.45\textwidth]{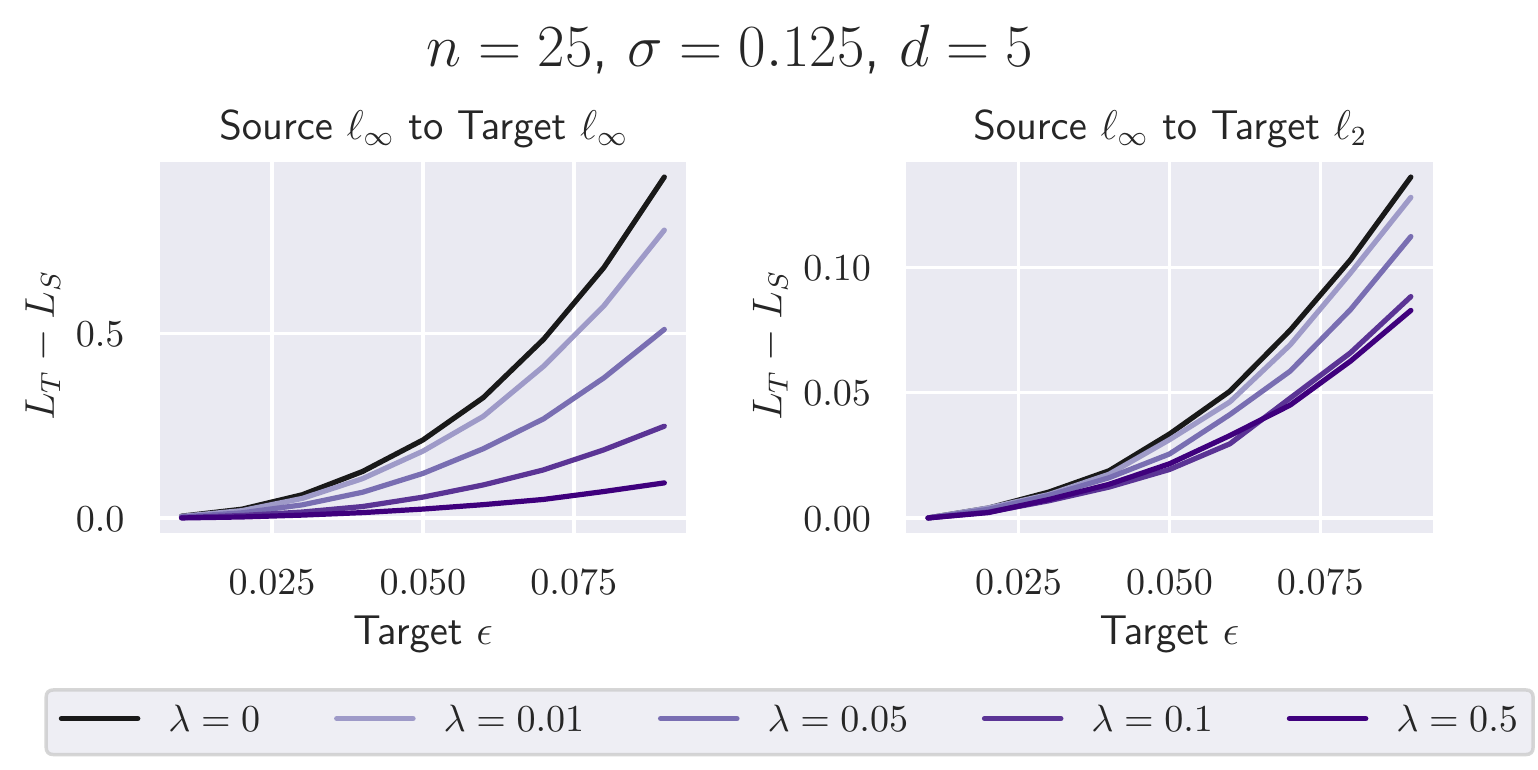}
    \includegraphics[width=0.45\textwidth]{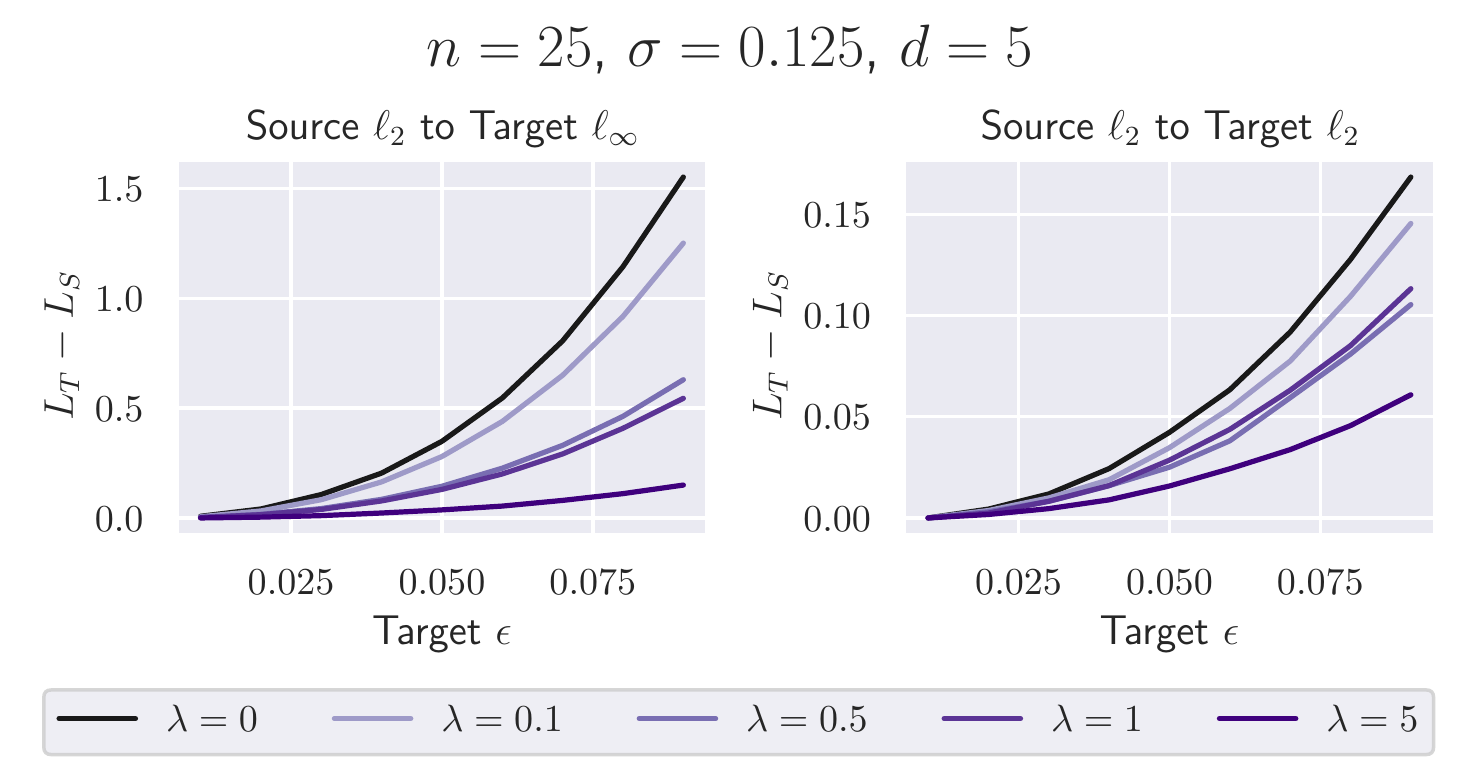}
    \includegraphics[width=0.45\textwidth]{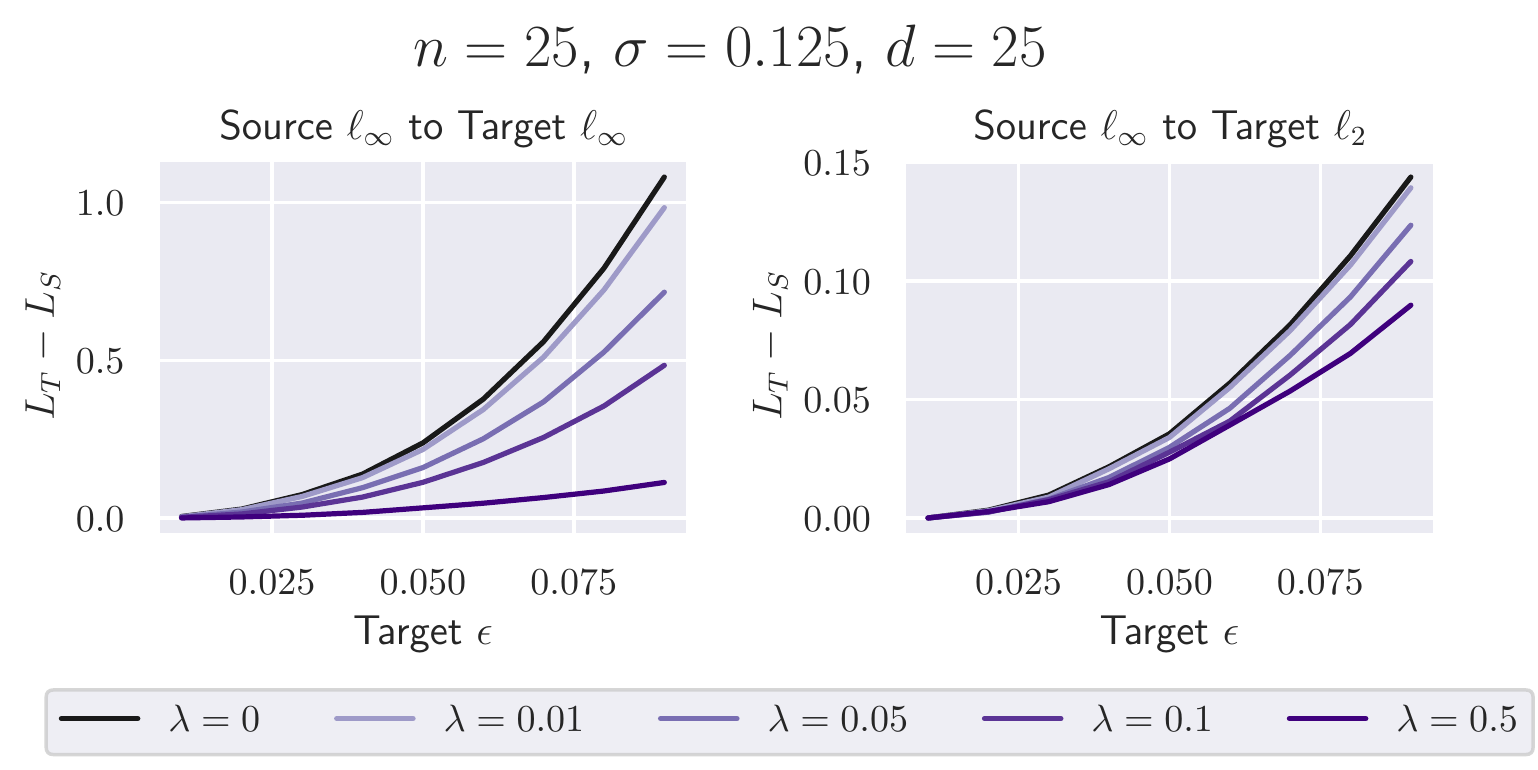}
    \includegraphics[width=0.45\textwidth]{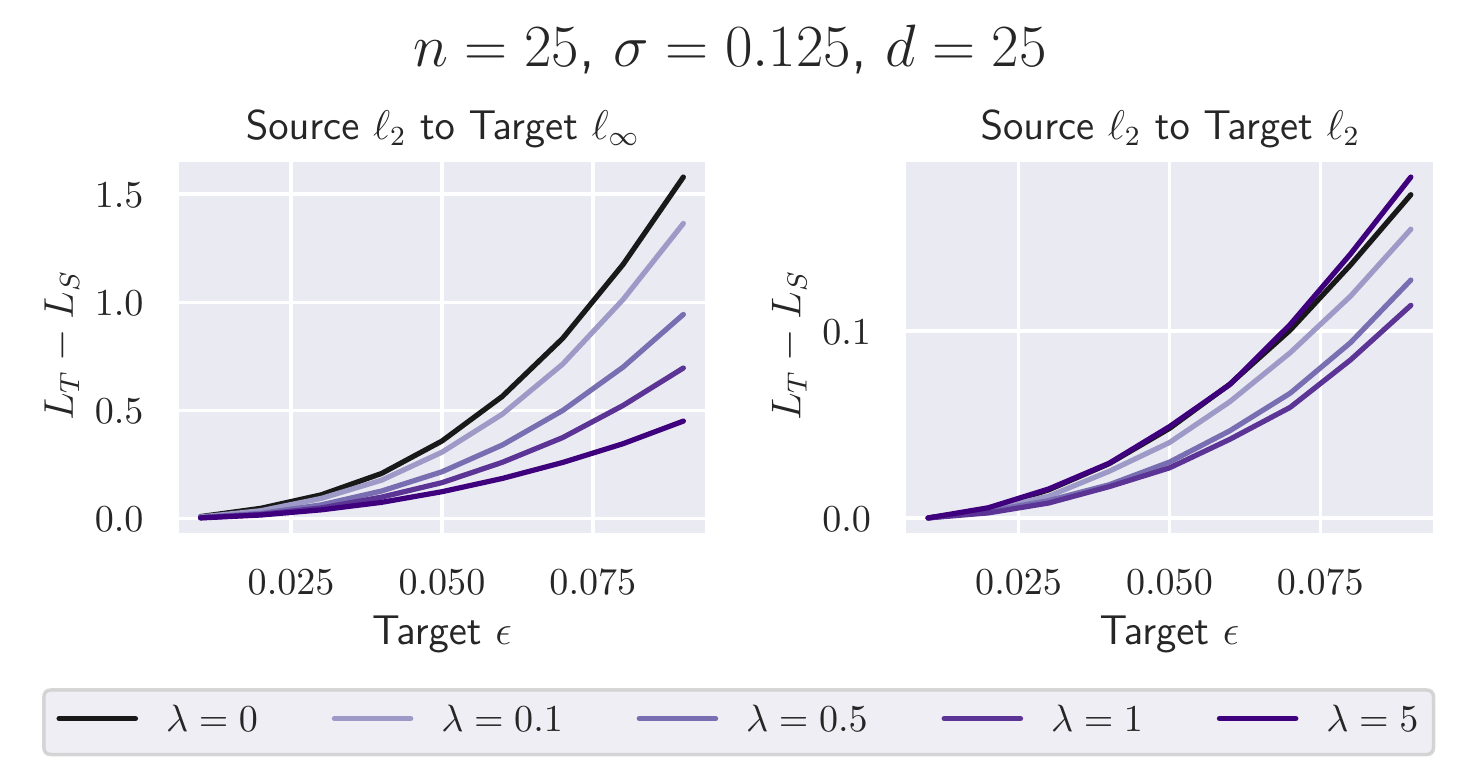}
    
    \includegraphics[width=0.45\textwidth]{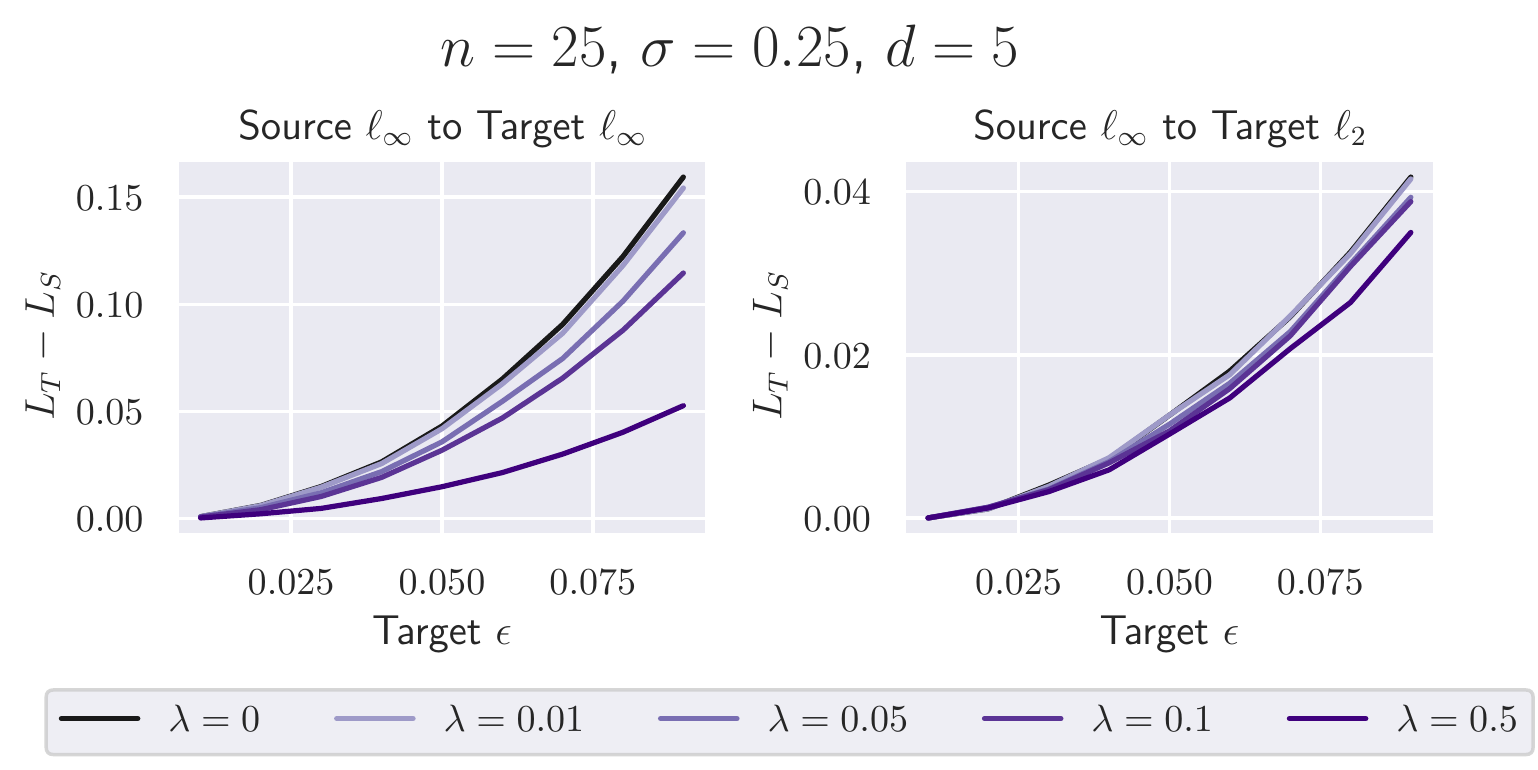}
    \includegraphics[width=0.45\textwidth]{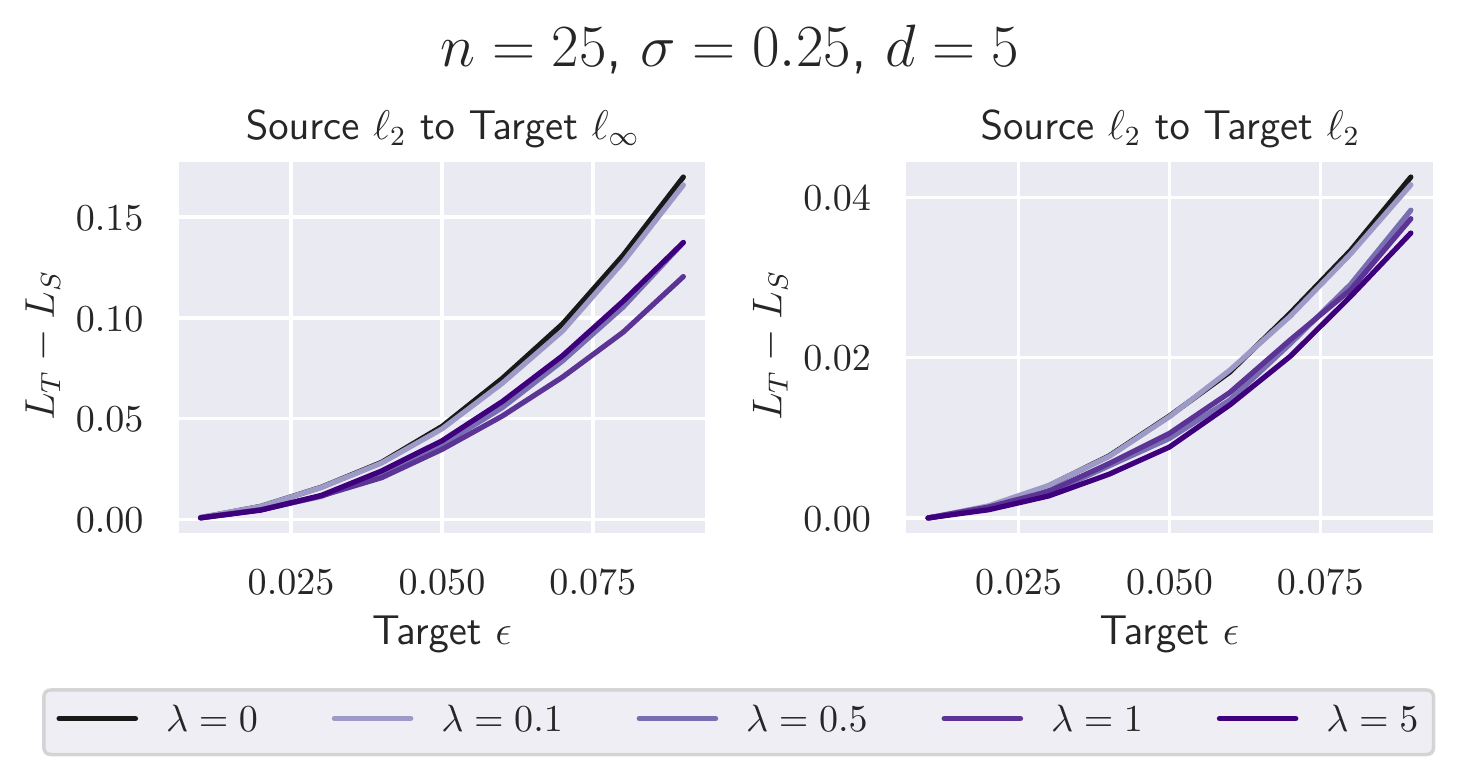}
    \includegraphics[width=0.45\textwidth]{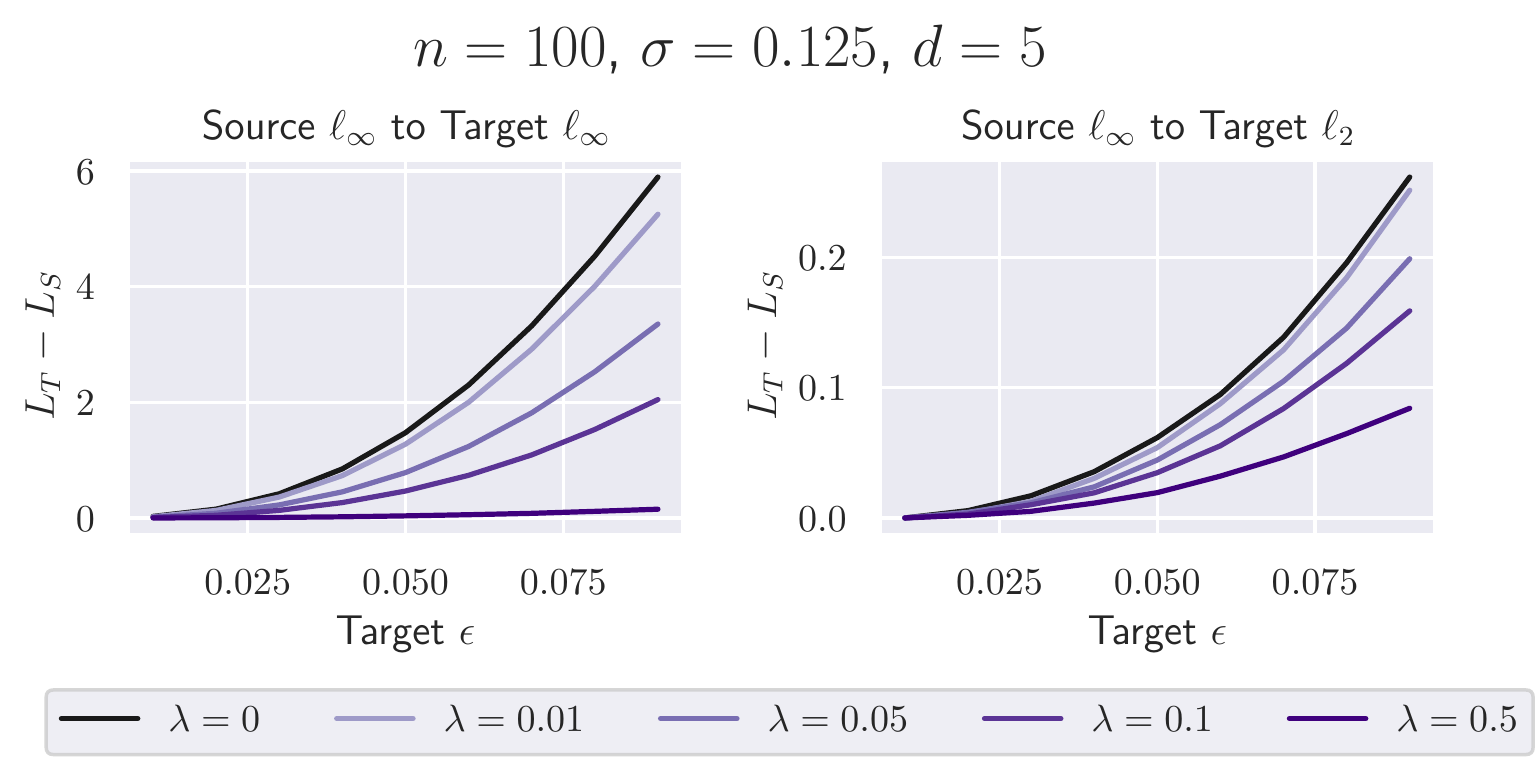}
    \includegraphics[width=0.45\textwidth]{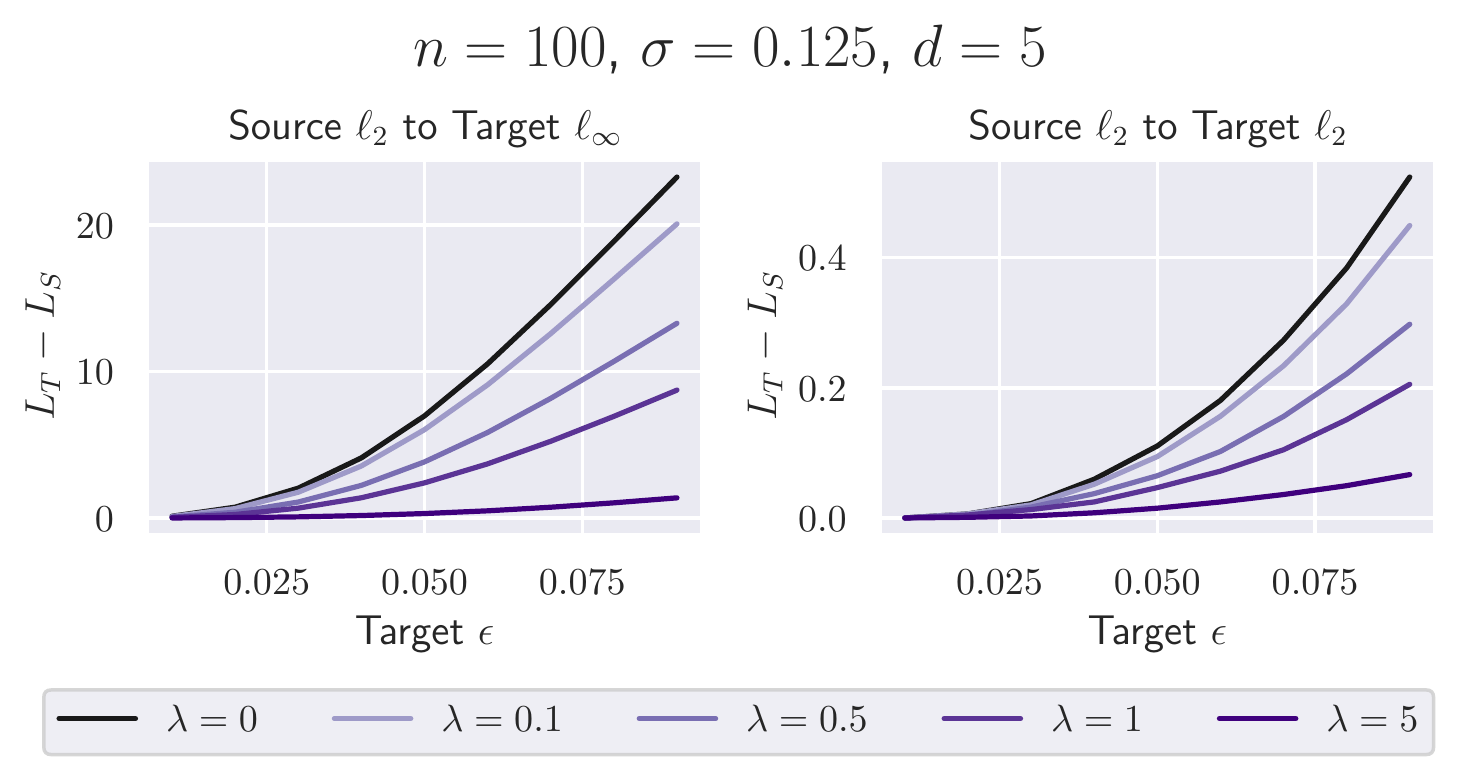}
    \caption{Threat model generalization gap of linear models on Gaussian data trained with varied variation regularization strength $\lambda$ measured on adversarial examples of generated by target $\ell_p, p=\{\infty, 2\}$ perturbations with radius $\epsilon$ at varied input dimension $n$, feature extractor output dimension $d$ and standard deviation $\sigma$.  The generalization gap is measured with respect to cross entropy loss. All models are trained with source $\ell_p, p=\{\infty, 2\}$ radius of 0.01.}
    \label{fig:gauss_gen_abl}
\end{figure}

\begin{figure}[h]
    \centering
    \includegraphics[width=0.47\textwidth]{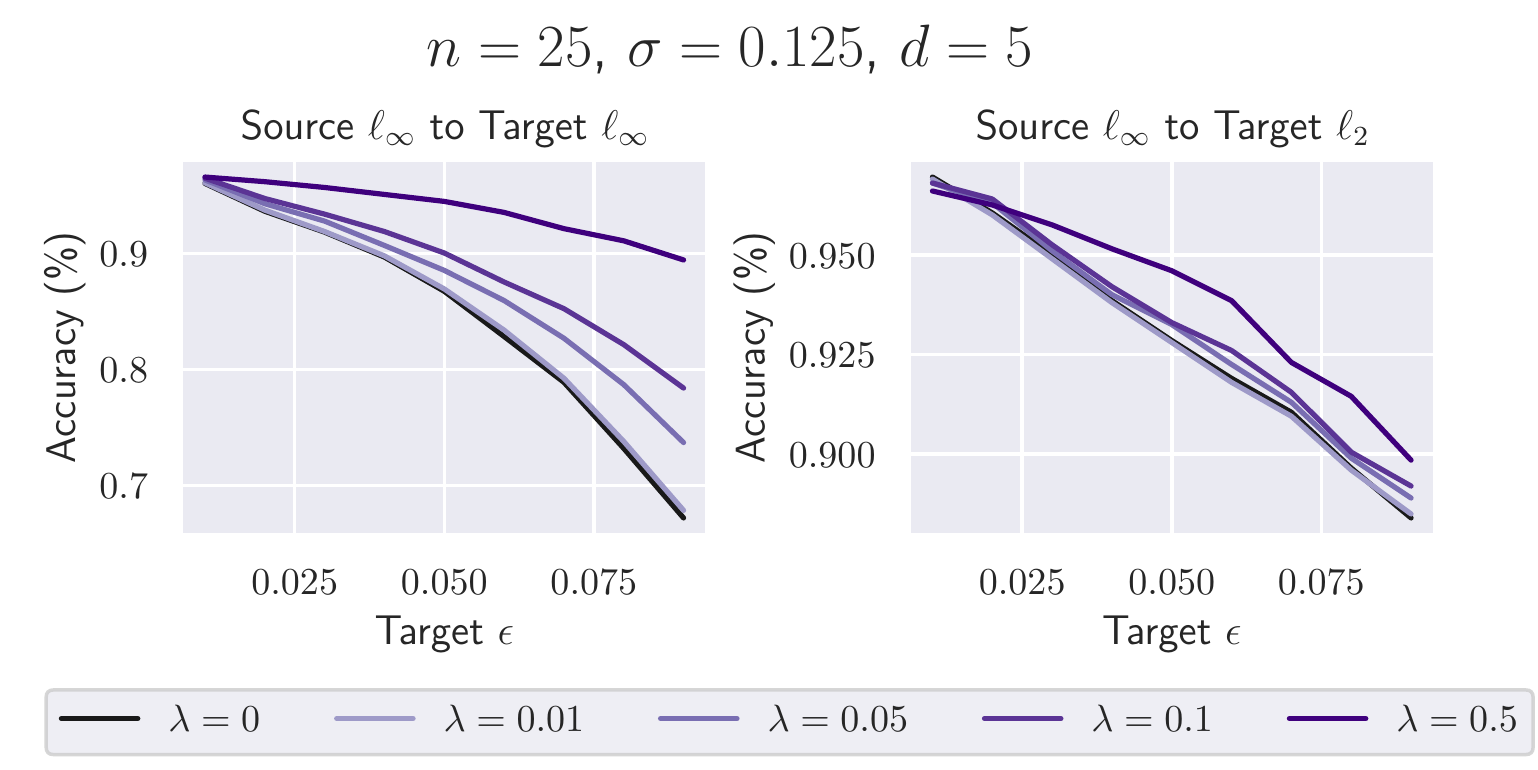}
    \includegraphics[width=0.47\textwidth]{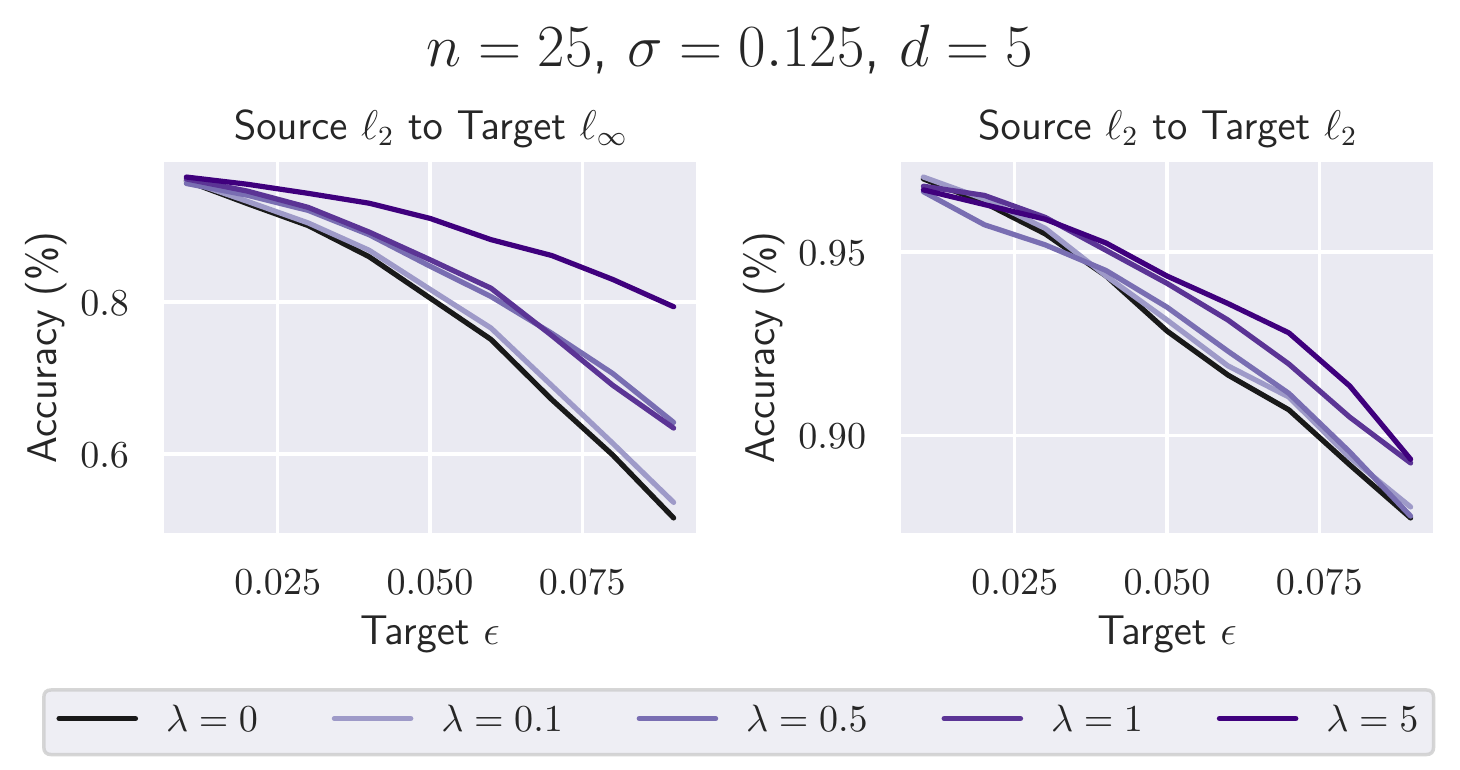}
    \caption{Accuracies of linear models on Gaussian data trained with at varied variation regularization strength $\lambda$ measured on adversarial examples of radius $\epsilon$.  All models are trained with radius 0.01.}
    \label{fig:gauss_acc}
\end{figure}

\subsection{Accuracies over regularization strength}  We plot the accuracies corresponding to the $n=25, \sigma=0.125$ and $d = 5$ setting in Figure \ref{fig:gauss_gen_abl} in Figure \ref{fig:gauss_acc}.  We note that while Figure \ref{fig:gauss_gen_abl} demonstrated that regularization improves decreases the size of the generalization gap, there is trade-off in accuracy which can be seen at small of $\epsilon$.  However, we find that generally variation regularization improves robust accuracy on unforeseen threat models.

\subsection{Evaluations on a separate test set} \label{app:gauss_sep_test} We now plot generalization gap  for the $n=25, \sigma=0.125$ and $d = 5$ setting with cross entropy loss on the target threat model measured on a separate test set of 2000 samples in Figure \ref{fig:gauss_sep_test}.  We find that the trends we observed on the train set are consistent with those observed when evaluating on the train set shown in Figure \ref{fig:gauss_gen_abl}.

\begin{figure}[h]
    \centering
    \includegraphics[width=0.47\textwidth]{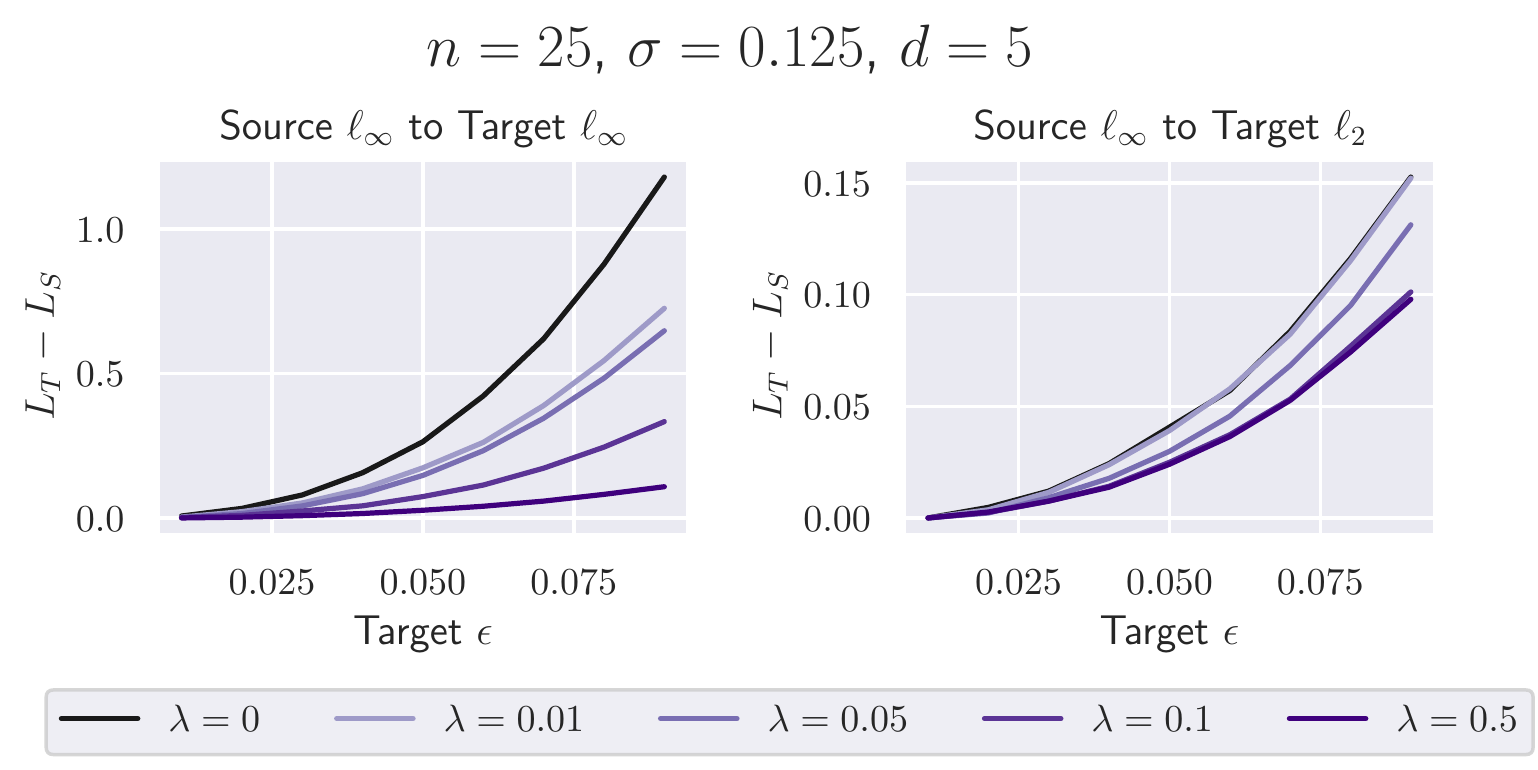}
    \includegraphics[width=0.47\textwidth]{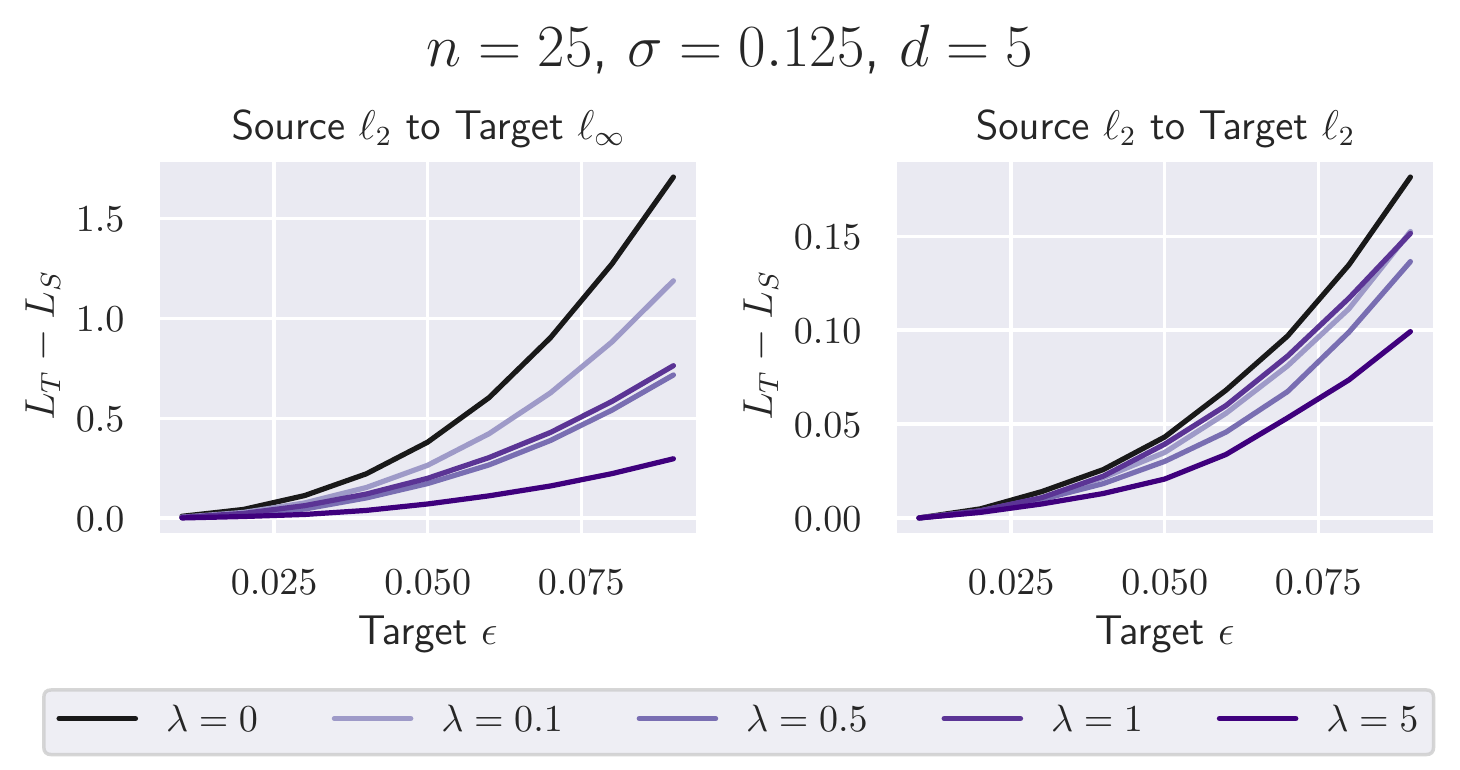}
    \caption{Cross entropy loss of linear models on Gaussian data trained with regularization strength $\lambda$ measured on adversarial examples of radius $\epsilon$. Cross entropy loss is measured with respect to a separate test set.  All models are trained with source $\ell_p$ radius of 0.01.}
    \label{fig:gauss_sep_test}
\end{figure}

\section{Additional Results for Logit Level AT-VR}
In this section, we present additional results for AT-VR when considering the logits to be the output of the feature extractor $h$.
\label{app:add_im_logits}

\subsection{Results on Additional Seeds}\label{app:seeds} 
In Table \ref{tab:seed} we present the mean and standard deviations for robust accuracy of CIFAR-10 ResNet-18 models over 3 trials seeded at 0, 1, and 2.  We find that the improvement observed through variation regularization on unforeseen target threat models is significant for both $\ell_2$ and $\ell_{\infty}$ attacks for ResNet-18 CIFAR-10 models; for bolded accuracies on target threat modes (with the exception of $\ell_{\infty}$ source to $\ell_2$ target which we did not report as an improvement in the main text) we find that the range of the error bars do not overlap with the results for standard adversarial training ($\lambda = 0$).  Additionally, we find that the trade-off with clean accuracy observed when using variation regularization is also significant.
\begin{table}[ht]
    \centering
    \fontsize{6}{9}\selectfont
    \begin{tabular}{@{}ccc|cc|cccc|c@{}}
    \hline
         & & & & & \multicolumn{4}{c|}{Union with Source} &\\
         \cline{6-9}
         Dataset  & Source & $\lambda$ & Clean & Source & $\ell_{\infty}$ & $\ell_2$ & StAdv & Re- & Union \\
         & & &  acc & acc &$\epsilon=\frac{12}{255}$& $\epsilon=1$ & &color & all \\
         \hline
         
         CIFAR-10 & $\ell_{2}$ & 0 & \textbf{88.29} $\pm$ 0.51 & 67.15 $\pm$ 0.45 & 6.77 $\pm$ 0.31 & 35.40 $\pm$ 0.65 & 1.21 $\pm$ 0.46 & 66.99 $\pm$ 0.43 & 0.52 $\pm$ 0.25 \\
         CIFAR-10 & $\ell_{2}$ & 0.25 & 87.39 $\pm$ 0.36& 68.75 $\pm$ 0.08 & 11.56 $\pm$ 0.57   & 39.57 $\pm$ 1.08 & 10.31 $\pm$ 0.41 & 68.60 $\pm$ 0.10 &5.93 $\pm$ 0.38\\
         CIFAR-10 &  $\ell_{2}$ & 0.5 & 86.07 $\pm$ 0.50& \textbf{68.78} $\pm$ 0.17 & 13.13 $\pm$ 1.33 & 41.52 $\pm$ 1.19 & 18.71 $\pm$ 2.79  & 68.59 $\pm$ 0.12 & 8.70 $\pm$ 0.50\\
         CIFAR-10 & $\ell_{2}$ &0.75 & 84.54 $\pm$ 0.61 & 67.97 $\pm$ 0.15 & \textbf{14.61} $\pm$ 0.59 & \textbf{42.13} $\pm$ 0.59 & 23.24 $\pm$ 0.94 & \textbf{67.83} $\pm$ 0.12 & 10.69 $\pm$ 0.15 \\
         CIFAR-10 &  $\ell_{2}$ & 1& 84.71 $\pm$ 1.32& 67.39 $\pm$ 0.29 & 13.62 $\pm$ 0.81 & 41.14 $\pm$ 1.56 & \textbf{33.70} $\pm$ 5.27 & 67.29 $\pm$ 0.31 & \textbf{11.58} $\pm$ 0.37\\
         
         \hline
         CIFAR-10 & $\ell_{\infty}$ & 0 & \textbf{83.01} $\pm$ 0.29& 47.44 $\pm$ 0.08 & 27.79 $\pm$ 0.44  & 24.67 $\pm$ 0.70 &  4.17 $\pm$ 0.26 & 47.44 $\pm$ 0.08 & 2.16 $\pm$ 0.28 \\
         CIFAR-10 &  $\ell_{\infty}$ & 0.05  & 82.70 $\pm$ 0.56  & 48.38 $\pm$ 0.39  & 29.55 $\pm$ 0.43 & 25.25 $\pm$ 0.97 & 4.87 $\pm$ 0.66 & 48.38$\pm$ 0.39 & 2.61 $\pm$ 0.43 \\
         CIFAR-10 &  $\ell_{\infty}$ & 0.1  & 81.79 $\pm$ 0.29 & 48.65 $\pm$ 0.25 & 29.89 $\pm$ 0.53 & 24.99 $\pm$ 0.73 & 6.33 $\pm$ 0.83 & 48.65 $\pm$ 0.25 & 3.47 $\pm$ 0.50 \\
         CIFAR-10 & $\ell_{\infty}$ & 0.3 & 78.87 $\pm$ 0.72 & \textbf{49.16} $\pm$ 0.12 & 31.89 $\pm$ 0.13 & \textbf{24.95} $\pm$ 0.26 & 12.96 $\pm$ 0.81 & \textbf{49.16} $\pm$ 0.12 & 8.53 $\pm$ 0.70 \\
         CIFAR-10 &  $\ell_{\infty}$ & 0.5& 74.24 $\pm$ 2.04 & 48.62 $\pm$ 0.20 & \textbf{33.07} $\pm$ 1.09 & 24.59 $\pm$ 1.32 & \textbf{19.91} $\pm$ 1.12 & 48.62 $\pm$ 0.20 & \textbf{13.35} $\pm$ 0.83 \\
         
    \end{tabular}
    \caption{Mean and standard deviation across 3 trials for robust accuracy of various models trained at different strengths of variation regularization applied on logits on various threat models.  $\lambda = 0$ corresponds to standard adversarial training.  Models are trained with either source threat model $\ell_{\infty}$ with radius $\frac{8}{255}$ or $\ell_2$ with radius $0.5$.   The ``source acc" column reports the accuracy on the source attack.  For each individual threat model, we evaluate accuracy on a union with the source threat model.  The union all column reports the accuracy obtained on the union across all listed threat models.}
    \label{tab:seed}
\end{table}

\subsection{Expansion function on random features}
In Section \ref{sec:cif_exp_fun}, we plotted the expansion function across models trained with adversarial training at various levels of variation regularization.  We now visualize the expansion function for random feature extractors to investigate to what extent the learning algorithm influences expansion function.  We plot the source and target variations (corresponding to the same setup as in Section \ref{sec:cif_exp_fun}) of 300 randomly initialized feature extractors of ResNet-18 on CIFAR-10 in Figure \ref{fig:cif_random_exp_logits}.  For random initialization, we use Xavier normal initialization \citep{GlorotB10} for weights and standard normal initialization for biases.

\begin{figure}[ht]
    \centering
    \includegraphics[width=0.5\textwidth]{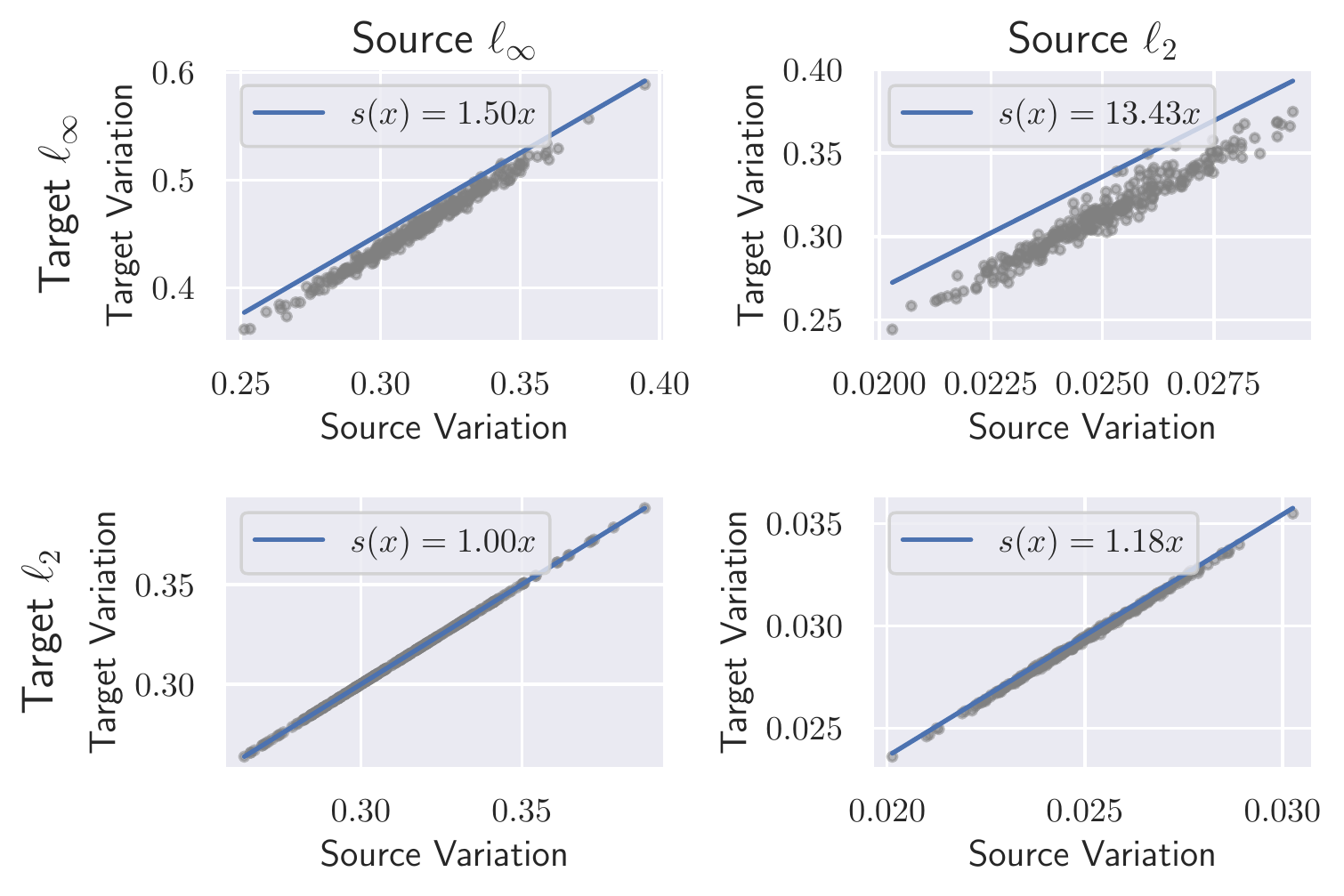}
    \caption{Plots of minimum linear expansion function $s$ shown in blue computed on 300 randomly initialized ResNet-18 models.  Each grey point represents the variation measured on the source and target attack.  Variation is computed on the logits.  The two columns represent the source adversary ($\ell_{\infty}$ and $\ell_2$ respectively).  The two rows represent the target adversary ($\ell_{\infty}$ and $\ell_2$ respectively).}
    \label{fig:cif_random_exp_logits}
\end{figure}
We find that with randomly initialized models, we can also find a linear expansion function with small slope.  In comparison to expansion functions from adversarially trained models (Figure \ref{fig:cif_random_exp_logits}), we find that using randomly initialized models leads to minimum linear expansion function $s$ with smaller slope for $\ell_{\infty}$ target threat model.  This suggests that learning algorithm can have an impact on expansion function.

\subsection{Expansion function between $\ell_p$ and StAdv threat model} While we have shown that an expansion function with small slope exists between $\ell_2$ and $\ell_{\infty}$ threat models, it is unclear whether this also holds for non-$\ell_p$ threat models.  However, we do observe from Table \ref{tab:logits_acc} that AT-VR with $\ell_2$ and $\ell_{\infty}$ sources leads to significant improvements in robust accuracy on StAdv, which is a non-$\ell_p$ threat model, suggesting that an expansion function between these threat models exists.  Using the same models for generating Figure \ref{fig:cifar_exp}, we plot the expansion function from $\ell_{\infty}$ and $\ell_2$ sources to StAdv ($\epsilon=0.05$) in Figure \ref{fig:lp_to_stadv}.  We observe that for both source threat models a linear expansion function exists to StAdv.

\begin{figure}[ht]
    \centering
    \includegraphics[width=0.5\textwidth]{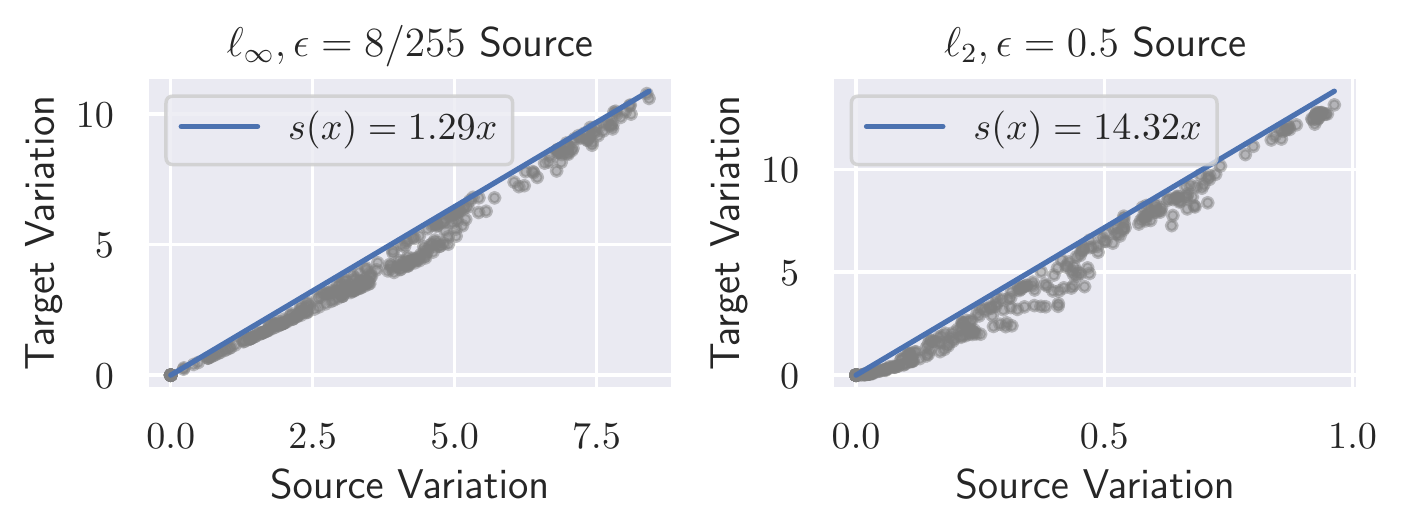}
    \caption{Expansion function from $\ell_{\infty}, \epsilon = 8/255$ and $\ell_2, \epsilon=0.5$ source threat models to StAdv ($\epsilon=0.05$) target threat model computed over 315 adversarially trained ResNet-18 models.  The linear expansion function with minimum slope is plotted in blue.}
    \label{fig:lp_to_stadv}
\end{figure}

We now visualize the expansion function from StAdv ($\epsilon=0.03$) source to $\ell_{\infty}$ ($\epsilon=8/255$), $\ell_2$ ($\epsilon=0.5$), and StAdv ($\epsilon=0.05$) target threat models.  We present plots in Figure \ref{fig:stadv_exp}.  Unlike our plots of expansion function for $\ell_2$ and $\ell_\infty$ source threat models, we find that for StAdv a linear expansion function is not a tight upper bound on the true trend in source vs target variation.  A better model for expansion function would be piecewise linear function with 2 slopes, one for variation values near 0 and one for larger variation values since the slopes at points where source variation is closer to 0 is much larger than the slopes computed at points further from 0.

\begin{figure}[ht]
    \centering
    \includegraphics[width=0.75\textwidth]{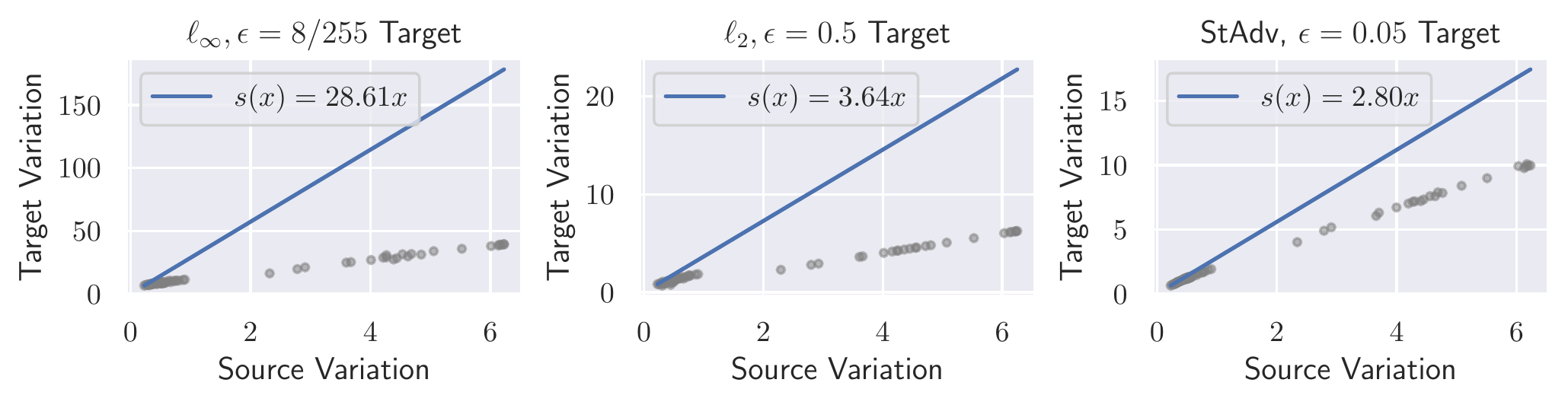}
    \caption{Expansion function from StAdv ($\epsilon=0.03$) source threat model to $\ell_{\infty}$ ($\epsilon=8/255$), $\ell_2$ ($\epsilon=0.5$), and StAdv ($\epsilon=0.05$) target threat models.  Expansion function is computed using 60 ResNet-18 models adversarially trained on CIFAR-10 with adversarial examples generated via StAdv ($\epsilon=0.03$).  The linear expansion function with minimum slope is plotted in blue.}
    \label{fig:stadv_exp}
\end{figure}

\subsection{Additional Results with $\ell_p$ Target Threat Models}\label{app:add_lp} In Figure \ref{fig:cif_gen} of the main text, we plotted the unforeseen generalization gap and robust accuracies for ResNet-18 models trained on CIFAR-10 with AT-VR at various perturbation size $\epsilon$ with $\ell_{\infty}$ source.  We plot the robust accuracy across $\ell_p$ threat models for the models trained with standard adversarial training ($\lambda = 0$) and with highest variation regularization strength used ($\lambda = 0.5$) in Figure \ref{fig:cif_acc_logits}.  We find that at large values of $\ell_p$ perturbation size, the model using variation regularization achieves higher robust accuracy than the model trained using standard adversarial training.  This improvement is most clear for $\ell_{\infty}$ targets.

\begin{figure}
    \centering
    \includegraphics[width=\textwidth]{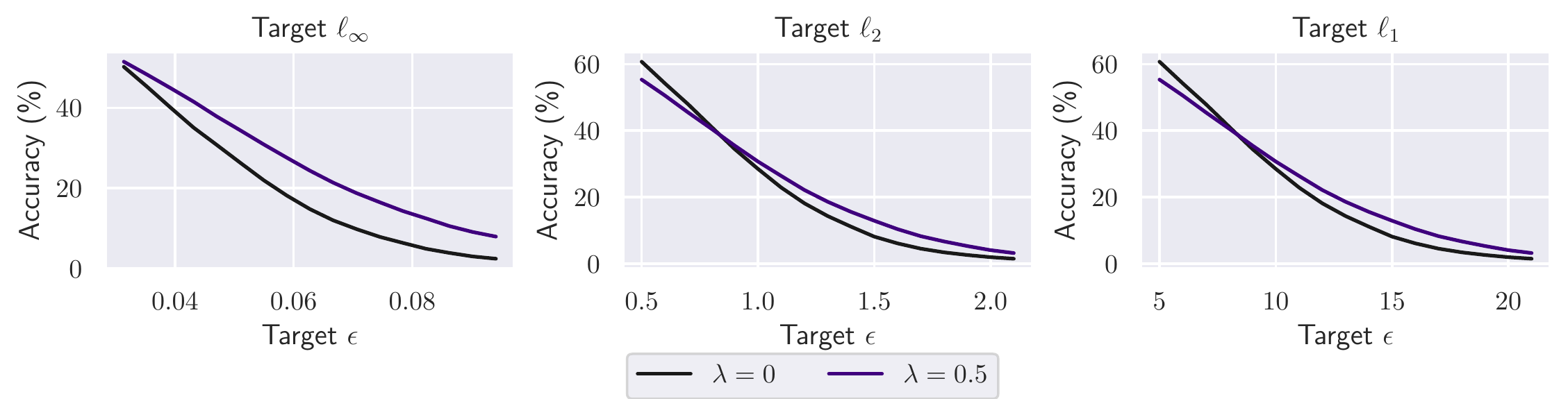}
    \caption{Robust accuracy on the CIFAR-10 test set of ResNet-18 models trained using AT-VR at varied regularization strength $\lambda$ measured on adversarial examples of generated by target $\ell_p, p=\{\infty, 2, 1\}$ perturbations with radius $\epsilon$.  The generalization gap is measured with respect to cross entropy loss. Variation regularization is applied on the logits.  All models are trained with source $\ell_\infty$ perturbations of radius $\frac{8}{255}$.}
    \label{fig:cif_acc_logits}
\end{figure}

We provide corresponding plots of unforeseen generalization gap and robust accuracy on CIFAR-10 on $\ell_p$ target threat models for an $\ell_2$ source in Figure \ref{fig:cif_gen_l2_logits}.  Similar to trends with $\ell_{\infty}$ source threat model, we find that increasing the strength of variation regularization decreases the unforeseen generalization gap.  We find that with an $\ell_2$ source threat model, the robust accuracy for the model trained with variation regularization also has consistently higher accuracy across target $\ell_p$ threat models compared to the model trained with standard adversarial training ($\lambda = 0$).

\begin{figure}[ht]
    \centering
    \includegraphics[width=\textwidth]{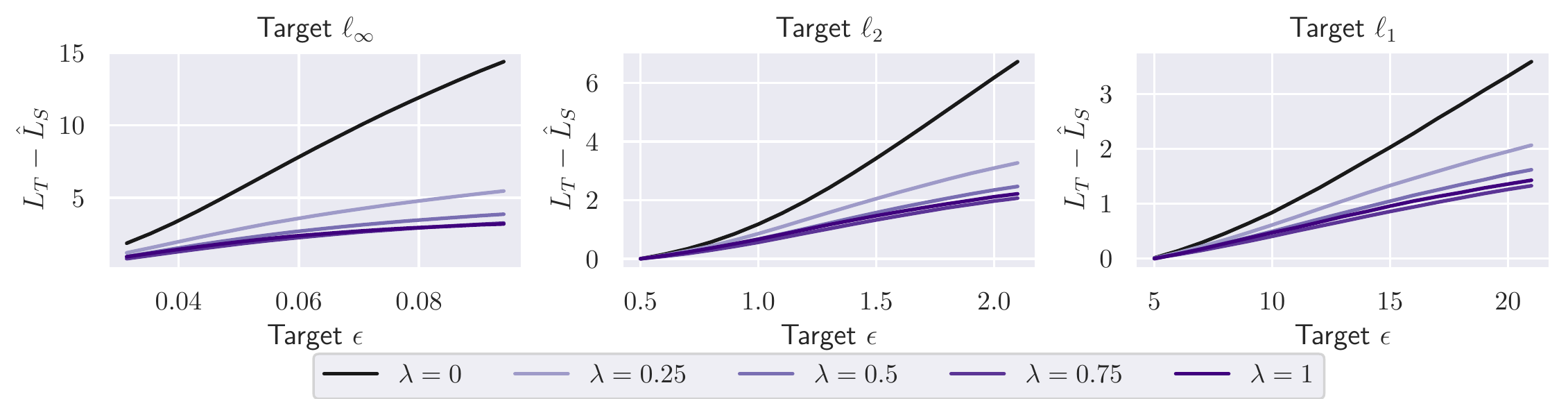}
    \includegraphics[width=1\textwidth]{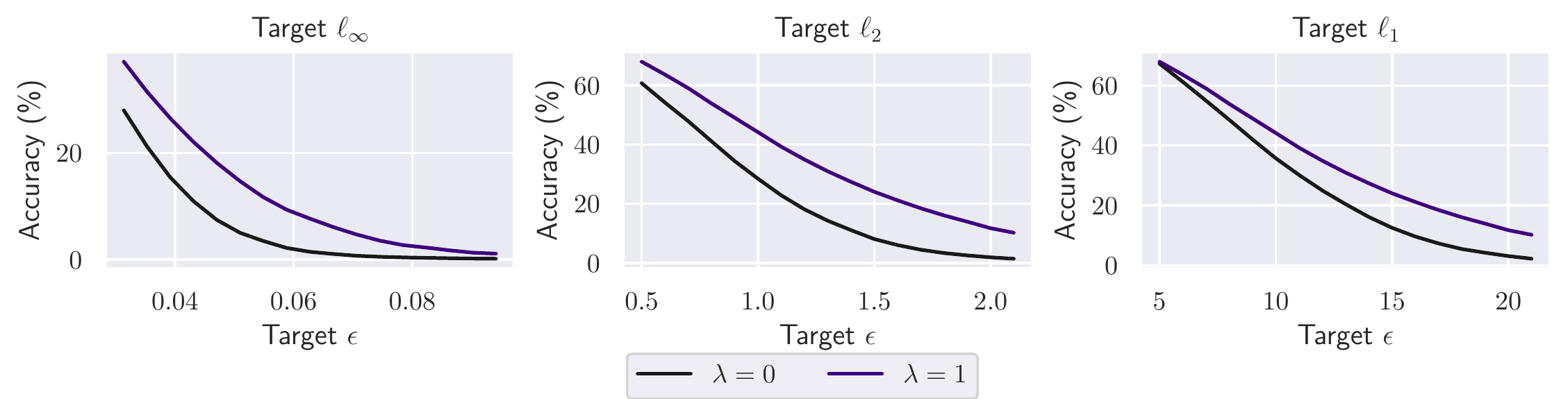}
    \caption{
    \textbf{Top row:} Generalization gap of on the CIFAR-10 test set for ResNet-18 models trained using AT-VR at varied regularization strength $\lambda$ measured on adversarial examples of generated by target $\ell_p, p=\{\infty, 2, 1\}$ perturbations with radius $\epsilon$.  Variation regularization is applied on the logits.  The generalization gap is measured with respect to cross entropy loss. All models are trained with source $\ell_2$ perturbations of radius $0.5$.
    \textbf{Bottom row:} Corresponding robust accuracy of $\lambda=0$ and $\lambda=1$ models displayed in top row.}
    \label{fig:cif_gen_l2_logits}
\end{figure}

\subsection{Additional strengths of variation regularization} \label{app:additional_lam} In Table \ref{tab:logits_acc_full} we present additional results for models trained at AT-VR at different strengths $\lambda$ of variation regularization.  Overall, we find that models using variation regularization improve robustness on unforeseen attacks.  Generally, we find that union accuracies are also larger with higher strengths of variation regularization.

\begin{table}[ht]
    \centering
    \fontsize{8}{9}\selectfont
    \begin{tabular}{@{}cccc|cc|cccc|c@{}}
    \hline
         & & & & & & \multicolumn{4}{c|}{Union with Source} &\\
         \cline{7-10}
         Dataset & Architecture & Source & $\lambda$ & Clean & Source & $\ell_{\infty}$ & $\ell_2$ & StAdv & Re- & Union \\
         & & & & acc & acc &$\epsilon=\frac{12}{255}$& $\epsilon=1$ & &color & all \\
         
         \hline
         
         CIFAR-10 & ResNet-18 & $\ell_{2}$ & 0 & \textbf{88.49}& 66.65 & 6.44 & 34.72 & 0.76 & 66.52 & 0.33\\
         CIFAR-10 & ResNet-18 & $\ell_{2}$ & 0.25 & 87.49 & 68.66 & 11.59 & 39.06 & 10.78 &  68.52 & 6.11 \\
         CIFAR-10 & ResNet-18 & $\ell_{2}$ & 0.5 & 85.75 & \textbf{68.93} & 13.42 & \textbf{41.90} & 20.13 &  \textbf{68.66} & 9.25 \\
         CIFAR-10 & ResNet-18 & $\ell_{2}$ & 0.75 & 84.78 & 67.86 & \textbf{14.27} & 41.68 & 23.45 & 67.75 &  10.58\\
         CIFAR-10 & ResNet-18 & $\ell_{2}$ & 1& 85.21 & 67.38 & 13.43 & 40.74 & \textbf{34.40} & 67.30 & \textbf{11.77} \\
         
         \hline
         CIFAR-10 & ResNet-18 & $\ell_{\infty}$ & 0 & 82.83 & 47.47 & 28.09 & 24.94 & 4.38 &  47.47 &  2.48 \\
         CIFAR-10 & ResNet-18 & $\ell_{\infty}$ & 0.05 & \textbf{83.34} & 48.04 & 29.28 & 24.34 & 4.32 & 48.04 & 2.29 \\
         CIFAR-10 & ResNet-18 & $\ell_{\infty}$ & 0.1  & 81.94 & 48.64 & 29.50 & 24.15 & 6.01 & 48.64 & 3.13 \\
         CIFAR-10 & ResNet-18 & $\ell_{\infty}$ & 0.3 & 79.36& \textbf{49.28} & 31.94 & \textbf{25.08} & 12.75 & \textbf{49.28} & 8.28\\
         CIFAR-10 & ResNet-18 & $\ell_{\infty}$ & 0.5& 72.91 & 48.84 & \textbf{33.69} & 24.38 & \textbf{18.62} & 48.84 & \textbf{12.59}\\
         \hline
         CIFAR-10 & WRN-28-10 & $\ell_{\infty}$ & 0 &  \textbf{85.93}& 49.86 & 28.73 & 20.89 & 2.28 & 49.86 & 1.10 \\
         CIFAR-10 & WRN-28-10 & $\ell_{\infty}$ & 0.1 & 84.82 & 50.42 & 30.10 & 20.85 & 5.54 & 50.42 & 3.42 \\
         CIFAR-10 & WRN-28-10 & $\ell_{\infty}$ & 0.3 & 83.47 & \textbf{51.19} & 31.71 & 20.01 & 15.00 & \textbf{51.19} & 9.81 \\
         CIFAR-10 & WRN-28-10 & $\ell_{\infty}$ & 0.5& 80.43 & 51.16& 33.18 & 20.76 & 21.05 & 51.16 & 12.06 \\
         CIFAR-10 & WRN-28-10 & $\ell_{\infty}$ & 0.7 & 72.73 & 49.94 & \textbf{35.11} & \textbf{22.30} & \textbf{25.33} & 49.94 & \textbf{14.72} \\
         \hline
         CIFAR-10 & VGG-16 & $\ell_{\infty}$ & 0 & \textbf{79.67} & 44.36 & 26.14 & 30.82 & 7.31 & 44.36 & 4.35\\
         CIFAR-10 & VGG-16 & $\ell_{\infty}$ & 0.01 & 78.12 & 45.13& 27.22& \textbf{33.34} & 8.09 & 45.13  & 5.24 \\
         CIFAR-10 & VGG-16 & $\ell_{\infty}$ & 0.05 & 79.24 & 44.73 & 27.07 & 30.99 & 7.94 &  44.73 & 4.65 \\
         CIFAR-10 & VGG-16 & $\ell_{\infty}$ & 0.1 & 77.80 & \textbf{45.42} & \textbf{28.41} & 32.08 & 10.57 & \textbf{45.42} & 6.83\\
         CIFAR-10 & VGG-16 & $\ell_{\infty}$ & 0.15 & 76.19 & 44.65 & 27.44 & 29.65 & \textbf{11.99} & 44.65 & \textbf{7.48} \\

                 \hline
         ImageNette & ResNet-18 & $\ell_2$ &  0 &\textbf{88.94} & \textbf{84.99} & 0.00 & 79.08 & 1.27 & 72.15  & 0.00 \\
         ImageNette & ResNet-18 & $\ell_{2}$ & 0.25 & 86.83 & 84.28 & 1.58 & 80.66 & 8.38 &  74.17 & 0.94 \\
         ImageNette & ResNet-18 & $\ell_{2}$ & 0.5& 86.80 & 84.00 & 4.25 & 80.23 & 12.82 &  74.17 & 2.88 \\
         ImageNette & ResNet-18 & $\ell_{2}$ & 0.75 &85.83 & 83.92 & 5.81 & \textbf{80.94} & 15.57 &  74.73 & 4.28 \\
         ImageNette & ResNet-18 & $\ell_2$ & 1 & 85.22 & 83.08 & \textbf{9.53}& 80.43 & \textbf{18.04} & \textbf{75.26} & \textbf{6.80}\\
         
          \hline

         ImageNette & ResNet-18 & $\ell_{\infty}$ & 0 & \textbf{80.56} & 49.63 & 32.38 & 49.63 & 34.27 & 49.63 & 25.68\\
         ImageNette & ResNet-18 & $\ell_{\infty}$ & 0.05 & 79.64 & 50.50 & 33.27 & 50.50 & 39.01 & 50.50 & 28.48 \\
         ImageNette & ResNet-18 & $\ell_{\infty}$ & 0.1 & 78.01 & \textbf{50.80} & 35.57 & \textbf{50.80} & 42.37 & \textbf{50.80} &31.82  \\
         ImageNette & ResNet-18 & $\ell_{\infty}$ & 0.15 & 75.62 & 49.94 &36.15 & 49.94 & 43.26 & 49.94 & 33.22 \\
         ImageNette & ResNet-18 & $\ell_{\infty}$ & 0.2 & 73.50 & 49.17 & \textbf{36.28} & 49.15 & \textbf{44.00} & 49.17 &  \textbf{34.22} \\

         \hline
         CIFAR-100 & ResNet-18 & $\ell_{2}$ & 0 & \textbf{60.92} & 36.01 & 3.98 & 16.90 & 1.80 & 34.87 & 0.40  \\
         CIFAR-100 & ResNet-18 & $\ell_{2}$  & 0.25 & 56.20 & 38.26 & 7.85 &  23.43 & 2.88 & 36.63 & 1.34 \\
         CIFAR-100 & ResNet-18 & $\ell_{2}$  & 0.5 & 54.23 & \textbf{38.51} & 9.70 & 24.33 & 4.03 & 36.87 & 2.29 \\
         CIFAR-100 & ResNet-18 & $\ell_{2}$  & 0.75 & 51.53 & 38.26 & \textbf{11.47} & \textbf{25.65} & \textbf{5.12} & \textbf{36.96} & \textbf{3.11}\\
         CIFAR-100 & ResNet-18 & $\ell_{2}$ & 1 & 50.85 & 37.00 & 10.53 & 24.89& 5.09&35.82 &3.05\\
         
         \hline
         CIFAR-100 & ResNet-18 & $\ell_{\infty}$ & 0 & \textbf{54.94} & 22.74 & 12.61 & 14.40 & 3.99 & 22.71 &  2.42 \\
         CIFAR-100 & ResNet-18 & $\ell_{\infty}$  & 0.05 & 53.59 & 24.24 & 14.05 & 14.93 & 4.26  & 24.19 & 2.65 \\
         CIFAR-100 & ResNet-18 & $\ell_{\infty}$ & 0.1 & 52.32 & 24.78 & 15.09 & 15.57 & 4.25& 24.74 & 2.91 \\
         CIFAR-100 & ResNet-18 & $\ell_{\infty}$  & 0.2 & 48.97 & \textbf{25.04} &\textbf{16.48}  & \textbf{15.82}  & 4.96 & \textbf{24.95} & \textbf{3.48} \\
         CIFAR-100 & ResNet-18 & $\ell_{\infty}$  & 0.3 & 46.39 & 24.82 & 16.45 & 14.54 & \textbf{5.04} & 24.75 & 3.07 \\
         
         \hline
    \end{tabular}
    \caption{Robust accuracy of various models trained at different strengths of variation regularization applied on logits on various threat models.  Models are trained with either source threat model $\ell_{\infty}$ with radius $\frac{8}{255}$ or $\ell_2$ with radius $0.5$.   The ``source acc" column reports the accuracy on the source attack.  For each individual threat model, we evaluate accuracy on a union with the source threat model.  The union all column reports the accuracy obtained on the union across all listed threat models.}
    \label{tab:logits_acc_full}
\end{table}

\subsection{Full AutoAttack results on CIFAR-10} \label{app:full_aa}
In Table \ref{tab:full_aa}, we report the full AutoAttack evaluation of the ResNet-18 models trained on CIFAR-10 with AT-VR with highest regularization strength displayed in in Table \ref{tab:unseen_acc}.  We find that robust accuracy is consistent across attack types, suggesting that variation regularization is not obfuscating gradients.

\begin{table}[h]
    \centering
    \fontsize{8}{9}\selectfont
    \begin{tabular}{@{}l|lccc@{}}
    \hline
    & & Source &$\ell_{\infty}$ & $\ell_2$ \\
      & & & ($\epsilon=\frac{12}{255}$)& ($\epsilon=1$)\\ 
    \hline
         \multirow{4}{1.45cm}{$\ell_{\infty}$ source \\ $\lambda=0.5$}&APGD-CE & 51.56 &37.84 & 30.56 \\
         &APGD-T & 48.84 &33.72 & 24.93 \\
         &FAB-T & 48.84 &33.69 & 24.38 \\
         & Square & 48.84 &33.69 & 24.38\\
         \hline
         
         \multirow{4}{1.45cm}{$\ell_{2}$ source\\ $\lambda=1$} & APGD-CE & 68.00 &18.23 & 44.16 \\
          & APGD-T & 67.39 &14.73 & 41.67 \\
         &FAB-T & 67.38 &13.43 & 40.74 \\
         & Square & 67.38 &13.43 & 40.74 \\
\end{tabular}
    \caption{Full AutoAttack evaluations for the ResNet-18 models trained with AT-VR with variation regularization strength on $\ell_{\infty}$ ($\epsilon=\frac{8}{255}$) and $\ell_2$ ($\epsilon = 0.5)$ source adversaries.}
    \label{tab:full_aa}
\end{table}

\subsection{Evaluations on other adversaries} \label{app:additional_adv}
In Table \ref{tab:additional_acc}, we present additional evaluations for ResNet-18 CIFAR-10 models on additional adversaries including Wasserstein adversarial attacks, JPEG attacks, elastic attacks, and perceptual attacks.  For Wasserstein adversarial attacks, we use PGD with dual projection \citep{wu2020stronger}.  We use $\ell_{\infty}$ JPEG, $\ell_1$ JPEG and elastic attacks from \citep{kang2019robustness} and AlexNet LPIPS-based attacks (PPGD and LPA) from \citep{laidlaw2020perceptual} with perturbation size $\epsilon$ specified in Table \ref{tab:additional_acc}.  Overall, we find that using variation regularization with $\ell_p$ sources also improves performance on these additional adversaries.  This suggests that an expansion function exists between variation on $\ell_p$ sources to more complicated threat models, including more perceptually-aligned threat models such as the bounded AlexNet LPIPS distance used by PPGD and LPA attacks.

\begin{table}[ht]
    \centering
    \fontsize{8}{9}\selectfont
    \begin{tabular}{@{}cc|cc|ccccccc@{}}
    \hline
          & & & & \multicolumn{7}{c}{Union with Source} \\
         \cline{5-11}
          Source & $\lambda$ & Clean & Source & \multicolumn{2}{c}{Wasserstein }  & $\ell_{\infty}$ JPEG & $\ell_{1}$ JPEG & Elastic & PPGD &  LPA \\
          & & acc & acc &$\epsilon=0.007$& $\epsilon=0.01$ & $\epsilon = 0.125$ & $\epsilon = 64$ & $\epsilon=0.25$ &  $\epsilon=0.5$ &  $\epsilon=0.5$\\
         \hline
          $\ell_{2}$ & 0 & \textbf{88.49}& 66.65 & 45.80 & 31.75 & 48.73 & 6.42 & 13.49 & 3.65 &  0.49 \\
          $\ell_{2}$ & 1& 85.21 & \textbf{67.38} & \textbf{48.32} & \textbf{47.58} & \textbf{56.99} & \textbf{21.75} & \textbf{25.67} &  \textbf{24.21}  & \textbf{4.12} \\
          
          \hline
          $\ell_{\infty}$ & 0 & \textbf{82.83} & 47.47 & 26.03 & 16.62 & 34.45 & 2.32 & 25.23  & 2.44  & 0.28 \\
         $\ell_{\infty}$ & 0.5& 72.91 & \textbf{48.84} & \textbf{29.50} & \textbf{20.89} & \textbf{36.63} & \textbf{7.25} & \textbf{29.62}  & \textbf{4.96} & \textbf{2.18} \\
         \hline
    \end{tabular}
    \caption{Robust accuracy of ResNet-18 models on CIFAR-10 evaluated on additional adversaries including Wasserstein adversaries \citep{wu2020stronger}, JPEG compression adversaries and elastic adversary \citep{kang2019robustness}, and AlexNet LPIPS perceptual adversaries \citep{laidlaw2020perceptual}.  Perturbation size $\epsilon$ is specified for each threat model.  Accuracies on each target adversary reported are given with a union computed on the source.}
    \label{tab:additional_acc}
\end{table}

\subsection{Comparison to training with all attacks}
In this section, we compare to the MAX adversarial training approach from \citet{TB19} which trains directly on all target threat models of interest.  Since MAX training uses information about the target threat model for training while our approach does not, the union accuracies achieved via MAX training should be viewed as an upper bound on performance.  For models using MAX training, we train for a total of 100 epochs and evaluate performance on the model saved at the final epoch.  We compare robust accuracies for MAX training to robust accuracies of models trained with variation regularization from Table \ref{tab:logits_acc}.  For fair comparison, we report robust accuracies of models trained with variation regularization measured through evaluating on the target threat model (instead of the union of target with source threat model as in Table \ref{tab:logits_acc}.  We provide results in Table \ref{tab:max_train}.  We note that in general, training with the union of all attacks achieves more balanced accuracies across threat models.  The only exception is with WRN-28-10 for which MAX training achieves lower union accuracy and StAdv accuracy; however, this may be due to using only 100 epochs of training.

\begin{table}[h]
    \centering
\fontsize{8}{9}\selectfont
    \begin{tabular}{@{}cccc|c|cccc|c@{}}
    \hline
         Dataset & Architecture & Source & $\lambda$ & Clean  & $\ell_{\infty}$ & $\ell_2$ & StAdv & Re- & Union \\
         & & & & acc  &$\epsilon=\frac{12}{255}$& $\epsilon=1$ & &color & all \\
         \hline
         
         CIFAR-10 & ResNet-18 & MAX & 0 & 77.81 & 41.76 & 36.74 & 38.90 & 63.99 & 27.23\\
         CIFAR-10 & ResNet-18 & $\ell_{2}$ & 1& 85.21 & 13.39 & 40.79 & 32.15 & 60.07 & 11.27 \\
         
         CIFAR-10 & ResNet-18 & $\ell_{\infty}$ & 0.5& 72.91& 33.66 & 24.34 & 16.09 & 47.38 & 10.58\\
        
         \hline
         
         \hline
         CIFAR-10 & WRN-28-10 & MAX & 0 & 80.98 & 29.78 & 38.25 & 9.83 & 57.60 & 4.98 \\
         CIFAR-10 & WRN-28-10 & $\ell_{\infty}$ & 0.7 & 72.73 & 35.08 & 22.35 & 23.15 & 49.01 & 13.26\\
         \hline
         CIFAR-10 & VGG-16 & MAX & 0 & 67.26 & 29.00 & 41.48 & 42.98 & 59.11 & 26.12 \\
         CIFAR-10 & VGG-16 & $\ell_{\infty}$ & 0.1 & 77.80  & 28.38 & 32.08 & 11.96 & 62.24 & 5.38 \\
                 \hline
         ImageNette & ResNet-18 & MAX & 0 & 69.32 & 38.57 & 64.36 & 57.04 & 59.34 & 37.25 \\
         ImageNette & ResNet-18 & $\ell_2$ & 1 & 85.22 & 9.50 & 80.43 & 16.25 & 44.99 & 5.83 \\
         ImageNette & ResNet-18 & $\ell_{\infty}$ & 0.1 & 78.01 & 35.54 & 72.94 & 51.75 & 64.15 & 31.03 \\
         
         \hline
         CIFAR-100 & ResNet-18 & MAX & 0 & 52.06 & 13.29 & 19.35 & 8.01 & 26.42 & 5.04 \\
         CIFAR-100 & ResNet-18 & $\ell_{2}$  & 0.75 & 51.53 & 11.46 & 25.64 & 4.17 & 18.98 & 2.26 \\
         CIFAR-100 & ResNet-18 & $\ell_{\infty}$  & 0.2 & 48.97 & 16.51 & 15.81 & 4.70 & 24.01 & 2.59 \\
         \hline
    \end{tabular}
    \caption{Robust accuracy of various models trained at different strengths of VR applied on logits on various threat models.  Source of MAX represents the accuracies obtained by directly training on all target threat models. The union all column reports the accuracy on the union across all listed threat models.}
    \vspace{-5pt}
    \label{tab:max_train}
\end{table}

\subsection{Combining variation regularization with TRADES}
\label{app:trades-VR}
In the main paper, we provided experiments with variation regularization combined with PGD adversarial training \citep{madry2017towards} on $\ell_{\infty}$ and $\ell_2$ source threat models.  We note that variation regularization is not exclusive to PGD adversarial training and can be applied in conjunction with other adversarial training based techniques including TRADES \citep{zhang2019theoretically}.  We present results for TRADES-VR with TRADES hyperparameter $\beta = 6.0$ in Table \ref{tab:trades_vr}. Interestingly, we find that compared to PGD adversarial training results in Table \ref{tab:logits_acc}, TRADES generally has better performance on larger $\ell_{\infty}$ and $\ell_2$ threat models (on par with AT-VR with PGD).  We find that across architectures and datasets applying variation regularization over TRADES adversarial training generally improves robustness on unforeseen $\ell_{\infty}$ and StAdv threat models, but trades off clean accuracy, source accuracy, and unforeseen $\ell_2$ target accuracy.  Despite this trade-off, we find that TRADES-VR consistently improves on the accuracy across the union of all threat models in comparison to standard TRADES.

\begin{table}[ht]
    \centering
    \fontsize{8}{9}\selectfont
    \begin{tabular}{@{}cccc|cc|cccc|c@{}}
    \hline
         & & & & & & \multicolumn{4}{c|}{Union with Source} &\\
         \cline{7-10}
         Dataset & Architecture & Source & $\lambda$ & Clean & Source & $\ell_{\infty}$ & $\ell_2$ & StAdv & Re- & Union \\
         & & & & acc & acc &$\epsilon=\frac{12}{255}$& $\epsilon=1$ & &color & all \\
         \hline

         CIFAR-10 & ResNet-18 & $\ell_{2}$ & 0 & \textbf{86.79} & \textbf{68.99} & 13.13 & \textbf{43.58} & 1.89 & \textbf{68.85} & 1.00\\
         CIFAR-10 & ResNet-18 & $\ell_{2}$ & 2& 79.81 & 64.84 & \textbf{16.04} & 42.46 & \textbf{42.05} & 64.76 & \textbf{14.67}\\
         
         \hline
         
         CIFAR-10 & ResNet-18 & $\ell_{\infty}$ & 0 & \textbf{82.67} & \textbf{49.15} & 31.00 & \textbf{28.04} & 5.32 &   \textbf{49.15} & 3.64 \\
         CIFAR-10 & ResNet-18 & $\ell_{\infty}$ & 0.5& 79.11 & 48.98 & \textbf{32.08} & 26.54 & \textbf{16.21} &  48.98 & \textbf{11.78} \\
         
         \hline
         
         CIFAR-10 & WRN-28-10 & $\ell_{\infty}$ & 0 & \textbf{84.73} & \textbf{52.09} & 32.74 &\textbf{24.68 }& 4.54 & \textbf{52.09} & 2.98 \\
         CIFAR-10 & WRN-28-10 & $\ell_{\infty}$ & 1 & 75.99 & 50.09 & \textbf{33.96} & 21.76 & \textbf{25.09} & 50.09 & \textbf{15.08} \\
         
         \hline
         ImageNette & ResNet-18 & $\ell_{2}$ & 0 & \textbf{88.66} & \textbf{85.55} & 0.08 & \textbf{80.71} & 1.94 & 72.74 & 0.03 \\
         ImageNette & ResNet-18 & $\ell_{2}$ & 2 & 83.80 & 82.27 & \textbf{14.37} & 80.10 & \textbf{27.08} & \textbf{75.03} & \textbf{11.75} \\
         
         \hline
         ImageNette & ResNet-18 & $\ell_{\infty}$ & 0 & \textbf{78.32} & \textbf{50.62} & 35.54 & \textbf{50.62} & 44.03 & \textbf{50.62} & 33.32\\
         ImageNette & ResNet-18 & $\ell_{\infty}$ & 0.2 & 73.83 & 48.79 & \textbf{36.15} & 48.79 & \textbf{44.97}  &  48.79 &  \textbf{34.62}\\
         
         \hline
         CIFAR-100 & ResNet-18 & $\ell_{2}$ & 0 & \textbf{58.71} & \textbf{37.79} & 6.73 & 21.40 & 2.72 & \textbf{36.74} & 1.21 \\
         CIFAR-100 & ResNet-18 & $\ell_{2}$ & 1 & 53.44 & 37.71 & \textbf{10.09} & \textbf{24.55} & \textbf{4.05}& 36.59 & \textbf{2.40}\\
         
         \hline
         CIFAR-100 & ResNet-18 & $\ell_{\infty}$ & 0 & \textbf{53.80} & 23.02 & 13.77 & \textbf{15.21} & \textbf{4.94} & 22.99 & 3.33\\
         CIFAR-100 & ResNet-18 & $\ell_{\infty}$ & 0.5 & 51.16 & \textbf{24.39} & \textbf{15.93} & 14.33 & 4.87 & \textbf{24.34} & \textbf{3.40} \\

         \hline
         
    \end{tabular}
    \caption{Robust accuracy of various models trained with TRADES-VR at different strengths of variation regularization applied on logits on various threat models.  Models are trained with either source threat model $\ell_{\infty}$ with radius $\frac{8}{255}$ or $\ell_2$ with radius $0.5$.   The ``source acc" column reports the accuracy on the source attack.  For each individual threat model, we evaluate accuracy on a union with the source threat model.  The union all column reports the accuracy obtained on the union across all listed threat models.}
    \label{tab:trades_vr}
\end{table}

\subsection{Combining variation regularization with other sources}
\label{app:additional_sources}

\begin{table}[ht]
    \centering
    \fontsize{8}{9}\selectfont
    \begin{tabular}{@{}cc|cc|cccc|c@{}}
    \hline
          & & & & \multicolumn{4}{c|}{Union with Source} &\\
         \cline{5-8}
         Source & $\lambda$ & Clean & Source & $\ell_{\infty}$ & $\ell_2$ & StAdv & Re- & Union \\
          & &  acc & acc &$\epsilon=\frac{4}{255}$& $\epsilon=0.5$ &  &color & all \\
         \hline
         StAdv & 0 & \textbf{86.94} & 54.04 & 3.57 & 5.63 & 12.62 & 3.27 & 0.96 \\
         StAdv & 0.5 & 83.88 & 60.11 & 3.29 & 5.66 & 24.64 & 8.60 & 2.39 \\
         StAdv & 1 & 81.24& \textbf{62.78} & \textbf{5.83} & \textbf{9.97}& \textbf{31.09} & \textbf{13.10} & \textbf{4.44} \\
         \hline
         Recolor & 0 & \textbf{94.88} & 39.11 & 5.38 & 3.07 & 0.00 & 21.75 & 0.00 \\
         Recolor & 0.5& 94.18 & 71.36 & 25.39 & 19.22 & 0.06 & 64.39 & 0.03 \\
         Recolor & 1& 94.13 & \textbf{72.92} & \textbf{26.81} & \textbf{20.03} &  \textbf{0.20} & \textbf{66.02} & \textbf{0.17} \\
         \hline
    \end{tabular}
    \caption{Robust accuracy of ResNet-18 models trained using AT-VR with StAdv and Recolor source threat models with variation regularization applied on logits on various threat models. During training we use 0.03 and 0.04 as the perturbation bounds for StAdv and Recolor respectively.  During testing we use 0.05 for StAdv and 0.06 for Recolor.  The ``source acc" column reports the accuracy on the source attack.  For each individual threat model, we evaluate accuracy on a union with the source threat model.  The union all column reports the accuracy obtained on the union across all listed threat models.}
    \label{tab:stadv_Recolor}
\end{table}

To demonstrate that variation regularization can be applied with source threat models outside of $\ell_p$ balls, we evaluate the performance of AT-VR with other sources including StAdv and Recolor.

\subsubsection{Computing variation with StAdv and Recolor sources} In StAdv \citep{XiaoZ0HLS18}, adversarial examples are generated by optimizing for a per pixel flow field $f$, where $f_i$ corresponds to the displacement vector of the $i$th pixel of the image.  This flow field is obtained by solving:
\begin{equation}
\label{eq:functional}
  \argmin_f \ell_{\text{adv}}(x, f) + \tau \ell_{\text{flow}}(f)  
\end{equation}
where $\ell_{\text{adv}}$ is the CW objective \citep{carlini2017towards} and $\ell_{flow}$ is a regularization term minimizing spatial transformation distance to ensure that the perturbation is imperceptible. $\tau$ controls the strength of this regularization.

We adapt this objective to compute variation, replacing $\ell_{\text{adv}}$ with the variation objective.  Rather than solving for a single flow field, we solve
$$\mathcal{V}(h, \text{StAdv}_{\tau}) = \max_{f_1, f_2} ||h(f_1(x)) - h(f_2(x))||_2 - \tau( \ell_{\text{flow}}(f_1) + \ell_{\text{flow}}(f_2))$$
Here $f_1(x)$ and $f_2(x)$ denote the perturbed input after applying $f_1$ and $f_2$ respectively and $h$ denotes our feature extractor.  The second term ensures that both $f_1$ and $f_2$ have small spatial transformation distance.  We solve the optimization problem using PGD.

In Recolor attacks \citep{LaidlawF19}, the objective function takes the same form as Equation \ref{eq:functional} where $f$ is now a color perturbation function.  We optimize for adversarial examples for Recolor in the same way as for StAdv, but incorporate an additional clipping step to ensure that perturbations of each color are within the specified bounds.

\subsubsection{Experimental setup details for StAdv and Recolor sources} We train ResNet-18 models using AT-VR with StAdv and Recolor sources.  For these models, we train with source perturbation bound of 0.03 for StAdv and 0.04 for Recolor attacks.  We use 10 iterations for both StAdv and Recolor during training.  We train for a maximum of 100 epochs and evaluate on the model saved at epoch achieving the highest source accuracy.  For evaluation, we use StAdv perturbation bound of 0.05 and Recolor perturbation bound of 0.06 and use 100 iterations to generate adversarial examples for both attacks.  Additionally, we evaluate on $\ell_{\infty}$ and $\ell_2$ attacks with radius $\frac{4}{255}$ and 0.5 respectively using AutoAttack.

\subsubsection{Experimental Results for StAdv and Recolor sources}
We report results in Table \ref{tab:stadv_Recolor}. Similar to trends observed with $\ell_p$ sources, we find that using variation regularization improves unforeseen robustness when using StAdv and Recolor sources.  For example, with StAdv source, using variation regularization with $\lambda = 1$ increases robust accuracy on unforeseen $\ell_2$ attacks from 3.57\% to 5.83\%, and with Recolor source, using variation regularization with $\lambda=1$ increases robust accuracy on unforeseen $\ell_2$ attacks from 3.07\% to 20.03\%.  The largest increase is for attacks of the same perturbation type but larger radius; for example, for StAdv source ($\epsilon=0.03$) on unforeseen StAdv target ($\epsilon=0.05$), the robust accuracy increases from 12.62\% to 31.09\%.  Similarily, for Recolor source ($\epsilon=0.04$) on unforeseen Recolor target ($\epsilon=0.06$), the robust accuracy increases from 21.75\% to 66.02\%.  Interestingly, we also find that using variation regularization with StAdv and Recolor sources leads to a significant increase in source accuracy as well. For StAdv, source accuracy increases from 54.04\% without variation regularization to 62.78\% with variation regularization at $\lambda=1$. For Recolor source, this increase is even larger; source accuracy increases from 39.11\% without variation regularization to 72.92\% with variation regularization.

Further inspecting the source accuracy of the models trained Recolor source, we find that without variation regularization, the model overfits to the 10-iteration attack used during training.  When evaluated with 10-iteration Recolor, the standard adversarial training ($\lambda=0$) model achieves 86.62\% robust accuracy, but when 100 iterations of the attack is used during testing, the model's accuracy drops to 39.11\%.  Interestingly, variation regularization helps prevent adversarial training from overfitting to the 10-iteration attack, causing the resulting source accuracy on the models with variation regularization to be significantly higher.

\subsection{Computational complexity of AT-VR}
One limitation of AT-VR is that it can be 3x as expensive compared to adversarial training.  This is because the computation for variation also requires gradient based optimization.  We note that this computational expense occurs when we use the same number of iterations of PGD for variation computation as standard adversarial training.  In this section, we study whether we can use fewer iterations of PGD for generating the adversarial example and computing variation.  We present our results for training ResNet-18 on CIFAR-10 with source threat model $\ell_2, \epsilon=0.5$ using AT-VR and standard AT ($\lambda=0$) in Table \ref{tab:fewer_iters}. We find that even with a single iteration, AT-VR is able to significantly improve union accuracy over standard AT.
\begin{table}[ht]
\centering
\begin{tabular}{@{}cc|cc|cccc|c@{}}
\hline
& &  & & \multicolumn{4}{c|}{Union with Source} &\\
         \cline{5-8}
         $\lambda$ & PGD & Clean & Source & $\ell_{\infty}$ & $\ell_2$ & StAdv & Recolor & Union all \\
          & iters & acc & acc &($\epsilon=\frac{12}{255}$)& ($\epsilon=1$) & & &  \\
\hline
0      & 1              & \textbf{89.00} & 66.53          & 5.54            & 31.55          & 0.26             & 33.43              & 0.05           \\
0      & 3              & 88.72          & \textbf{67.58} & 7.07            & 35.47          & 0.55             & 36.41              & 0.18           \\
0      & 10             & 88.49          & 66.65          & 6.44            & 34.72          & 0.76             & \textbf{66.52}     & 0.33           \\
1      & 1              & 86.88          & 67.00          & \textbf{11.52}  & \textbf{37.24} & \textbf{38.34}   & 64.44              & \textbf{10.09}\\
\hline
\end{tabular}
\caption{Robust accuracy of ResNet-18 models trained on CIFAR-10 using AT-VR with a single PGD iteration ($\lambda=1$, PGD iters=1) in comparison to standard adversarial training ($\lambda=0$) with various numbers of PGD iterations.}
\label{tab:fewer_iters}
\end{table}
For $\ell_{\infty}$ adversarial training, we find that a single iteration of training does not work well due to the poor performance of adversarial training with FGSM.

\section{Additional Results for Feature Level AT-VR}
\label{app:add_im_features}
In the main paper, we provide results for AT-VR with variation regularization applied at the layer of the logits.  In terms of our theory, this would correspond to considering the identity function to be the top level classifier. In this section, we consider the top level classifier to be all fully connected layers at the end of the NN architectures used and evaluate variation regularization applied at the input into the fully connected classifier.

\subsection{Expansion function for variation on features} In Figure \ref{fig:cif_feature_exp}, we plot the minimum linear expansion function computed on 315 adversarially trained feature extractors (analogous to Figure \ref{fig:cifar_exp} in the main text).  Additionally, we plot the minimum linear expansion function on 300 randomly initialized feature extractors.  For random initialization, we use Xavier normal initialization \citep{GlorotB10} for weights and standard normal initialization for biases.  We find that we can find a linear expansion function with small slope across $\ell_{\infty}$ and $\ell_2$ source and target pairs.  In comparison to expansion function for variation computed at the logits, we find that the slope of the expansion function $s$ found is similar.

\begin{figure}[ht]
    \centering
    \includegraphics[width=\textwidth]{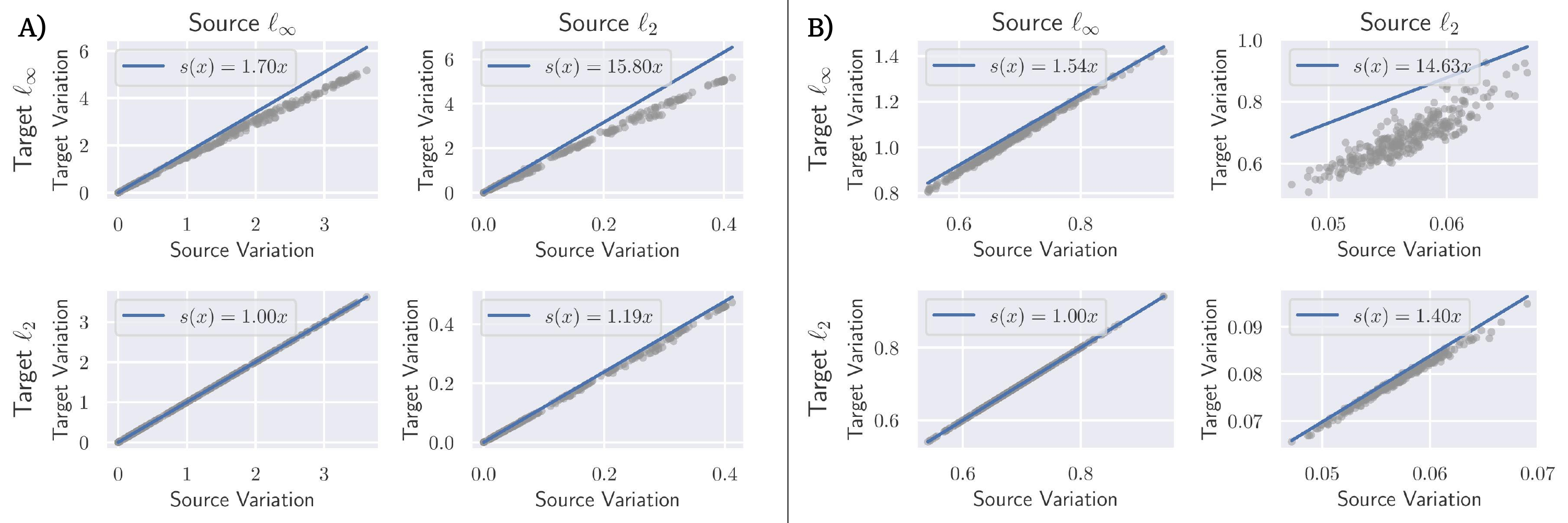}
    \caption{Plots of minimum linear expansion function $s$ shown in blue computed on A) 315 adversarially trained feature extractors and B) 300 randomly initialized feature extractors.  Each grey point represents the variation measured on the source and target attack.  The two columns represent the source adversary ($\ell_{\infty}$ and $\ell_2$ respectively).  The two rows represent the target adversary ($\ell_{\infty}$ and $\ell_2$ respectively).}
    \label{fig:cif_feature_exp}
\end{figure}

\subsection{Additional results with $\ell_p$ target threat models}
We present the unforeseen generalization gap for ResNet-18 models on CIFAR-10 trained with source $\ell_{\infty}$ threat model with radius $\frac{8}{255}$ at various strengths of variation regularization and the corresponding robust accuracy of the model trained with standard AT ($\lambda=0$) and highest variation regularization ($\lambda=2$) in Figure \ref{fig:cif_feature_gen}. We find at large values of unforeseen $\epsilon$, the model trained with variation regularization achieve both smaller unforeseen generalization gap and higher robust accuracy.
\begin{figure*}[ht]
    \centering
    \includegraphics[width=\textwidth]{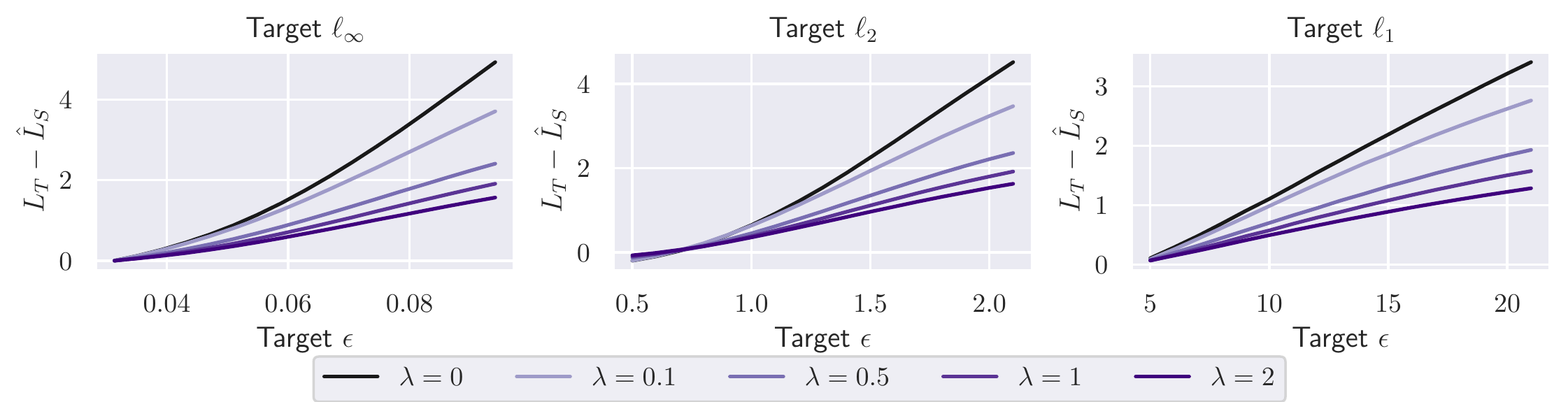}
    \includegraphics[width=1\textwidth]{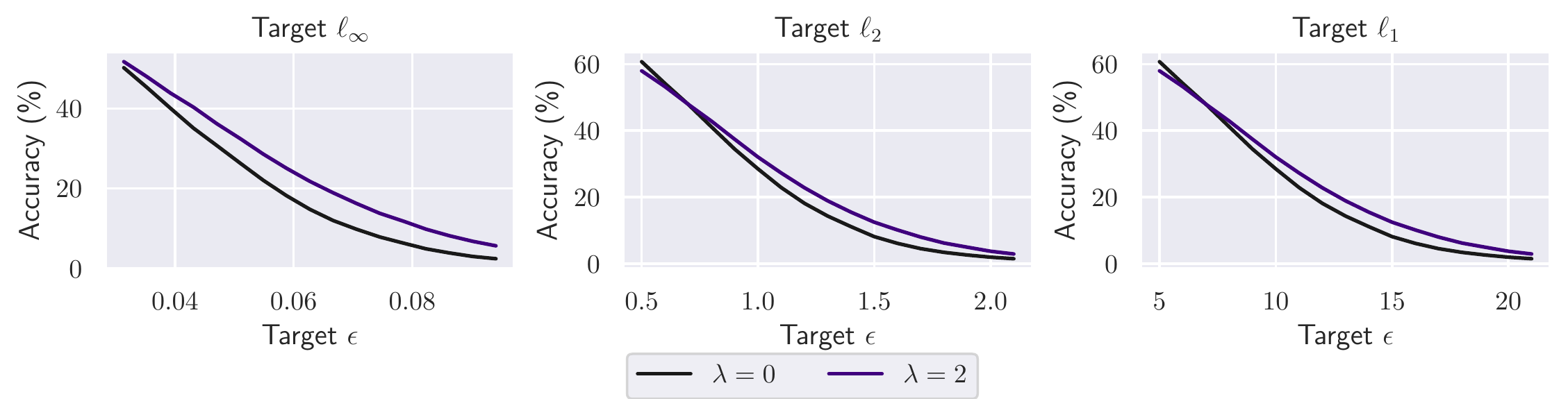}
    \vspace{-20pt}
    \caption{
    Top row: Unforeseen generalization gap of on the CIFAR-10 test set for ResNet-18 models trained using AT-VR at varied regularization strength $\lambda$ measured on adversarial examples of generated by target $\ell_p, p=\{\infty, 2, 1\}$ perturbations with radius $\epsilon$. The generalization gap is measured with respect to cross entropy loss. All models are trained with source $\ell_\infty$ perturbations of radius $\frac{8}{255}$.
    Bottom row: Corresponding robust accuracy of $\lambda=0$ and $\lambda=2$ models displayed in top row.}
    \label{fig:cif_feature_gen}
\end{figure*}

We repeat experiments with ResNet-18 models on CIFAR-10 trained with source $\ell_2$ threat model with radius of 0.5.  We report the measured unforeseen generalization gap to $\ell_{\infty}, \ell_2,$ and $\ell_1$ target threat models at different radii (measured via cross entropy loss on adversarial examples generated with APGD) along with corresponding robust accuracy of the no regularization and maximum regularization strength models in Figure \ref{fig:cif_feature_gen_l2}.  We find that trends observed when the source threat model was $\ell_{\infty}$ are consistent with the trends for $\ell_2$ source threat model: increasing the strength of variation regularization decreases the size of the unforeseen generalization gap and increases robust accuracy across various $\ell_p$ target threat models.
\begin{figure*}[ht]
    \centering
    \includegraphics[width=\textwidth]{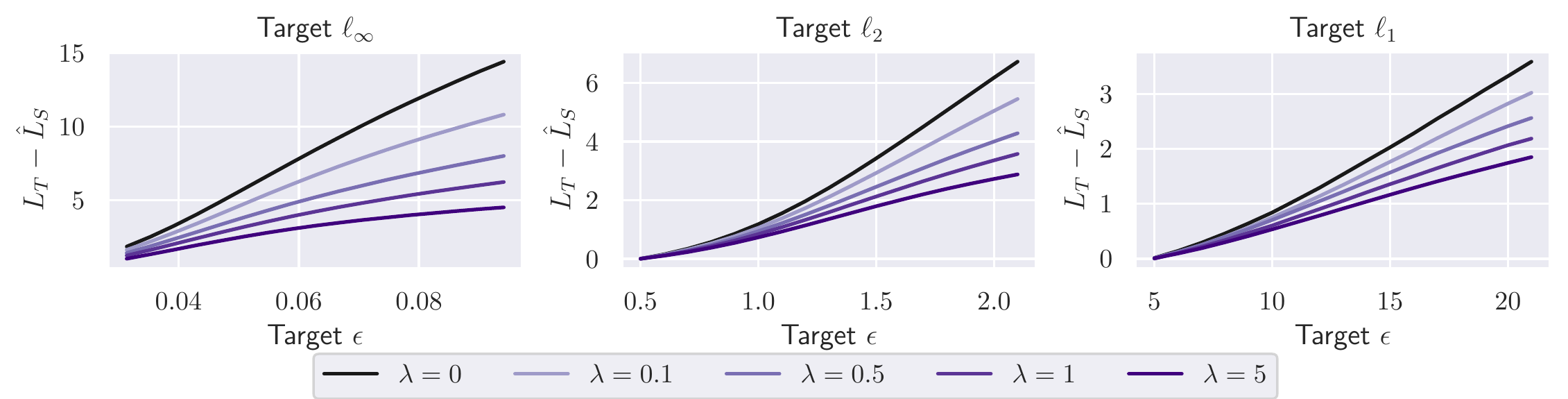}
    \includegraphics[width=1\textwidth]{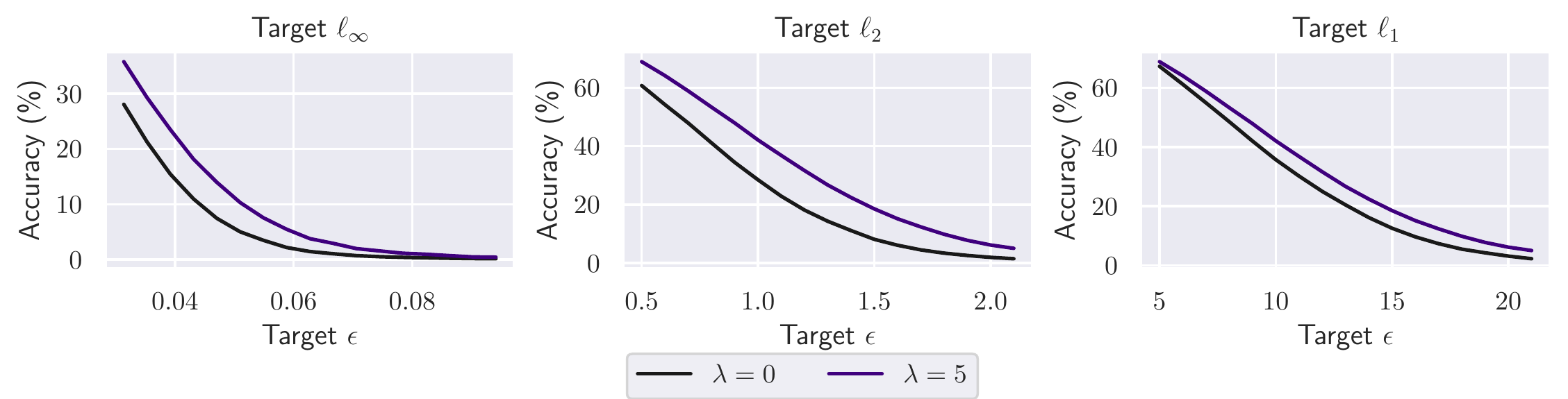}
    \caption{
    Top row: Unforeseen generalization gap of on the CIFAR-10 test set for ResNet-18 models trained using AT-VR at varied regularization strength $\lambda$ measured on adversarial examples of generated by target $\ell_p, p=\{\infty, 2, 1\}$ perturbations with radius $\epsilon$.  The generalization gap is measured with respect to cross entropy loss. All models are trained with source $\ell_2$ perturbations of radius $0.5$.
    Bottom row: Corresponding robust accuracy of $\lambda=0$ and $\lambda=5$ models displayed in top row.}
    \label{fig:cif_feature_gen_l2}
\end{figure*}

\subsection{Robust accuracies with feature level AT-VR}
We repeat experiments corresponding to Table \ref{tab:logits_acc} in the main paper for models trained with feature level AT-VR in Table \ref{tab:unseen_acc}.  We observe similar trends with variation regularization applied at the features instead of logits: variation regularization improves unforeseen accuracy and generally improves source accuracy, but trades off performance on clean images.

\begin{table}[ht]
    \centering
    \fontsize{8}{9}\selectfont
    \begin{tabular}{@{}cccc|cc|cccc|c@{}}
    \hline
         & & & & & & \multicolumn{4}{c|}{Union with Source} &\\
         \cline{7-10}
         Dataset & Architecture & Source & $\lambda$ & Clean & Source & $\ell_{\infty}$ & $\ell_2$ & StAdv & Recolor & Union all \\
         & & & & acc & acc &($\epsilon=\frac{12}{255}$)& ($\epsilon=1$) & & &\\
         
         \hline
         CIFAR-10 & ResNet-18 & $\ell_{2}$ & 0 & \textbf{88.49}& 66.65 & 6.44 & 34.72 & 0.76 & 66.52 & 0.33\\
         CIFAR-10 & ResNet-18 & $\ell_{2}$ & 2.0 &  86.87 & \textbf{68.24}  & 12.66 & 41.05 & 10.01 & \textbf{68.06} & 6.16\\
         CIFAR-10 & ResNet-18 & $\ell_{2}$ & 5.0 & 84.75 & 66.93 & \textbf{13.29 }& \textbf{40.71}& \textbf{29.20}& 66.84 & \textbf{10.83} \\
         
         \hline
         CIFAR-10 & ResNet-18 & $\ell_{\infty}$ & 0 & \textbf{82.83} & 47.47 & 28.09 & 24.94 & 4.38 &  47.47 &  2.48 \\
         CIFAR-10 & ResNet-18 & $\ell_{\infty}$ & 1.0 & 79.72 & \textbf{49.43} & \textbf{32.57} & 26.64 & 11.38 & \textbf{49.43} & 7.09 \\
         CIFAR-10 & ResNet-18 & $\ell_{\infty}$ & 2.0 & 75.58 &48.35 & 32.19& \textbf{26.89}& \textbf{16.56} & 48.35 & \textbf{11.44} \\
         
         \hline
         CIFAR-10 & WRN-28-10 & $\ell_{\infty}$ & 0 &  \textbf{85.93}& 49.86 & 28.73 & 20.89 & 2.28 & 49.86 & 1.10 \\
         CIFAR-10 & WRN-28-10 & $\ell_{\infty}$ & 0.5 & 85.86 & 50.13 & 30.04 & 21.62 & 5.36 & 50.13 & 4.14 \\
         CIFAR-10 & WRN-28-10 & $\ell_{\infty}$ & 1 & 84.27 & \textbf{51.01} & \textbf{31.47} & \textbf{22.86} & \textbf{9.71} & \textbf{51.01} & \textbf{7.59} \\
         \hline
         CIFAR-10 & VGG-16 & $\ell_{\infty}$ & 0 & \textbf{79.67} & 44.36 & 26.14 & 30.82 & 7.31 & 44.36 & 4.35\\
         CIFAR-10 & VGG-16 & $\ell_{\infty}$ & 0.01 & 76.38 & \textbf{44.87} & \textbf{27.35} & \textbf{32.59}& 9.14 & \textbf{44.87}  & 5.69\\
         CIFAR-10 & VGG-16 & $\ell_{\infty}$ & 0.05 & 72.27 & 42.14 & 26.80 & 32.41 & \textbf{12.18} & 42.14  & \textbf{8.02}\\
         
                 \hline
         ImageNette & ResNet-18 & $\ell_2$ &  0 &\textbf{88.94} & \textbf{84.99} & 0.00 & 79.08 & 1.27 & 72.15  & 0.00 \\
         ImageNette & ResNet-18 & $\ell_2$ & 1.0 & 86.29 & 83.62 & 2.55 & \textbf{80.20} & 8.66 & 73.25 & 1.45 \\
         ImageNette & ResNet-18 & $\ell_2$ & 5.0 &  83.06 & 80.89 &\textbf{ 10.11} & 78.60 & \textbf{22.98} & \textbf{74.22 }&\textbf{7.75} \\
         
          \hline
         
         ImageNette & ResNet-18 & $\ell_{\infty}$ & 0 & \textbf{80.56} & 49.63 & 32.38 & 49.63 & 34.27 & 49.63 & 25.68\\
         ImageNette & ResNet-18 & $\ell_{\infty}$ & 0.05 & 79.06 & \textbf{50.47} & 34.06 & \textbf{50.47} & 37.40 & \textbf{50.47} & 28.89 \\
         ImageNette & ResNet-18 & $\ell_{\infty}$ & 0.1 & 78.09 & 50.01 & \textbf{34.11} & 50.01& \textbf{38.32} & 50.01 &\textbf{ 29.30}  \\
         
          \hline
         CIFAR-100 & ResNet-18 & $\ell_{2}$ & 0 & \textbf{60.92} & 36.01 & 3.98 & 16.90 & 1.80 & 34.87 & 0.40  \\
         CIFAR-100 & ResNet-18 & $\ell_{2}$ & 1 & 56.37 & \textbf{38.66} & 8.65 & \textbf{23.41} &4.81 & \textbf{37.52} & 2.10\\
         CIFAR-100 & ResNet-18 & $\ell_{2}$ & 2 & 52.73 & 36.15 & \textbf{8.76} & 22.33 & \textbf{7.46} & 35.28 & \textbf{3.14}\\
         
         \hline
         CIFAR-100 & ResNet-18 & $\ell_{\infty}$ & 0 & \textbf{54.94} & 22.74 & 12.61 & 14.40 & 3.99 & 22.71 &  2.42 \\
         CIFAR-100 & ResNet-18 & $\ell_{\infty}$ & 0.1 & 54.21 & 23.52 & 13.61 & 15.10& 4.10 & 23.48 & 2.54 \\
         CIFAR-100 & ResNet-18 & $\ell_{\infty}$ & 0.5 & 49.29 & \textbf{24.66} & \textbf{16.02} & \textbf{15.62} & \textbf{5.74 }& \textbf{24.58} & \textbf{3.70}\\
        \hline

    \end{tabular}
    \caption{Robust accuracy of various models trained at different strengths of variation regularization on various threat models.  Models are trained with either source threat model $\ell_{\infty}$ with radius $\frac{8}{255}$ or $\ell_2$ with radius $0.5$.   The ``source acc" column reports the accuracy on the source attack.  For each individual threat model, we evaluate accuracy on a union with the source threat model.  The union all column reports the accuracy obtained on the union across all listed threat models.}
    \label{tab:unseen_acc}
\end{table}


\end{document}